\documentclass{ecai} 

\usepackage{latexsym}
\usepackage{enumitem}

\usepackage[utf8x]{inputenc}
\usepackage{amsmath,amssymb}
\usepackage{xcolor}
\usepackage{mathtools}
\usepackage{amsthm}
\usepackage{graphicx}
\usepackage{amsfonts}
\usepackage{MnSymbol}
\usepackage{tabulary}
\usepackage{tabularx}
\usepackage[ruled,vlined]{algorithm2e}
\usepackage{algorithmic}
\usepackage{multirow}
\usepackage{booktabs} 
\usepackage{url}

\newtheorem{theorem}{Theorem}
\newtheorem{lemma}[theorem]{Lemma}

\newcommand{\BibTeX}{B\kern-.05em{\sc i\kern-.025em b}\kern-.08em\TeX}

\begin{document}


\begin{frontmatter}


\paperid{8277} 


\title{HypER: Hyperbolic Echo State Networks for Capturing Stretch-and-Fold Dynamics in Chaotic Flows
}


\author[A]{\fnms{Pradeep}~\snm{Singh}\thanks{Corresponding Author. Email: pradeep.cs@sric.iitr.ac.in}}
\author[A]{\fnms{Sutirtha}~\snm{Ghosh}\thanks{Corresponding Author. Email: s\_ghosh@cs.iitr.ac.in}}
\author[A]{\fnms{Ashutosh}~\snm{Kumar}}
\author[A]{\fnms{Hrishit}~\snm{B P}}
\author[A]{\fnms{Balasubramanian}~\snm{Raman}}

\address[A]{Machine Intelligence Lab, Department of Computer Science and Engineering, IIT Roorkee, India}



\begin{abstract}
Forecasting chaotic dynamics beyond a few Lyapunov times is difficult because infinitesimal errors grow exponentially.  Existing Echo State Networks (ESNs) mitigate this growth but employ reservoirs whose Euclidean geometry is mismatched to the \emph{stretch–and–fold} structure of chaos.  We introduce the \textit{Hyperbolic Embedding Reservoir} (\textsc{HypER}), an ESN whose neurons are sampled in the Poincaré ball and whose connections decay exponentially with hyperbolic distance.  This negative‑curvature construction embeds an \emph{exponential metric} directly into the latent space, aligning the reservoir’s local expansion–contraction spectrum with the system’s Lyapunov directions while preserving standard ESN features such as sparsity, leaky integration and spectral‑radius control.  Training is limited to a Tikhonov-regularised read-out.  On the chaotic Lorenz-63 and Rössler systems, and the hyperchaotic Chen–Ueta attractor, HypER consistently lengthens the mean valid-prediction horizon beyond Euclidean and graph-structured ESN baselines, with statistically significant gains confirmed over 30 random seeds; parallel results on real-world benchmarks—including the Santa Fe laser series, MIT-BIH heart-rate variability and international sunspot numbers—corroborate its advantage.  We further establish a lower bound on the  rate of state divergence for HypER, mirroring Lyapunov growth.  

\end{abstract}

\end{frontmatter}


\section{Introduction}

Neural circuits in the brain are often conceptualized as high-dimensional dynamical systems that can host \emph{attractor} states—regions in phase space where network activity converges or persists over time—and these attractors underlie vital processes such as working memory, decision-making, and robust perception \cite{hopfield1982neural,Seung2000}. 
Electrophysiological and calcium‑imaging evidence further indicates that many cortical microcircuits operate near the edge of chaos, where trajectories explore a rich repertoire of quasi‑stable states yet remain exquisitely sensitive to perturbations \cite{BuonomanoMauk1994,Mante2013}.  Capturing such behaviour is notoriously difficult because any modelling or measurement error is amplified exponentially: if two initial conditions differ by~$\|\delta x_0\|$, their separation grows on average as $\|\delta x_t\|\!\approx\!\|\delta x_0\|\exp(\lambda_{\max}t)$ wherever the largest Lyapunov exponent (LLE) $\lambda_{\max}$ is positive.  Beyond a few Lyapunov times accurate prediction becomes impossible unless the model itself encodes the local “stretch–and–fold’’ geometry that underlies chaotic divergence.
From a theoretical standpoint, capturing these highly sensitive, self-organizing phenomena calls for a modeling framework that can explicitly reflect the underlying geometry of exponential separation. 

Reservoir computing (RC), and especially Echo State Networks (ESNs), has emerged as a promising data‑driven framework for chaos forecasting because it replaces costly recurrent training with a fixed nonlinear reservoir and a linear read‑out \cite{jaeger2001echo}.  ESNs can learn short‑term structure and even replicate Lyapunov spectra of benchmark systems, yet the typical prediction horizon remains limited to roughly 5-8 Lyapunov times \cite{huawei2020longterm,jaurigue2024chaotic}.  A plausible reason is geometric: conventional reservoirs are embedded in flat Euclidean space or simple random graphs and therefore lack an explicit inductive bias for the exponential expansion and contraction intrinsic to chaotic flows.

\begin{figure}[!ht]
  \centering

  \begin{minipage}[t]{0.495\linewidth}
    \centering
    \includegraphics[width=1\textwidth, height=0.56\textwidth]{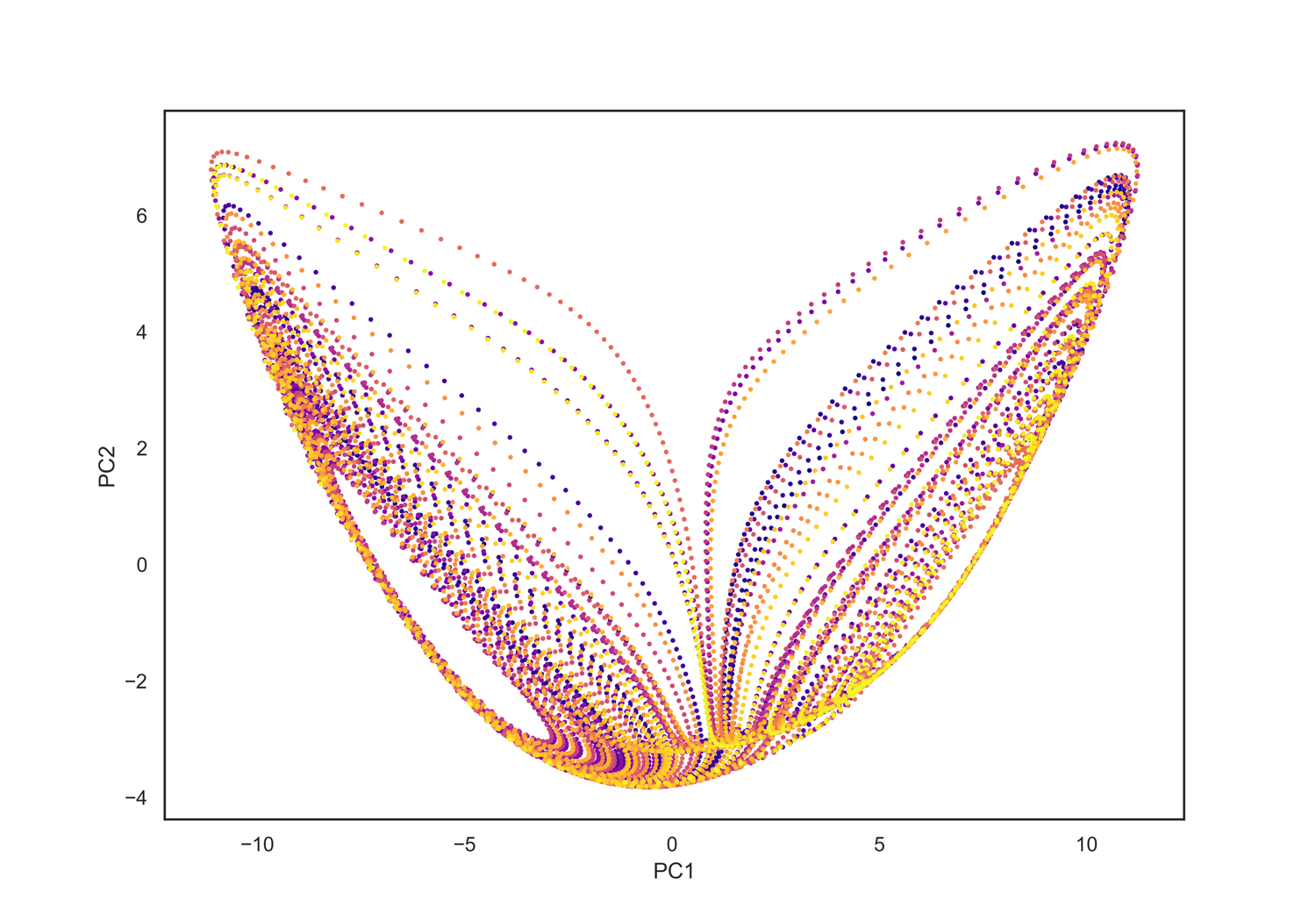}\\
    \textbf{(a)} Standard ESN
    \label{fig:sub1}
  \end{minipage}
  \hfill
  \begin{minipage}[t]{0.495\linewidth}
    \centering
    \includegraphics[width=1\textwidth, height=0.56\textwidth]{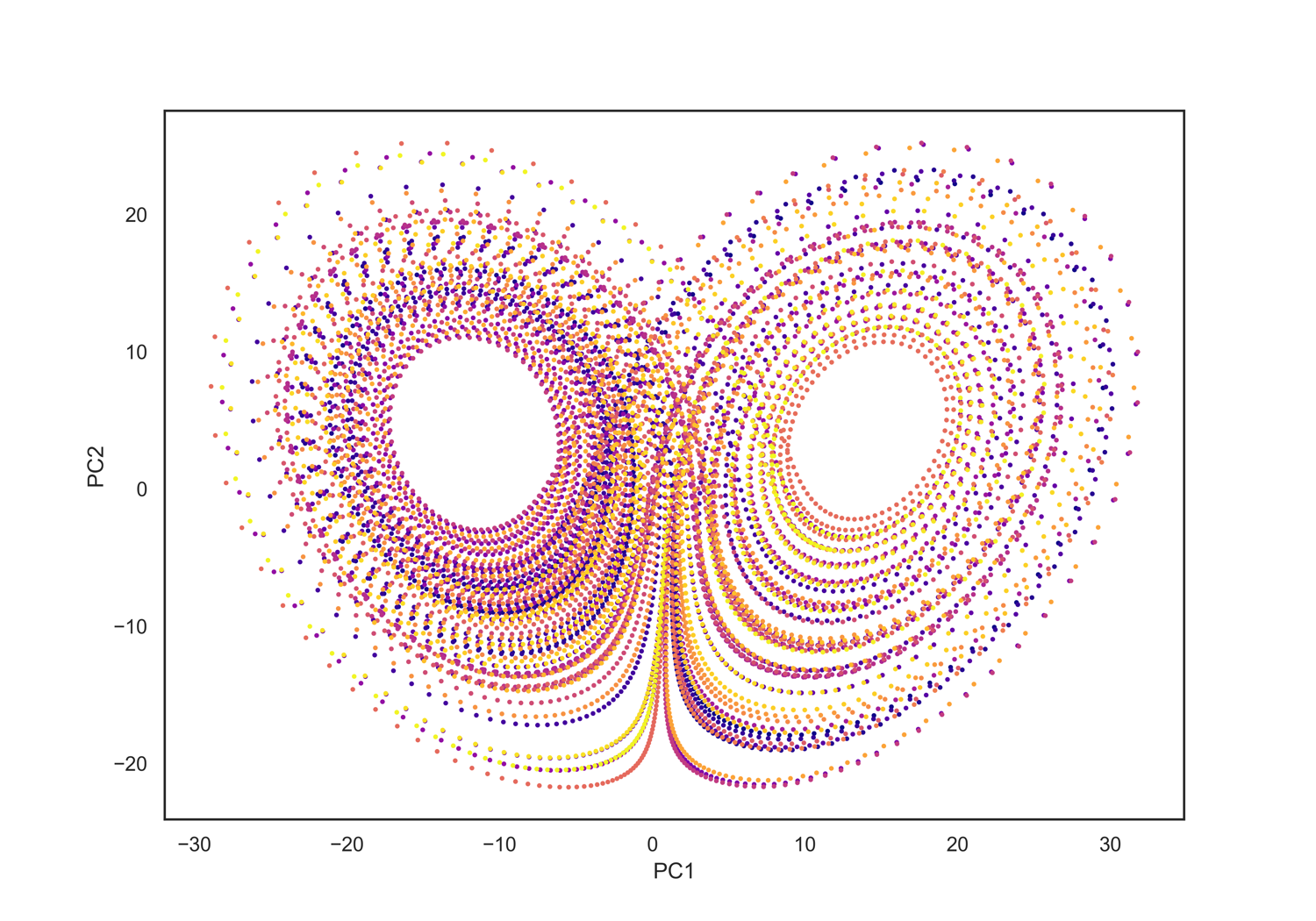}\\
    \textbf{(b)} Proposed HypER
    \label{fig:sub3}
  \end{minipage}

  \caption{Two‑dimensional PCA projections of high-dimensional reservoir states (\emph{no read-out training applied}) for (a) a standard ESN and (b) the proposed HypER, when both networks are driven by the Lorenz system.}
  \label{fig:lorenz_comparison_2d}
\end{figure}

 Hyperbolic spaces with constant negative curvature are characterized by metric properties that mirror exponential growth: distances and volumes expand exponentially with radius, unlike Euclidean spaces where growth is polynomial \cite{zhang2023hippocampal}. Intuitively, a small change in coordinates on a hyperbolic manifold can produce a large change in geodesic distance, akin to how a small state perturbation in a chaotic system leads to a large future divergence. This makes hyperbolic geometry a natural substrate to model the local instabilities of chaos. Yet, most sequential models, bound to flat Euclidean latents, struggle to reproduce this exponential pull—begging a sharper question:

\begin{quote}
\emph{Does wiring a reservoir in hyperbolic space unlock a longer glimpse into chaotic futures?}
\end{quote}

We contend that a negatively curved latent space supplies exactly that bias.  
Building on this intuition we introduce the \textit{Hyperbolic Embedding Reservoir} (\textsc{HypER}), an ESN whose neurons are placed in the Poincaré ball and whose recurrent weights decay exponentially with hyperbolic distance.  Hyperbolic sampling preserves the natural volume element, so nodes near the boundary correspond to rapidly diverging directions while central nodes capture contracting modes.

Empirical tests on the Lorenz–63, Rössler and Chen-Ueta attractors show that a modest‑sized \textsc{HypER} consistently extends the valid‑prediction window and slows error growth relative to Euclidean and graph‑structured baselines.  Qualitative visualisations (cf. Figure \ref{fig:lorenz_comparison_2d}) confirm that the reservoir’s state evolution inherits the expected stretch‑and‑fold pattern, providing an interpretable link between negative curvature and chaotic forecasting performance.  These results position hyperbolic geometry as a principled lever for long‑range prediction in both neuroscience and nonlinear‑dynamics applications.

\paragraph{Summary of Key Contributions.}  
We present \textsc{HypER}, the first reservoir computer whose neurons are embedded in the Poincaré ball and whose recurrent weights decay exponentially with hyperbolic distance. 
This curvature-aware wiring installs an inductive bias missing from Euclidean or graph-structured ESNs and, supported by a state-divergence theorem, translates into substantially longer and more stable forecasts. Empirically, HypER not only extends the prediction horizon on canonical Lorenz-63, Rössler, Chen-Ueta, and Chua attractors, but also delivers strong multi-step accuracy on real-world chaotic benchmarks—including the Santa Fe laser series, MIT-BIH heart-rate variability, and international sunspot numbers.

 \section{Background and Related Works}

\paragraph{Chaotic Dynamics and the Limits of Predictability.}
A deterministic dynamical system is specified by a flow
$\psi_t:\mathcal M\!\to\!\mathcal M$ on a smooth manifold
$\mathcal M$; the state $\mathbf x_t=\psi_t(\mathbf x_0)$
evolves according to $\dot{\mathbf x}=F(\mathbf x)$ in continuous
time or $\mathbf x_{k+1}=G(\mathbf x_k)$ in discrete time
\cite{katok1997modern}. 
With initial conditions $\mathbf x_0, \mathbf y_0$, nearby trajectories satisfy
$\|\psi_t(\mathbf x_0)-\psi_t(\mathbf y_0)\|
      \simeq\|\mathbf x_0-\mathbf y_0\|\exp(\lambda_{\max}t)$,
so they diverge when the largest Lyapunov exponent $\lambda_{\max}>0$ (the classical signature of chaos) and converge exponentially when $\lambda_{\max}<0$ \cite{eckmann1985ergodic}.
  The amplification
continues until nonlinear folding confines the motion to a
\emph{strange attractor}, whose fractal geometry can be visualised
through Poincaré sections or return maps—classical examples are the
double‑scroll Lorenz set and the horseshoe‑like Rössler map
\cite{lorenz1963deterministic,rossler1976equation}.  Because the
forecast error grows by a factor of $e$ roughly every $1/\lambda_{\max}$\,time units, even perfect models face a finite
predictability horizon (about two weeks for mid‑latitude weather
\cite{palmer2000predictability}).  Any learning system aimed at
long‑range forecasting must therefore encode, rather than merely fit,
the local stretch–and–fold geometry that drives error growth.

\paragraph{Reservoir Computing for Chaotic Time Series.}
ESNs and their real‑time spiking analogue, the
Liquid State Machine, sidestep the difficulties of full recurrent training by fixing the recurrent weights of a large, sparsely connected “reservoir’’ and learning only a linear read‑out \cite{jaeger2001echo,maass2002real}. 
  When the spectral radius \(\rho(\mathbf{W})\!<\!1\) the \emph{Echo-State Property} (ESP) guarantees that all influences of arbitrary initialisation vanish, leaving a fading-memory encoding of recent inputs.  But long-horizon chaos prediction pushes the reservoir to the opposite regime: one must drive \(\rho(\mathbf{W})\) \emph{towards} unity to sustain the internal dynamics, at which point ESP becomes fragile and the model either explodes or collapses.  Closed-loop ESNs tuned by hand achieve only \(5\!-\!8\) Lyapunov times before divergence on Lorenz-63, Kuramoto–Sivashinsky and climate benchmarks \cite{huawei2020longterm,jaurigue2024chaotic,pathak2018hybrid}.  Subsequent variants—small-world, scale-free, even fully \emph{uncoupled} reservoirs—shift the memory-nonlinearity balance but remain \emph{Euclidean}, offering no bias toward the exponential expansion that defines chaos \cite{lu2017reservoir,rodan2011minimum}.  Our hyperbolic reservoir remedies this tension: by embedding neurons in the Poincaré ball and normalising the spectrum after curvature-aware weight construction, HypER raises the \emph{minimum} singular value while keeping \(\rho(\mathbf{W})<1\).  The network thus retains ESP yet still magnifies perturbations at a provably super-unit rate, extending prediction horizons without sacrificing stability.


\paragraph{Hyperbolic Geometry in Machine Learning.}
Negative‑curvature manifolds accommodate exponential growth in
volume, making them ideal for embedding trees and other hierarchical
structures with low distortion.  Nickel and Kiela’s Poincaré
embeddings \cite{nickel2017poincare} sparked a wave of
hyperbolic‑representation methods, from gyrovector‑based feed‑forward
layers to hyperbolic graph neural networks
\cite{chami2019hyperbolic,ganea2018hyperbolic}.  These models exploit
the fact that a small Euclidean displacement near the boundary of the
Poincaré ball corresponds to a large geodesic increment, mirroring
the sensitivity required to separate similar objects in hierarchical
data.  Despite this progress, hyperbolic deep learning has focused
almost exclusively on static tasks; we are unaware of prior work that
harnesses negative curvature to model \emph{continuous‑time chaotic
dynamics}.  Our \textsc{HypER} architecture fills that gap by
sampling reservoir nodes directly in hyperbolic space and defining
connectivity as an exponential function of hyperbolic distance,
thereby hard‑coding the stretch–and–fold metric into the recurrent
state itself.

\smallskip
Taken together, these strands motivate our \textit{central hypothesis}:
embedding the reservoir in a Poincaré ball couples the network’s internal geometry to the stretch–fold mechanism of chaos, enabling longer and more stable forecasts than any flat-space or purely graph-based ESN examined to date. Figure \ref{fig:lorenz_comparison_2d} illustrates the point. Driven by the same Lorenz signal—and \emph{before} any supervised read-out is trained—the vanilla ESN does reproduce a two-lobe structure, but the lobes are warped and compressed toward the centre; in contrast, HypER traces two crisp, well-separated wings that closely match the topology of the true attractor.

\section{Methodology}
\label{sec:methodology}


\paragraph{Problem Statement.} 
Let \(\{\mathbf{u}_t\}_{t\in\mathcal{T}}\subset \mathbb{R}^{m}\) be a time series generated by a chaotic system of dimension \(m\).  Our objective is to construct an ESN that, given a finite training segment \(\{\mathbf{u}_t\}_{t=1}^T\), learns to forecast \(\mathbf{u}_{t+1}\in\mathbb{R}^m\) in an autoregressive fashion for \(t>T\).  Formally, we seek to minimize the multi-step prediction error \(\sum_{t=T+1}^{T+H}\|\mathbf{u}_t-\hat{\mathbf{u}}_t\|^2\) over some horizon \(H\gg 1\).  

\begin{figure*}[!ht]
    \centering
    \includegraphics[width=0.9\linewidth, height=0.46\linewidth]{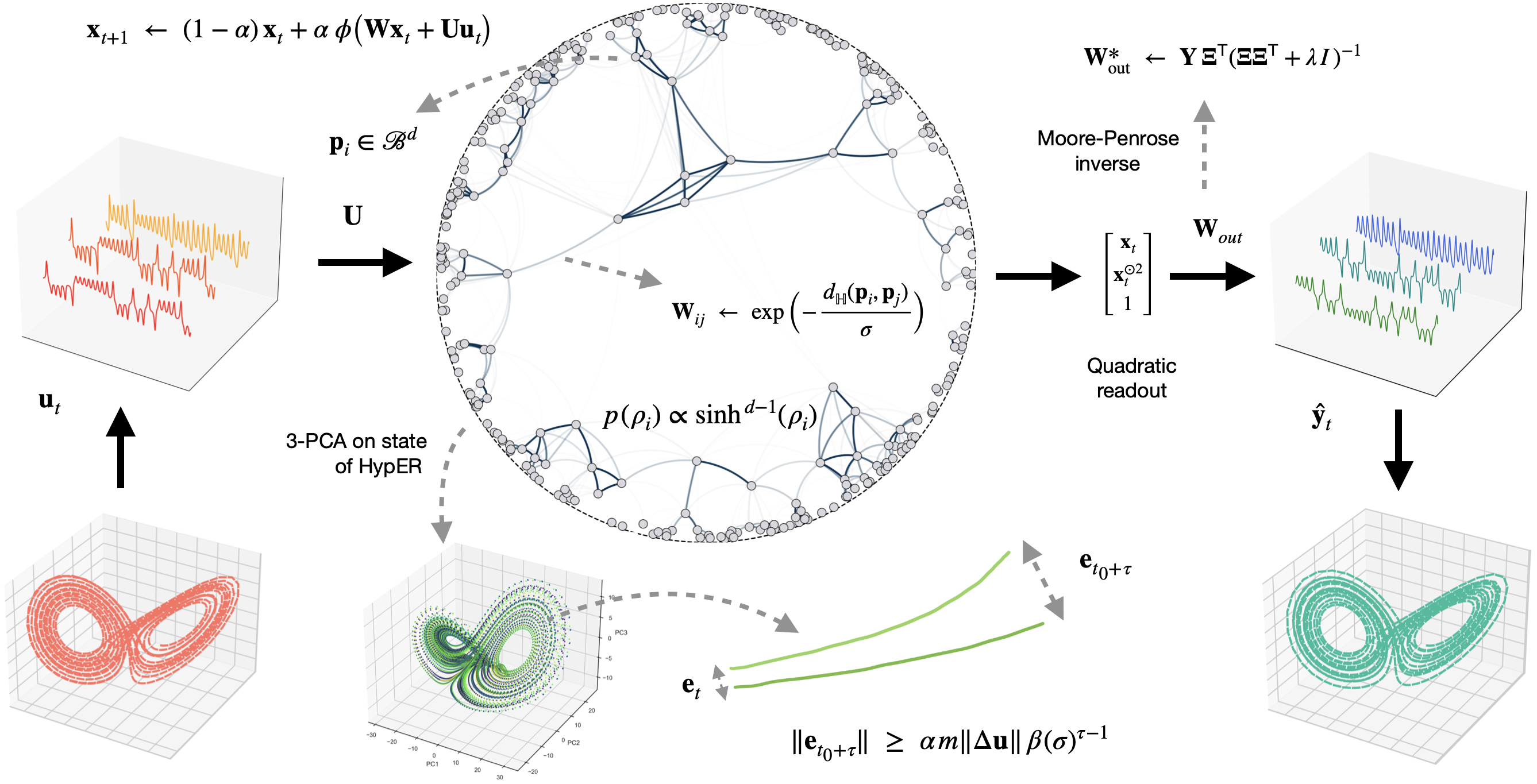}

   \caption{HypER embedded in the Poincaré disk: nodes placed evenly in hyperbolic volume, links decay with distance.}
    \label{fig:adjacency}
\end{figure*}

\subsection{Hyperbolic Embeddings and Adjacency}
\label{sec:hyperbolic-embedding}

\paragraph{Poincaré Ball Model.} 
Let \(\bigl(\mathcal{B}^{d},g_{\chi}\bigr)\) be the \(d\)-dimensional
Poincaré ball endowed with the Riemannian metric \(g_{\chi}\) of constant
sectional curvature \(\chi<0\) (\(\chi=-1\) in all experiments).
For a fixed integer \(d\ge2\), the \emph{Poincaré \(d\)-ball} of radius \(1\) is given by
\(
  \mathcal{B}^{d}
  \;=\; \bigl\{
     \mathbf{p} \in \mathbb{R}^d 
     \,\mid\, \|\mathbf{p}\| < 1 
  \bigr\},
\)
where \(\|\cdot\|\) denotes the usual Euclidean norm in \(\mathbb{R}^d\). The geometry in this model is governed by a Riemannian metric \cite{doCarmo1992, Lee2018} given by
$ds^2 = \frac{4 \sum_{i=1}^d dx_i^2}{(1 - \sum_{i=1}^d x_i^2)^2},$
ensuring that distances grow infinitely large as one approaches the boundary of the model. Distances in this ball are expressed as
\begin{equation}
  \label{eqn:hyp-distance}
d_{\mathbb H}\bigl(\mathbf{p}_i,\mathbf{p}_j\bigr) 
  \;=\; 
  \mathrm{arcosh}\ \!\Bigl(
    1 
    \;+\; 
    2\,\frac{\|\mathbf{p}_i - \mathbf{p}_j\|^{2}}
             {\bigl(1 - \|\mathbf{p}_i\|^{2}\bigr)\,\bigl(1 - \|\mathbf{p}_j\|^{2}\bigr)}
  \Bigr),
\end{equation}
where \(\mathrm{arcosh}(\zeta) = \ln\ \!\bigl(\zeta + \sqrt{\zeta^{2} - 1}\bigr)\). Geodesics in the Poincaré ball model are either straight lines through the origin or circular arcs that intersect the boundary orthogonally \cite{Anderson1999}. The model enjoys conformality, meaning that angles are preserved \cite{Thurston1997}, making it particularly valuable for embedding tasks in machine learning where local geometric relationships are important \cite{nickel2017poincare}. 
\paragraph{Node Placement.} 
Let \(N\) be the dimension (number of nodes) of the reservoir. To embed these nodes in a \(d\)-dimensional hyperbolic geometry, we assign each node \(i \in \{1,2,\dots,N\}\) a coordinate \(\mathbf{p}_i \in \mathcal{B}^{d}\). 
A variety of sampling schemes are possible for placing nodes in the Poincaré ball \(\mathcal{B}^{d}\). We highlight two distinct strategies, each of which assigns radial and angular coordinates \(\bigl(r_i,\,\boldsymbol{\omega}_i\bigr)\) and hence yields a node position 
\(\mathbf{p}_i = r_i\,\boldsymbol{\omega}_i\),
where 
\(r_i = \|\mathbf{p}_i\| < 1\) is the magnitude of $\mathbf{p}_i$
and
\(\boldsymbol{\omega}_i \in \mathbb{S}^{d-1}\) is a unit vector on the \((d-1)\)-dimensional sphere, representing the direction. The open-ball constraint \(\|\mathbf{p}_i\|<1\) guarantees that the hyperbolic distance (refer \eqref{eqn:hyp-distance}) is well-defined. 

(i) \textit{Euclidean-Isotropic Sampling (Baseline):}  
In this scheme, we do not take into account the hyperbolic distance distribution when choosing \(r_i\) and instead sample it to be uniformly distributed according to the Euclidean volume in the ball of radius \(R<1\). Concretely, we set \(r_i = \sqrt[d]{u_i} \,R\) where $u_i \sim \operatorname{Uniform}(0, 1)$. Then, the direction \(\boldsymbol{\omega}_i \in \mathbb{S}^{d-1}\) is sampled uniformly on the \((d-1)\)-dimensional sphere, owing to the symmetry of the Poincaré ball.  However, this method does \emph{not} produce a uniform distribution with respect to the hyperbolic volume element, and nodes closer to the boundary will not reflect the exponential concentration that true hyperbolic geometry would suggest. We include it as a baseline for ablation studies (cf. \S\ref{experiments}), to compare how a purely Euclidean-based distribution affects performance relative to a geometry-aware distribution.

(ii) \textit{Hyperbolic-Uniform Sampling:}  
For a full match to the negative-curvature volume element, we let the hyperbolic radius \(\rho_i \in [0,\rho_{\max})\) be distributed according to the  probability density function (PDF) 
$p(\rho_i)\propto \sinh^{\,d-1}(\rho_i)$,
where \(\rho_i=2\mathrm{artanh}(r_i)\) in the Poincaré model. After sampling \(\rho_i\), we then convert it to its Euclidean counterpart \(r_i=\tanh(\frac{\rho_i}{2})\) and again draw the angular components  \(\boldsymbol{\omega}_i\) uniformly on \(\mathbb{S}^{d-1}\).  Focusing on the case of the Poincaré disc, the PDF takes the form
    $p(\rho_i)\propto \sinh(\rho_i)$.
It is straightforward to obtain the cumulative distribution function (CDF) 
        $F(\rho_i) = \int_0^{\rho_i} p(\rho) \, d\rho \propto 
         \cosh(\rho_i) - 1$.
Equating the normalized CDF to $u_i \sim \operatorname{Uniform}(0, 1)$ guarantees that the samples are uniformly distributed in probability space
\begin{equation}
\label{eqn:hyp_cdf_ratio}
    \frac{F(\rho_i)}{F(\rho_{max})} = \frac{\cosh(\rho_i) - 1}{\cosh(\rho_{max}) - 1} = u_i
\end{equation}
Accordingly, we obtain the hyperbolic radius $\rho_i$ as follows
    $\rho_i = \operatorname{arcosh}(u_i(\cosh(\rho_{max})-1)+1)$.
This procedure is precisely uniform with respect to the hyperbolic volume element in a \(d\)-dimensional space of constant curvature \(-1\). Consequently, it generates point placements that reflect the genuine exponential “stretching” near the boundary \(\|\mathbf{p}_i\|\to 1\). In applications where it is crucial to respect the native hyperbolic measure, this approach is the most principled choice \cite{Anderson1999}. Note that both sampling schemes satisfy \(\|\mathbf{p}_i\| < 1\) by construction, and yield valid embeddings in the Poincaré ball for which the hyperbolic distance (Eqn. \ref{eqn:hyp-distance}) is well-defined.

\paragraph{Connectivity Kernel.} For any pair of nodes \(i\) and \(j\), we compute their hyperbolic distance
   $d_{\mathbb H}\bigl(\mathbf{p}_i,\mathbf{p}_j\bigr)$. 
 We define the reservoir’s adjacency matrix  by applying an exponential kernel to the hyperbolic metric
  $W_{ij} 
  \;=\; \exp\ \!\Bigl(-\,\frac{d_{\mathbb H}(\mathbf{p}_i,\mathbf{p}_j)}{\sigma}\Bigr),
  \quad
  \text{for } 1 \le i,j \le N,$
where \(\sigma>0\) is the \emph{kernel width} controlling how rapidly connectivity decays with increasing hyperbolic distance. Concretely, two nodes that lie close enough in \(\mathcal{B}^d\) (i.e., \(\mathbf{p}_i \approx \mathbf{p}_j\)) will have relatively large weight \(W_{ij}\), whereas nodes far apart in hyperbolic distance will be coupled only weakly (or nearly zero if we impose sparsity conditions) (cf. Figure \ref{fig:adjacency}). Consequently, \(W\) reflects the geometry of negative curvature: nodes near the boundary can have large pairwise distances and, thus, smaller mutual connection strengths.

\begin{lemma}[Spectrum of the hyperbolic kernel]\label{lemma1}
Let \(\{\mathbf{p}_i\}_{i=1}^N\subset\mathcal{B}^d\) be distinct points.
Set
\(\displaystyle
  \delta=\min_{i\neq j}d_{\mathbb H}(\mathbf{p}_i,\mathbf{p}_j).
\)
Then  
  $\lambda_{\min}(\mathbf{\widetilde W})\;\ge\;1-(N-1)e^{-\delta/\sigma},\
  \rho(\mathbf{\widetilde W})\;\le\;1+(N-1)e^{-\delta/\sigma}$.
 \textit{(Proof in supplementary)}
\end{lemma}

\paragraph{Sparsity \& Spectral Normalization.} 
For computational efficiency and enhanced dynamical properties, we impose \emph{row-level sparsity} on \(W\). For each row \(i\in\{1,\dots,N\}\), we keep only the top-\(\kappa\) largest entries of row \(i\) and set all others to zero. Denote the resulting sparse matrix by \(\widetilde{\mathbf{W}}\). We then ensure that the reservoir’s internal dynamics satisfy the ESP \cite{jaeger2001echo} by restricting the spectral radius of \(\widetilde{\mathbf{W}}\). Define the final reservoir matrix \(\mathbf{W}\in \mathbb{R}^{N\times N}\) as
  $\frac{\varrho}{\rho\ \!\bigl(\widetilde{\mathbf{W}}\bigr)} \,\widetilde{\mathbf{W}}$,
 where \(\rho(\widetilde{\mathbf{W}})\) denotes the spectral radius of \(\widetilde{\mathbf{W}}\), and \(\varrho\in(0,1)\) is a user-specified target radius. This \emph{spectral normalization} preserves the geometry-induced pattern in \(\widetilde{\mathbf{W}}\) while ensuring that repeated application of \(\mathbf{W}\) does not drive unbounded state growth. By choosing \(\varrho < 1\), we allow the reservoir to possess sufficiently rich dynamics near the edge of stability while still converging during teacher forcing.

\subsection{Reservoir Configuration}
\label{sec:esn-configuration}

\paragraph{Leaky Echo State Update.} The reservoir is modeled as a discrete-time recurrence relation with a state-space representation. Recall that $\mathbf{u}_{t} \in \mathbb{R}^m$ denotes the input at time $t$ and define $\mathbf{x}_{t} \in \mathbb{R}^N$ as the reservoir state vector at time $t$.  The dynamical evolution of the \emph{leaky ESN} is given by the state update equation
\begin{equation}
  \label{eqn:leaky-update}
  \mathbf{x}_{t+1}
  \;=\; (1-\alpha)\,\mathbf{x}_{t}
       \;+\;\alpha\,\phi\ \!\Bigl( \mathbf{W}\,\mathbf{x}_{t}
       + \mathbf{U}\,\mathbf{u}_{t} \Bigr),
\end{equation}
where \(\alpha\in(0,1]\) is the \emph{leak rate}, \(\mathbf{U}\in\mathbb{R}^{N\times m}\) is a random \emph{input weight} matrix, and \(\phi: \mathbb{R}^N \to \mathbb{R}^N\) 
 is a $1$‑Lipschitz 
pointwise nonlinear activation. The term \((1-\alpha)\mathbf{x}_{t}\) provides a leaky integration that can enhance memory and stability in the reservoir \cite{jaeger2001echo, jaeger2004harnessing}.

\begin{lemma}[Linearized forward sensitivity]\label{lemma2}
Let   
\(\mathbf{W}\in\mathbb R^{N\times N}\) be symmetric; denote its smallest eigenvalue by \(\lambda_{\min}(\mathbf{W})\ge 0\);
the activation \(\phi\) is \(C^{1}\) with derivative bounded on the reachable domain by  
  \(
    0<m\;\le\;\phi'(z)\;\le\;L<\infty.
  \)
For any current input \(u\) and reservoir state \(x\) define  
\(
J(\mathbf{x})\;=\; \frac{\partial \mathbf{x}_{t+1}}{\partial \mathbf{x}_t}
        \;=\;(1-\alpha)I+\alpha D_\phi(\mathbf{x})\,\mathbf{W},
\) where
\(D_\phi(\mathbf{x})=\operatorname{diag}\ \!\bigl(\phi'(\mathbf{W}\mathbf{x}+\mathbf{U}\mathbf{u})\bigr).
\)
Then  
  $s_{\min}\bigl(J(\mathbf{x})\bigr)\;\ge\;
  \sqrt{\frac{m}{L}}\Bigl[(1-\alpha)+\alpha m\,\lambda_{\min}(\mathbf{W})\Bigr]$
and  
  $\|J(\mathbf{x})\|\;\le\;
  \sqrt{\frac{L}{m}}\Bigl[(1-\alpha)+\alpha L\,\rho(\mathbf{W})\Bigr]$.
 \textit{(Proof in supplementary)}

\end{lemma}

\begin{theorem}[State–divergence lower bound]\label{thm:state-divergence}
    Let  \(\phi\) be twice differentiable, monotone, and \emph{strictly} bounded:
\(0<m\le\phi'(z)\le L\) on the reachable domain.  
Define
\begin{equation}\label{T2}
  \beta(\sigma):=\sqrt{\tfrac{m}{L}}\,
  \Bigl[(1-\alpha)+\alpha m\varrho\,
        \frac{\lambda_{\min}(\mathbf{\widetilde W}(\sigma))}
             {\rho(\mathbf{\widetilde W}(\sigma))}\Bigr]
\end{equation}
Let two input streams coincide up to time \(t_0-1\), differ at \(t_0\)
(\(\mathbf{u}_{t_0}\neq \mathbf{v}_{t_0}\)), and coincide thereafter.
For the resulting reservoir trajectories \(\mathbf{x}_t,\mathbf{y}_t\) put
\(\mathbf{e}_t:=\mathbf{x}_t-\mathbf{y}_t\) and set \(\Delta \mathbf{u}:=\mathbf{U}(\mathbf{u}_{t_0}-\mathbf{v}_{t_0})\neq 0\).
If \(\beta(\sigma)>1\), then for every integer \(\tau\ge 1\),
$\|\mathbf{e}_{t_0+\tau}\|
 \;\ge\;
 \alpha m \|\Delta \mathbf{u}\| \,\beta(\sigma)^{\tau-1}$.
\end{theorem}

\begin{proof}



For the initial error injection at $t_0$, 
because the inputs agree up to $t_0-1$ we have $\mathbf{e}_{t_0}=0$.
At $t_0$, 
\(
\mathbf{e}_{t_0+1}
  =(1-\alpha)\mathbf{e}_{t_0}
  +\alpha\ \!\bigl[\phi(\mathbf{W}\mathbf{x}_{t_0}+\mathbf{U}\mathbf{u}_{t_0})
                 -\phi(\mathbf{W}\mathbf{x}_{t_0}+\mathbf{U}\mathbf{v}_{t_0})\bigr],
\)
and the mean–value theorem plus $\phi'\ge m$ yields
$\|\mathbf{e}_{t_0+1}\|\ge\alpha m\|\Delta \mathbf{u}\|$
with $\Delta \mathbf{u}:=\mathbf{U}(\mathbf{u}_{t_0}-\mathbf{v}_{t_0})$.

For propagation for $\tau\ge 2$,
\(
\mathbf{e}_{t_0+\tau}=J(\xi_{t_0+\tau-1})\,\mathbf{e}_{t_0+\tau-1},
\)
where $\xi_{t}$ lies on the segment between $\mathbf{x}_t$ and $\mathbf{y}_t$.
Invoking the lower singular-value bound in Lemma \ref{lemma2} gives
$\|\mathbf{e}_{t_0+\tau}\|\ge\beta(\sigma)\|\mathbf{e}_{t_0+\tau-1}\|$.
Iterating and bootstrapping yields
\(
\|\mathbf{e}_{t_0+\tau}\|
  \;\ge\;
  \alpha m\|\Delta \mathbf{u}\|\,
  \beta(\sigma)^{\tau-1}.
\)
Ensuring $\beta(\sigma)>1$, and
combining with \eqref{T2} gives
$\beta(\sigma) \geq
\sqrt{\tfrac{m}{L}}\,
\Bigl[(1-\alpha)+\alpha m\varrho\,
\frac{1-(N-1)e^{-\delta/\sigma}}
     {1+(N-1)e^{-\delta/\sigma}}\Bigr]$.
For fixed $m,L,\alpha,\varrho$ this expression is strictly increasing as
$\sigma\downarrow 0$ because the fraction inside the brackets rises to~1.
Thus one can always pick a kernel width $\sigma>0$ small enough that $\beta(\sigma)>1$, guaranteeing exponential separation while the spectral normalisation $\varrho$ keeps the echo–state property intact. Ergo, one can always choose \(\sigma\) small enough (yet
positive) so that \(\beta(\sigma)>1\); the kernel width thus becomes a
tunable knob controlling the guaranteed expansion.
\end{proof}

 Theorem \ref{thm:state-divergence} crystallises HypER’s intuition: a hyperbolic kernel with width $\sigma$ and subsequent spectral rescale to $\varrho<1$ yields a Jacobian whose worst-case gain is
$s_{\min}\bigl(J(\mathbf{x})\bigr)\ge\beta(\sigma)=\sqrt{m/L}\bigl[(1-\alpha)+\alpha m\varrho\,\lambda_{\min}(\mathbf{\widetilde W})/\rho(\mathbf{\widetilde W})\bigr]$.
Because $\lambda_{\min}(\mathbf{\widetilde W})/\rho(\mathbf{\widetilde W})$ grows \emph{exponentially} as $\sigma\!\downarrow\!0$ in the Poincaré ball, one can always push $\beta(\sigma)\!>\!1$; every infinitesimal perturbation is then amplified, faithfully reflecting the positive Lyapunov exponent of chaotic flows, yet global stability is preserved by the same $\varrho$ that enforces the echo-state property.  In Euclidean reservoirs this ratio shrinks polynomially with size, forcing $\beta(\sigma)\le 1$ and thus damping, not stretching, differences.  The bound also uncovers how node non-linearities condition expansion through the factor $\sqrt{m/L}$: a heterogeneous palette broadens the reservoir’s functional basis only when $\beta(\sigma)\!>\!1$; otherwise their contributions vanish or destabilise the network.  Hyperbolic connectivity therefore supplies the sole geometry that simultaneously guarantees per-step expansion, respects ESP, and lets diverse activations add predictive power, explaining why mixed non-linearities help HypER yet hurt or do nothing for flat-geometry ESNs.

\begin{algorithm}[t]
\caption{\textsc{HypER} Construction}
\label{alg:hyer}
\begin{algorithmic}[1]
\REQUIRE reservoir size $N$, manifold dimension $d>1$, kernel width $\sigma$, sparsity level $\kappa$, target spectral radius $\varrho<1$
\ENSURE sparse, spectrally‑normalised adjacency matrix $\mathbf W\in\mathbb R^{N\times N}$

\STATE \textit{// Hyperbolic node sampling}
\FOR{$i \gets 1$ \TO $N$}
    \STATE draw $u_i \sim \mathcal U(0,1)$ \COMMENT{inverse‑CDF for uniform hyperbolic volume}
    \STATE $\rho_i \gets \operatorname{arcosh}\ \!\bigl(u_i(\cosh\rho_{\max})-u_i+1\bigr)$
    \STATE $r_i \gets \tanh\ \!\bigl(\rho_i/2\bigr)$
    \STATE sample $\boldsymbol\omega_i \in \mathbb S^{d-1}$ uniformly
    \STATE $\mathbf p_i \gets r_i\,\boldsymbol\omega_i$ \COMMENT{node position in $\mathcal B^{d}$}
\ENDFOR

\STATE \textit{// Geometry‑aware weight kernel}
\FORALL{$(i,j)\in\{1,\dots,N\}^2$}
    \STATE $d_{ij} \gets \operatorname{arcosh}\ \!\Bigl(1+2\frac{\|\mathbf p_i-\mathbf p_j\|^2}{(1-\|\mathbf p_i\|^2)(1-\|\mathbf p_j\|^2)}\Bigr)$
    \STATE $W_{ij}\gets\exp\ \!\bigl(-d_{ij}/\sigma\bigr)$
\ENDFOR

\STATE \textit{// Row‑wise sparsification}
\FOR{$i \gets 1$ \TO $N$}
    \STATE keep the $\kappa$ largest entries in row $i$ of $\mathbf W$; set others to $0$
\ENDFOR

\STATE \textit{// Spectral normalisation to satisfy ESP}
\STATE $\rho \gets$ largest eigenvalue modulus of $\mathbf W$
\STATE $\mathbf W \gets (\varrho / \rho)\,\mathbf W$

\RETURN $\mathbf W$
\end{algorithmic}
\end{algorithm}

\begin{table*}[!ht]
\centering
\caption{NRMSE for autoregressive forecasting across multiple prediction horizons on canonical chaotic benchmarks.}
\label{tab:nrmse_horizons}
\resizebox{0.91\textwidth}{!}{
\begin{tabular}{l cc c c c ccc}
\toprule
\multirow{2}{*}{\textbf{Dataset}} & \multirow{2}{*}{\textbf{Horizon}}& \multicolumn{7}{c}{\textbf{NRMSE $\downarrow$}}\\
\cmidrule(lr){3-9}& & ESN& SCR& CRJ& SW-ESN& MCI-ESN &DeepESN &HypER\\
\midrule
 \multirow{5}{*}{Lorenz}& 200& 0.1050 ± 0.1586& 0.1061 ± 0.2249& 0.0494 ± 0.0243& 0.0224 ± 0.0143& 0.0035 ± 0.0026& 0.2746 ± 0.4853&\textbf{0.0002 ± 0.0001}\\
 & 400& 1.1795 ± 0.3572& 1.1718 ± 0.6335& 0.9309 ± 0.2552& 0.6417 ± 0.2895& 0.3086 ± 0.3725& 1.2616 ± 0.4225&\textbf{0.0124 ± 0.0104}\\
 & 600& 1.7705 ± 0.2058& 1.9862 ± 0.9876& 1.6462 ± 0.1888& 1.4118 ± 0.3359& 0.8031 ± 0.4980& 1.7473 ± 0.2495&\textbf{0.0483 ± 0.0414}\\
 & 800& 1.9840 ± 0.1512& 2.3832 ± 1.2483& 1.8971 ± 0.2022& 1.7489 ± 0.2176& 1.4190 ± 0.2962& 1.9723 ± 0.1619&\textbf{0.8865 ± 0.2899}\\
 & 1000& 2.1093 ± 0.2033& 2.5817 ± 1.4351& 2.0490 ± 0.3540& 1.8956 ± 0.2003& 1.6795 ± 0.2205&  2.0663 ± 0.1326&\textbf{1.2580 ± 0.2721}\\

 \midrule
 \multirow{5}{*}{Rössler}& 200& 0.0036 ± 0.0006& 0.0053 ± 0.0023& 0.0066 ± 0.0023& 0.0063 ± 0.0015& 0.0038 ± 0.0017& 0.0059 ± 0.0010&\textbf{0.0007 ± 0.0003}\\
 & 400& 0.0637 ± 0.1875& 0.0664 ± 0.1317& 0.0319 ± 0.0250& 0.0289 ± 0.0215& 0.0083 ± 0.0042& 0.0374 ± 0.0341&\textbf{0.0019 ± 0.0008}\\
 & 600& 0.0936 ± 0.2846& 0.1222 ± 0.2882& 0.0481 ± 0.0387& 0.0470 ± 0.0364& 0.0113 ± 0.0064& 0.0542 ± 0.0583&\textbf{0.0024 ± 0.0010}\\
 & 800& 0.1603 ± 0.4714& 0.2134 ± 0.5601& 0.0629 ± 0.0486& 0.0700 ± 0.0574& 0.0194 ± 0.0245& 0.0735 ± 0.0911&\textbf{0.0033 ± 0.0014}\\
 & 1000& 0.1963 ± 0.5000& 0.2494 ± 0.4606& 0.0812 ± 0.0499& 0.0932 ± 0.0830& 0.0312 ± 0.0509& 0.1095 ± 0.1243&\textbf{0.0061 ± 0.0031}\\
 \midrule
 \multirow{5}{*}{Chen-Ueta}& 200
& 1.6613 ± 0.7941& 1.6366 ± 0.4114& 1.6308 ± 1.0914& 1.5562 ± 0.2400& 1.1049 ± 0.3810& 1.8560 ± 0.5049&\textbf{0.0259 ± 0.0244}\\
 & 400
& 2.3164 ± 1.0948& 2.3481 ± 0.8614& 2.3479 ± 1.6356& 2.0338 ± 0.2648& 1.8995 ± 0.3987& 2.5622 ± 0.7545&\textbf{0.8992 ± 0.3583}\\
 & 600
& 2.6037 ± 1.1631& 2.7002 ± 1.1561& 2.5733 ± 1.8155& 2.2990 ± 0.1988& 2.1915 ± 0.3383& 2.7744 ± 0.9136&\textbf{1.5692 ± 0.3497}\\
 & 800
& 2.6638 ± 1.1837& 2.9429 ± 1.3532& 2.6705 ± 1.8522& 2.3862 ± 0.2291& 2.3021 ± 0.2791& 2.8746 ± 0.9632&\textbf{1.9195 ± 0.2265}\\
 & 1000& 2.6851 ± 1.1655& 3.0044 ± 1.5006& 2.7412 ± 1.8404& 2.3884 ± 0.1746& 2.3194 ± 0.2103& 2.8407 ± 0.9803&\textbf{2.0413 ± 0.1697}\\
 \bottomrule
\end{tabular}
}
\end{table*}

\begin{table*}[!ht]
\centering
\caption{NRMSE for forecasting tasks on real-world datasets ( $^*$ denotes open loop settings).}
\label{tab:nrmse_horizons_auto}
\resizebox{\textwidth}{!}{
\begin{tabular}{l cc c c c ccc}
\toprule
\multirow{2}{*}{\textbf{Dataset}} & \multirow{2}{*}{\textbf{Horizon}}& \multicolumn{7}{c}{\textbf{NRMSE $\downarrow$}}\\
\cmidrule(lr){3-9}& & ESN& SCR& CRJ& SW-ESN& MCI-ESN &DeepESN &HypER\\

\midrule
\multirow{3}{*}{MIT–BIH}& 200* &       2.3520 ± 0.5484 &  2.0300 ± 0.5225 &  1.5681 ± 0.3373 &  1.9976 ± 0.3539 &  1.1408 ± 0.0604 &  2.7377 ± 0.9642 &  \textbf{0.7412 ± 0.2972} \\   
& 500* &       1.7548 ± 0.3699 &  1.5411 ± 0.3469 &  1.2561 ± 0.2130 &  1.5351 ± 0.2293 &  1.0512 ± 0.0328 &  2.0153 ± 0.6588 &  \textbf{0.6735 ± 0.2018} \\   
& 1000* &       1.4852 ± 0.2254 &  1.3715 ± 0.2000 &  1.2103 ± 0.1132 &  1.3543 ± 0.1299 &  1.1262 ± 0.0156 &  1.6513 ± 0.4127 &  \textbf{0.7528 ± 0.1217} \\    

\midrule
\multirow{3}{*}{Sunspot (monthly)} 
&200*   &     0.4986 ± 0.0035  & 0.5626 ± 0.0215 &  0.5160 ± 0.0068  & 0.5041 ± 0.0036 &  \textbf{0.4870 ± 0.0003} &  0.4881 ± 0.0092  & 0.4903 ± 0.0035 \\   
&500*   &     0.3702 ± 0.0021  & 0.4197 ± 0.0105 &  0.3825 ± 0.0032  & 0.3731 ± 0.0020 &  0.3727 ± 0.0006 &  \textbf{0.3647 ± 0.0033}  & 0.3819 ± 0.0052 \\   
&1000*  &     0.3560 ± 0.0028  & 0.4143 ± 0.0117 &  0.3833 ± 0.0084  & 0.3643 ± 0.0049 &  \textbf{0.3361 ± 0.0002} &  0.3562 ± 0.0063  & 0.3488 ± 0.0033 \\      

\midrule
\multirow{3}{*}{Santa Fe Laser} 
&200*        &0.3014 ± 0.0046  & 0.3345 ± 0.0203 &  0.3127 ± 0.0103 &  0.3043 ± 0.0035&   0.3421 ± 0.0018 &  0.3057 ± 0.0074 &  \textbf{0.2853 ± 0.0123} \\   
&500*        &0.3107 ± 0.0015  & 0.3196 ± 0.0065 &  0.3120 ± 0.0036 &  0.3122 ± 0.0014&   0.3260 ± 0.0007 &  0.3105 ± 0.0026 &  \textbf{0.2905 ± 0.0077} \\   
&1000*       &0.2615 ± 0.0013  & 0.3046 ± 0.0194 &  0.2636 ± 0.0069 &  0.2529 ± 0.0011&   0.2737 ± 0.0004 &  0.2590 ± 0.0040 &  \textbf{0.2451 ± 0.0074} \\   

\midrule
\multirow{3}{*}{Sunspot (monthly)} 
&5     &     2.7590 ± 0.1903 &    3.0103 ± 0.6578  &   2.9065 ± 0.4330 &    \textbf{2.5280 ± 0.2193} &    2.9715 ± 1.0825&    2.6047 ± 0.2740   &  2.5474 ± 0.5361\\ 
&10    &     2.9977 ± 0.2709 &    2.7096 ± 0.4210  &   2.4177 ± 0.1901 &    2.6527 ± 0.3082 &    2.6360 ± 0.6055&    2.4866 ± 0.4580   &  \textbf{2.3770 ± 0.5746}\\ 
&15    &     3.9134 ± 0.4048 &    3.3159 ± 0.6076  &   2.9029 ± 0.2172 &    3.4978 ± 0.5451 &    3.0750 ± 0.7089&    3.0730 ± 0.7431   &  \textbf{2.8761 ± 0.8182}\\ 

\midrule
\multirow{3}{*}{Santa Fe Laser} 
&5  & 2.0033 ± 0.3560 & 1.6533 ± 0.3241& \textbf{1.1579 ± 0.4128} &  2.0322 ± 0.2692 &  2.3768 ± 0.1001 &  1.8717 ± 0.3477 & 1.4738 ± 0.3298 \\          
&10 & 2.2450 ± 0.4354 & 1.7234± 0.2513& 1.7210 ± 0.3957 &  2.1163 ± 0.2813 &  2.3586 ± 0.1106 &  2.1346 ± 0.3688 & \textbf{1.7108 ± 0.3346} \\          
&15 & 2.1451 ± 0.3156 & 2.1405 ± 0.2713& 2.4815 ± 0.9362 &  2.1235 ± 0.1442 &  2.1557 ± 0.0746 &  2.2228 ± 0.2123 & \textbf{1.9747 ± 0.4748} \\          

\bottomrule
\end{tabular}
}
\end{table*}

\paragraph{Polynomial Readout.} Following standard ESN methodology, we \emph{train} only an output layer that maps \(\mathbf{x}_t\) to the next state \(\mathbf{u}_{t+1}\). During training, the input signal \(\{\mathbf{u}_t\}_{t=1}^T\) is fed into the reservoir and the corresponding reservoir states \(\{\mathbf{x}_t\}_{t=T_w+1}^T\) are recorded after the washout. We adopt a polynomial expansion for the readout
  $\boldsymbol{\xi}_t
  \;=\;\bigl[\,x^1_t,\dots,x^N_t,\,(x^1_t)^2,\dots,(x^N_t)^2,\;1\,\bigr]^{\mathsf{T}},$
where the squared terms add a limited second-order nonlinearity, and the constant \(1\) term captures any bias \cite{Carroll2022, Ohkubo2024}. Hence, \(\boldsymbol{\xi}_t\in\mathbb{R}^{2N+1}\) and the readout matrix \(\mathbf{W}_{\mathrm{out}}\in\mathbb{R}^{m\times(2N+1)}\). The single-step-ahead predicted output is expressed as
  $\hat{\mathbf{u}}_{t+1}
  \;=\;\mathbf{W}_{\mathrm{out}}\;\boldsymbol{\xi}_t$.
After collecting \(\{\boldsymbol{\xi}_t, \mathbf{u}_{t+1}\}\) for \(t=T_{\mathrm{w}}+1,\dots,T\), we solve the \emph{ridge regression} problem
  $\min_{\mathbf{W}_{\mathrm{out}}}
  \;\sum_{t=T_{\mathrm{w}}+1}^{T}
    \bigl\| \mathbf{u}_{t+1} \;-\;\mathbf{W}_{\mathrm{out}}\;\boldsymbol{\xi}_t\bigr\|^{2}
   \;+\;\lambda\,\|\mathbf{W}_{\mathrm{out}}\|^{2}_{F},$
where \(\|\cdot\|_{F}\) denotes the Frobenius norm, and \(\lambda>0\) is the $L_2$-regularization coefficient. The closed-form analytical solution for the ridge regression yields
  $\mathbf{W}_{\mathrm{out}}^*
  \;=\;\mathbf{Y}\,\mathbf{\Xi}^{\mathsf{T}} (\mathbf{\Xi} \mathbf{\Xi}^{\mathsf{T}} + \lambda I)^{-1}$,
where \(\mathbf{Y}\) stacks the row vectors \(\mathbf{u}_{t+1}\), \(\mathbf{\Xi}\) stacks the column features \(\boldsymbol{\xi}_t\) and $I$ is the $(2N+1) \times (2N+1)$ identity matrix. The regularization term $\lambda I$ ensures that $\mathbf{\Xi} \mathbf{\Xi}^{\mathsf{T}} + \lambda I$ is invertible, thus guaranteeing a unique solution, while also penalizing overfitting. 

\paragraph{Operating Protocols.}
\emph{Washout:}  We start from $\mathbf{x}_0=\mathbf 0$ and iterate the leaky update for $T_{\mathrm w}$ steps without collecting data to allow transients to decay.

\emph{Teacher forcing:}  During training and validation, the true signal $\mathbf u_t$ is injected at every step, so the read-out learns from single-step errors while the reservoir itself remains open-loop and perfectly stabilised.

\emph{Autoregressive forecasting:}  For test-time prediction we close the loop, feeding the network’s own output $\hat{\mathbf u}_{t+1}=\mathbf{W}_{\text{out}}\mathbf{\xi}_t$ back as the next input.  This autonomous mode probes long-horizon fidelity: any geometric or spectral mismatch inside $W$ accumulates multiplicatively, so gains in Lyapunov-time accuracy translate directly into extended prediction windows.




\section{Experiments and Discussion}\label{experiments}

\paragraph{Setup.}
To evaluate the effectiveness of  HypER, we conduct a series of experiments on several canonical chaotic systems—standard benchmarks in nonlinear time series modeling— including the butterfly-shaped \emph{Lorenz–63} attractor (\(\lambda_{\max}\!\approx\!0.905\)) \cite{lorenz1963deterministic},  the slower \emph{Rössler} scroll (\(\lambda_{\max}\!\approx\!0.071\)) \cite{rossler1976equation}, the kink-dominated double scroll of \emph{Chua’s circuit} \cite{chua1986double}, the hyper-chaotic \emph{Chen–Ueta} flow with two positive exponents \cite{chen1999yet} and the infinite-dimensional delay-differential \emph{Mackey–Glass} system \cite{mackey1977oscillation}. All benchmark trajectories are generated with the LSODA integrator in \texttt{SciPy}’s \verb|odeint| at a fixed step size $\Delta t = 0.02$, producing $12,500$ samples per system.  The first $2000$ points are discarded as wash‑out to mitigate the effects of transient dynamics, the next $80\%$ used for training, and the final $20\%$ reserved for evaluation. Specifically, we integrate the Lorenz $(\sigma=10,\rho=28,\beta=\tfrac83;\,x_0=y_0=z_0=1.0)$, Rössler $(a=0.2,b=0.2,c=5.7;\,x_0=0.0,y_0=1.0,z_0=0.0)$, Chua $(\alpha=15.6,\beta=28,m_0=-1.143,m_1=-0.714;\,x_0=0.7,y_0=z_0=0.0)$ and Chen-Ueta $(a=35,b=3,c=28;\,x_0=y_0=z_0=0.1)$ systems, parameter regimes that ensure sustained chaotic behaviour in all benchmarks.

We benchmark HypER against several baseline models including ESN \cite{jaeger2001echo}, SCR \cite{li2024simplecyclereservoirsuniversal}, CRJ \cite{rodan2011simple}, SW-ESN \cite{kawai2019smallworld}, MCI-ESN \cite{liu2024minimum} and DeepESN \cite{gallicchio2020deepechostatenetwork}. All of these are single-reservoir models, except MCI-ESN, which employs two interconnected reservoirs, and DeepESN, which stacks multiple reservoirs in a hierarchical structure. The input weights $\mathbf{U}$ are drawn from a zero-meaned Gaussian distribution clipped to a small symmetric interval, for all models. We instantiate all on equal‑footing with 300  units (three $100$‑unit layers for DeepESN). Global hyperparameters—input weights $\mathbf{U}\!\sim\!\mathcal{N}(0,0.2^{2})$, spectral radius $\rho(\mathbf{W})=0.99$, leak rate $\alpha=0.8$ and ridge coefficient $\lambda=10^{-5}$—were selected through budgeted hyperparameter search tailored to the Lorenz system; analogous tuning was performed for other datasets. Model‑specific hyperparameters were also optimized using the same procedure \textit{(Further details are laid out in the supplementary file)}. All metrics are averaged over thirty independent random seeds. 


\begin{table}[!ht]
\centering
\caption{Normalized VPT and ADev for autoregressive forecasting over a $1000$-step horizon.}
\label{tab:vpt_adev}
\resizebox{0.37\textwidth}{!}{
\begin{tabular}{l lc c }
\toprule
{\textbf{Dataset}}& \textbf{Model}&\textbf{Norm. VPT $\uparrow$}&\textbf{ADev $\downarrow$}\\
\midrule
 \multirow{7}{*}{Lorenz}& ESN& 5.285 ± 0.72& 51.43 ± 17.56\\
 & SCR& 5.437 ± 1.19&51.67 ± 13.29\\
 & CRJ& 5.558 ± 0.18&43.83 ± 13.39\\
 & SW-ESN& 5.927 ± 0.95&35.70 ± 9.86\\
 & MCI-ESN& 8.049 ± 2.36&33.50 ± 14.57\\
 & DeepESN& 4.883 ± 1.15&54.43 ± 22.81\\
 & HypER& \textbf{12.215 ± 1.23}&\textbf{19.67 ± 6.82}\\

 \midrule
 \multirow{7}{*}{Rössler}& ESN
& 3.442 ± 1.42& 4.27 ± 6.23\\
 & SCR
& 2.857 ± 1.72&7.97 ± 6.02\\
 & CRJ
& 3.150 ± 1.15&3.03 ± 1.52\\
 & SW-ESN
& 3.104 ± 1.17&3.10 ± 1.87\\
 & MCI-ESN
& 3.427 ± 0.68&1.80 ± 1.52\\
 & DeepESN
& 3.177 ± 1.35&2.10 ± 1.63\\
 & HypER& \textbf{5.142 ± 0.76}&\textbf{1.27 ± 1.17}\\
 \midrule
 \multirow{7}{*}{Chen-Ueta}& ESN
& 2.006 ± 0.34&121.77 ± 20.48\\
 & 
SCR& 1.946 ± 0.41&130.20 ± 25.69\\
 & CRJ& 1.934 ± 0.33&101.23 ± 17.04\\
 & 
SW-ESN& 2.107 ± 0.18&104.87 ± 14.99\\
 & MCI-ESN& 2.628 ± 0.44&97.57 ± 14.77\\
 & 
DeepESN& 1.635 ± 0.60&138.37 ± 40.28\\
 & HypER& \textbf{5.067 ± 0.70}&\textbf{80.60 ± 9.25}\\
 \bottomrule
\end{tabular}
}
\end{table}

\begin{table}[!ht]
\centering
\caption{Ablation on \emph{manifold geometry}, \emph{embedding dimension}, and—in the Poincaré disc setting—\emph{node-sampling strategy} ($1000$-step Lorenz forecasting). Here, `Euc. / Uni.' refers to uniform sampling in Euclidean space; `Hyp. / Euc.-Iso' refers to Euclidean-volume isotropic sampling but in hyperbolic space; `Hyp. / Hyp.-Uni.' refers to uniform sampling in hyperbolic space.}
\label{tab:ablation_full}
\resizebox{0.47\textwidth}{!}{
\begin{tabular}{lcccc}
\toprule
\textbf{Manifold / Sampling}&\textbf{Dim.\,$d$}& \textbf{NRMSE$\downarrow$} & \textbf{Norm. VPT$\uparrow$} & \textbf{ADev$\downarrow$} \\
\midrule
Euc. / Uni. &2& 1.3981& 11.191& 23.13\\
Hyp. / Euc.–Iso.&2& 1.3540& 11.947& 23.43\\
Hyp. / Hyp.–Uni. &2& \textbf{1.2580}& \textbf{12.215}& \textbf{19.67}\\
Hyp. / Hyp.–Uni. &3& 1.3702& 11.192& 21.37\\
Hyp. / Hyp.–Uni. &4& 1.4639& 10.817& 22.03\\
\bottomrule
\end{tabular}
}
\end{table}

\begin{table}[!ht]
\centering
\caption{Performance comparison of models on the Lorenz dataset using mixed heterogeneous node-wise activation functions.  MCI-ESN is excluded as its built-in complementary sine-cosine activations preclude a fair comparison.}
\label{tab:vpt_lorenz_tanh_mixed}
\resizebox{0.33\textwidth}{!}{
\begin{tabular}{lccc}
\toprule
\textbf{Model} & \textbf{NRMSE} $\downarrow$& \textbf{Norm. VPT} $\uparrow$ &\textbf{ADev} $\downarrow$\\
\midrule
ESN        & 1.9436& 5.933&42.57\\
SCR        & 1.7270& 7.086&32.30\\
CRJ        & 1.6718& 7.423&31.97\\
SW-ESN     & 1.4974& 9.636&26.20\\
DeepESN    & 2.0225& 5.257&40.53\\
HypER & \textbf{1.2580}& \textbf{12.215}&\textbf{19.67}\\
\bottomrule
\end{tabular}
}
\end{table}

\begin{table}[!ht]
\centering
\caption{Ablation over HypER parameters ($1000$-step autoregressive forecasting of the Lorenz dataset).}
\label{tab:ablation_hyper_params}
\resizebox{0.46\textwidth}{!}{
\begin{tabular}{llccc}
\toprule
 \textbf{Parameter}&\textbf{Setting}& \textbf{NRMSE$\downarrow$} & \textbf{Norm. VPT$\uparrow$} & \textbf{ADev$\downarrow$}\\
\midrule
 \multirow{5}{*}{Kernel Width}&$\sigma = 0.05$& 1.3421& 12.052& 24.57\\
 & $\sigma = 0.1$& \textbf{1.2580}& \textbf{12.215}&\textbf{19.67}\\
 &$\sigma = 0.2$& 1.3311& 12.101& 20.87\\
 &$\sigma = 0.3$& 1.3668& 11.591& 23.63\\
 &$\sigma = 0.5$& 1.4257& 11.524& 23.60\\
\midrule
 \multirow{5}{*}{Row-level Sparsity}&$\kappa = 10$& 1.2827& 12.210& 19.77\\
 &$\kappa = 20$& 1.2923& 12.211& 20.67\\
 &$\kappa = 40$& \textbf{1.2580}& \textbf{12.215}& 19.67\\
 &$\kappa = 60$& 1.2665& 12.213& 18.67\\
 &$\kappa = 80$& 1.2761& 12.212& \textbf{17.70}\\
\bottomrule
\end{tabular}
}
\end{table}

\paragraph{Metrics.}
\textit{Normalized Root Mean Squared Error (NRMSE):} We assess the model’s prediction accuracy using NRMSE. Given a true trajectory \(\{\mathbf{u}_t\}_{t=1}^T \subset \mathbb{R}^m\) and a model-predicted trajectory \(\{\hat{\mathbf{u}}_t\}_{t=1}^T\), the NRMSE is defined as
    $\sqrt{\frac{\sum_{t=1}^T \|\mathbf{u}_t\;-\;\hat{\mathbf{u}}_t\|^2}{\sum_{t=1}^T \|\mathbf{u}_t\;-\;\overline{\mathbf{u}} \|^2}}$.

\textit{Valid Prediction Time (VPT):} 
To quantify long-horizon accuracy in a chaotic setting, we compute \emph{valid prediction time} \(T_{\mathrm{VPT}}\) as follows. We first define the time-averaged variance of \(\mathbf{u}_t\) by centering about its mean \(\overline{\mathbf{u}} \in \mathbb{R}^m\) and computing  
  $\bigl\langle\|\mathbf{u}_t-\overline{\mathbf{u}}\|^2\bigr\rangle$ 
 over the entire prediction horizon. Next, we define the \emph{normalized prediction error} at time \(t\) by
  $\delta(t)
  \;=\; 
  \frac{\|\mathbf{u}_t\;-\;\hat{\mathbf{u}}_t\|^2}
       {\bigl\langle\|\mathbf{u}_t-\overline{\mathbf{u}}\|^2\bigr\rangle}$.
For a task-specific threshold \(\theta\) (we use \ \(\theta=0.4\)  \cite{Pathak2018ModelFree}), \(T_{\mathrm{VPT}}\) is the earliest time \(t\) at which \(\delta_t\) exceeds \(\theta\). If \(\delta_t\le\theta\) for all \(t\), then \(T_{\mathrm{VPT}}\) is taken to be the final available time. Finally, to relate \(T_{\mathrm{VPT}}\) to the system’s characteristic divergence, we introduce the \emph{Lyapunov time} \(T_{L}=1/\lambda_{\max}\), where \(\lambda_{\max}>0\) is the LLE. The ratio \(\tfrac{T_{\mathrm{VPT}}}{T_{L}}\) indicates how many Lyapunov $e$-foldings the model’s predictions remain within the acceptable error threshold.

\textit{Attractor Deviation (ADev):}
To measure how well the predicted trajectory resembles the true one in phase‑space, we partition the domain into a uniform grid of $N_x\times N_y\times N_z$ cubes \cite{Zhai2023ResonanceML}.  For each cube $(i,j,k)$ we record an occupancy indicator  
$\chi^{\text{true}}_{ijk},\chi^{\text{pred}}_{ijk}\in\{0,1\}$,  
equal to 1 if the true (respectively predicted) trajectory visits that cube at least once during the prediction window and 0 otherwise.  The \emph{attractor deviation} is then
  $\mathrm{ADev}
  \;=\;
  \sum_{i=1}^{N_x}\sum_{j=1}^{N_y}\sum_{k=1}^{N_z}
        \bigl|\,
          \chi^{\text{true}}_{ijk}-\chi^{\text{pred}}_{ijk}
        \bigr|$.
ADev counts the number of cubes that are visited by exactly one of the two trajectories (the symmetric‐difference volume).  ADev $=0$ indicates perfect geometric agreement, whereas larger values reflect increasing mismatch.

\textit{Power Spectral Density (PSD):} While previous evaluation metrics assessed the models from a time-domain perspective, we now turn to a \emph{frequency-domain analysis} to evaluate how well the reconstructed trajectory preserves the spectral characteristics of the true signal. Specifically, we examine PSD of each component of the trajectory using \emph{Welch's Method}. Initially, a \emph{Hamming Window} $w_k$ is applied to the discrete signal $z_k = \mathbf{z}(k\,\Delta t)$ to minimize spectral leakage by tapering the edges of the signal. The PSD is then computed as the squared magnitude of the \emph{Fast Fourier Transform (FFT)} of the windowed signal $z_k \cdot w_k$, averaged over all segments
    $S(\omega) = \mathbb{E} \left[ \left| \mathcal{F} \left\{ z_k \cdot w_k \right\} \right|^2 \right]$.

\paragraph{Datasets.}

Our real‐world testbed spans three orders of temporal scale.  The \emph{Sunspot Monthly} series provides a 270-year, quasi-periodic benchmark whose Schwabe and Gleissberg modulations have long served as a litmus for nonlinear predictors; we standardise the SILSO v2.0 index and withhold the last 84 years for out-of-sample scoring \cite{SILSO2020}.  , we use the first 4,500 samples for training and the following 1,000 for testing in our forecasting setup \cite{Jaeger2007leaky,Weigend1994}.  Biomedical variability is probed with the \emph{MIT–BIH Arrhythmia} corpus, where 48 annotated ECG records sampled at 360 Hz are partitioned 80/20; this sequence mixes quasi-periodic sinus segments with abrupt ectopic events, stressing robustness to morphological outliers \cite{Goldberger2000,MoodyMark2001}. The aforesaid datasets are normalized to the $[0,1]$ range and processed using 3-dimensional delay embedding prior to training.  \textit{(details in supplementary)}

\begin{figure}[!ht]
    \centering
    \includegraphics[width=0.97\linewidth,
                     height=0.12\textheight]{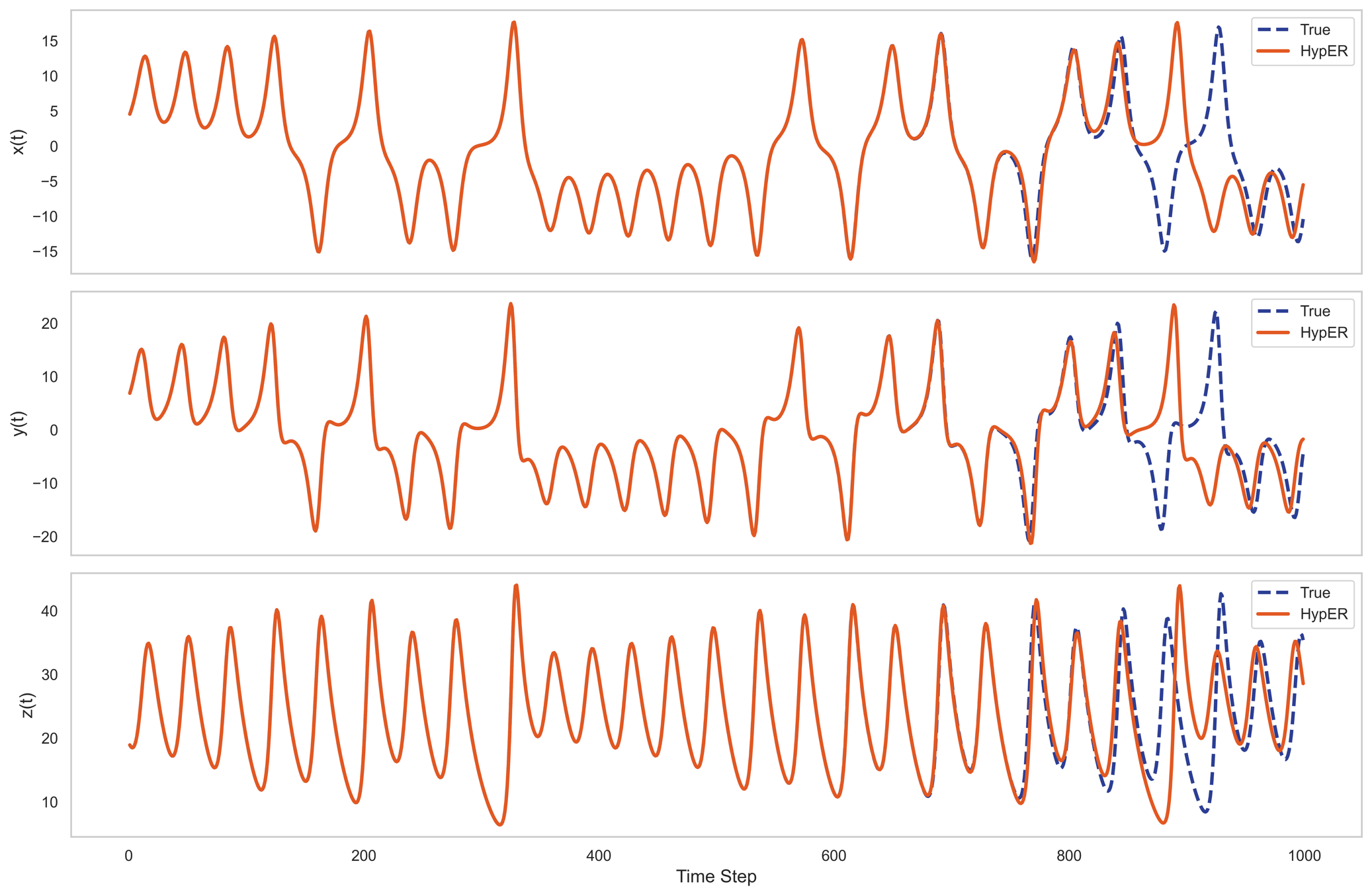}

    \caption{Predicted trajectories by HypER alongside ground truth for the test segment of the Lorenz system under autoregressive forecasting.}
    \label{fig:predicted-trajectories-hyper}
\end{figure}

\begin{figure}[ht]
  \centering

  \begin{minipage}[b]{0.49\linewidth}
    \centering
    \includegraphics[width=0.75\textwidth, height=0.66\textwidth]{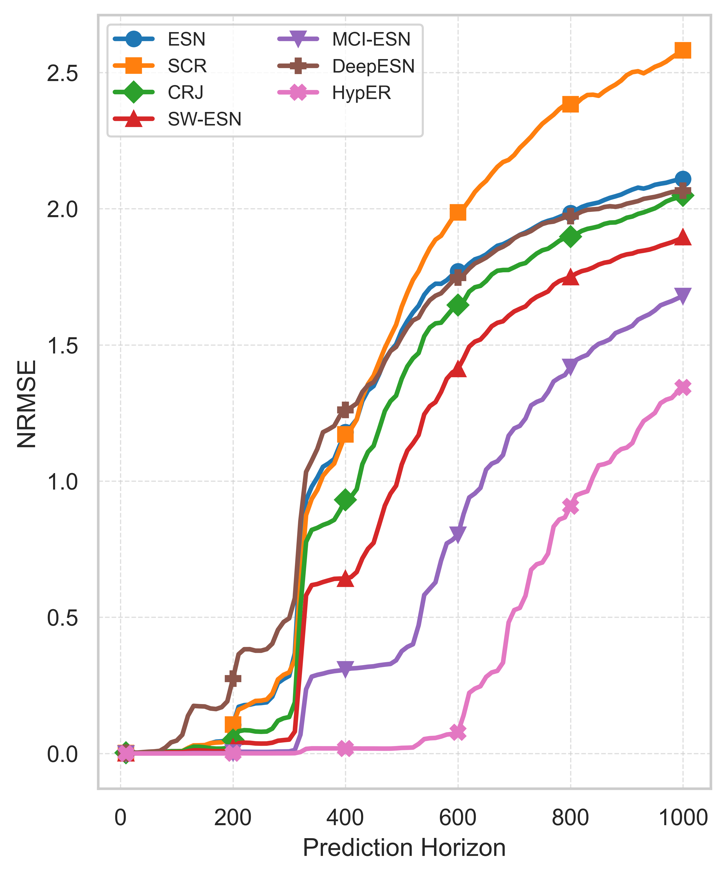}
    \textbf{\\(a)} Lorenz system
    \label{fig:nrmse_lorenz}
  \end{minipage}
  \hfill
  \begin{minipage}[b]{0.49\linewidth}
    \centering
    \includegraphics[width=0.75\textwidth, height=0.66\textwidth]{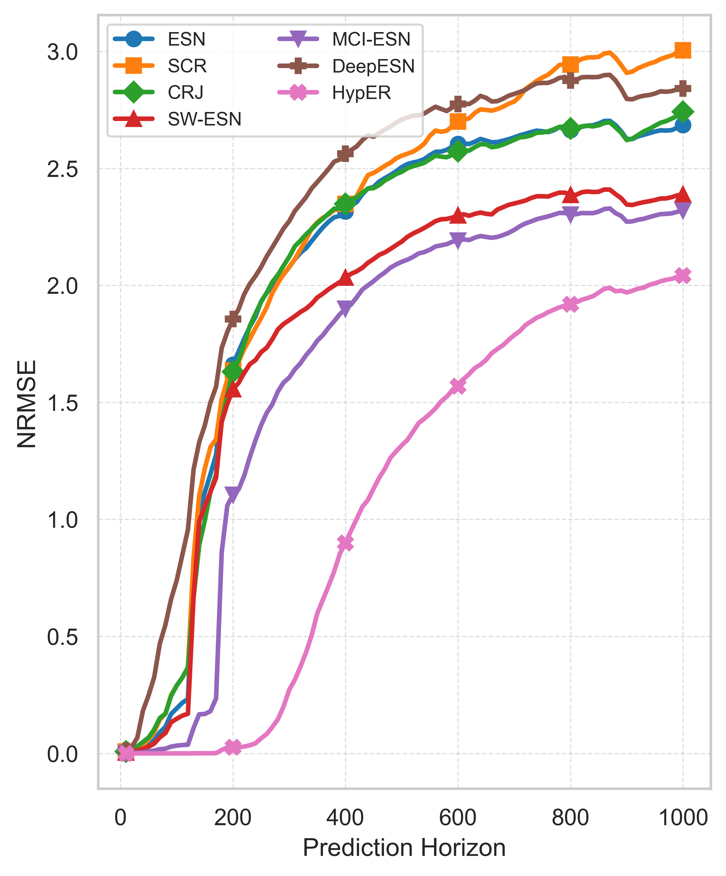}
   \\ \textbf{(b)} Chen-Ueta system
    \label{fig:nrmse_chen}
  \end{minipage}

  \caption{NRMSE for autoregressive predictions across multiple horizons.}
  \label{fig:nrmse}
\end{figure}

\begin{figure}[!ht]
  \centering

  \begin{minipage}[b]{0.49\linewidth}
    \centering
    \includegraphics[width=0.75\textwidth, height=0.67\textwidth]{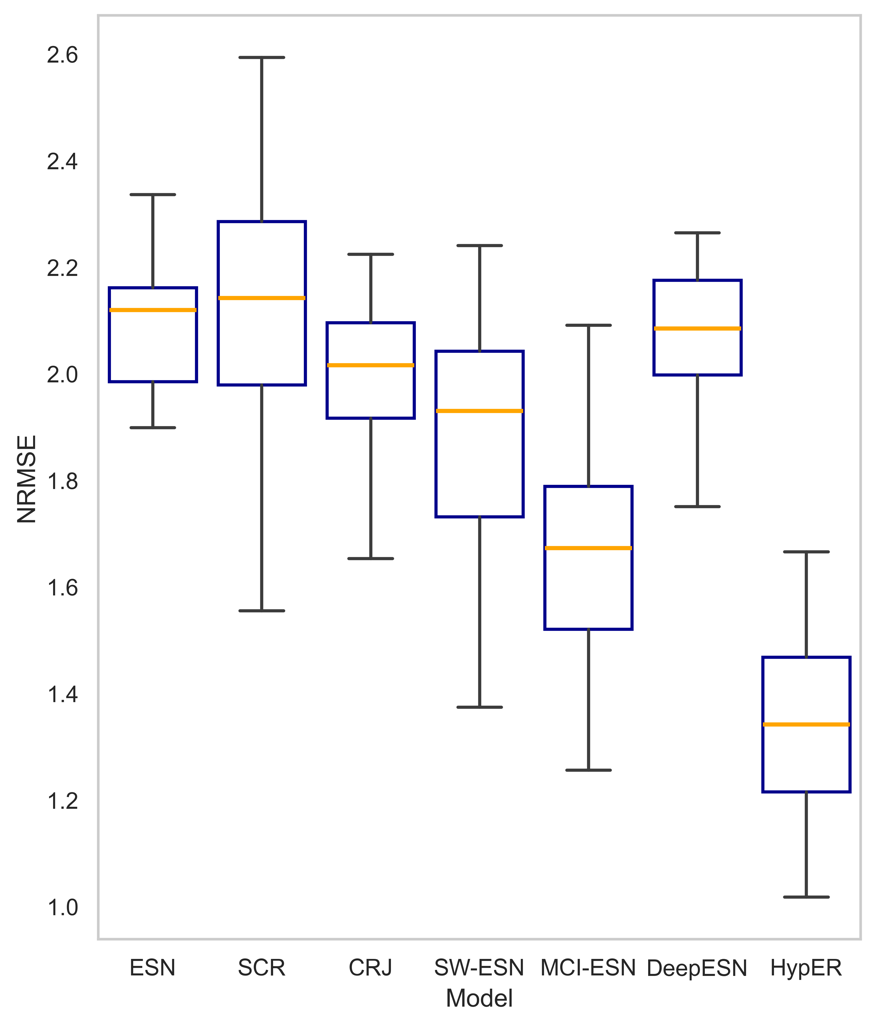}
   \\ \textbf{(a)} Lorenz system
    \label{fig:nrmse_lorenz2}
  \end{minipage}
  \hfill
  \begin{minipage}[b]{0.49\linewidth}
    \centering
    \includegraphics[width=0.75\textwidth, height=0.67\textwidth]{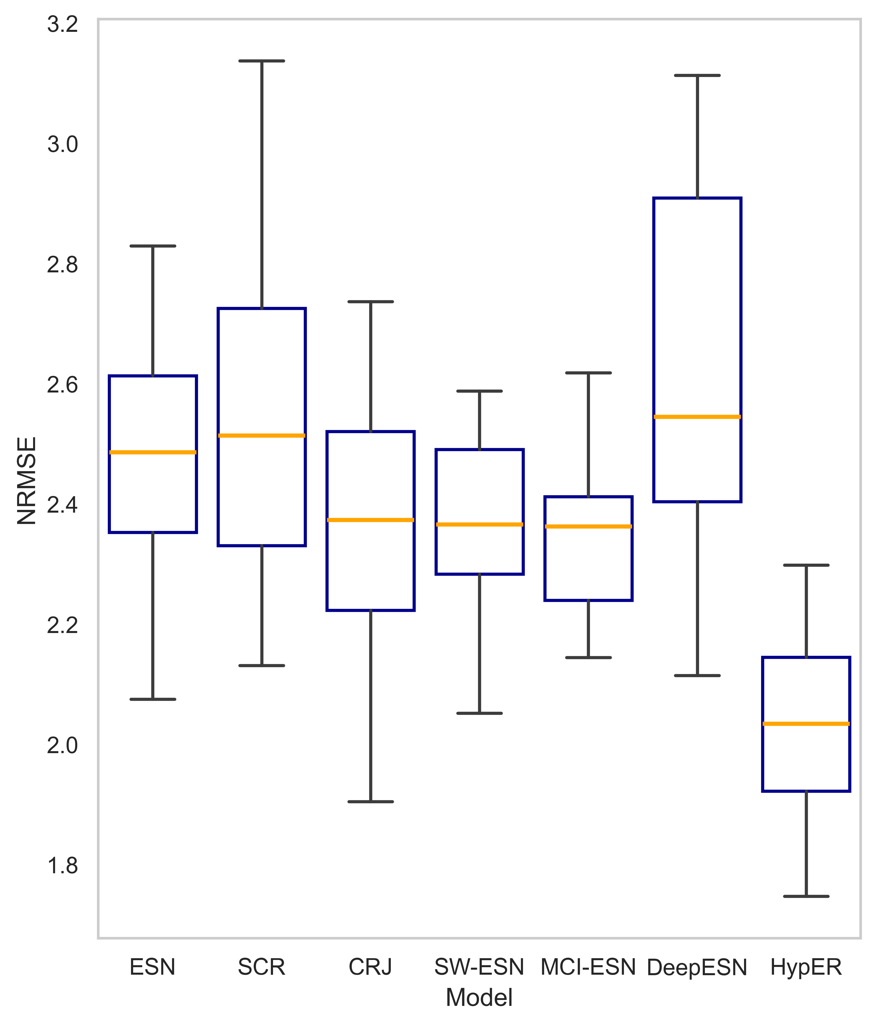}
   \\ \textbf{(b)} Chen-Ueta system
    \label{fig:nrmse_chen2}
  \end{minipage}

  \caption{Boxplots of NRMSE for autoregressive predictions at a 1000-step prediction horizon.}
  \label{fig:nrmse2}
\end{figure}

\begin{figure}[ht]
  \centering

 \begin{minipage}[b]{0.49\linewidth}
    \centering
    \includegraphics[width=\textwidth, height=0.75\textwidth]{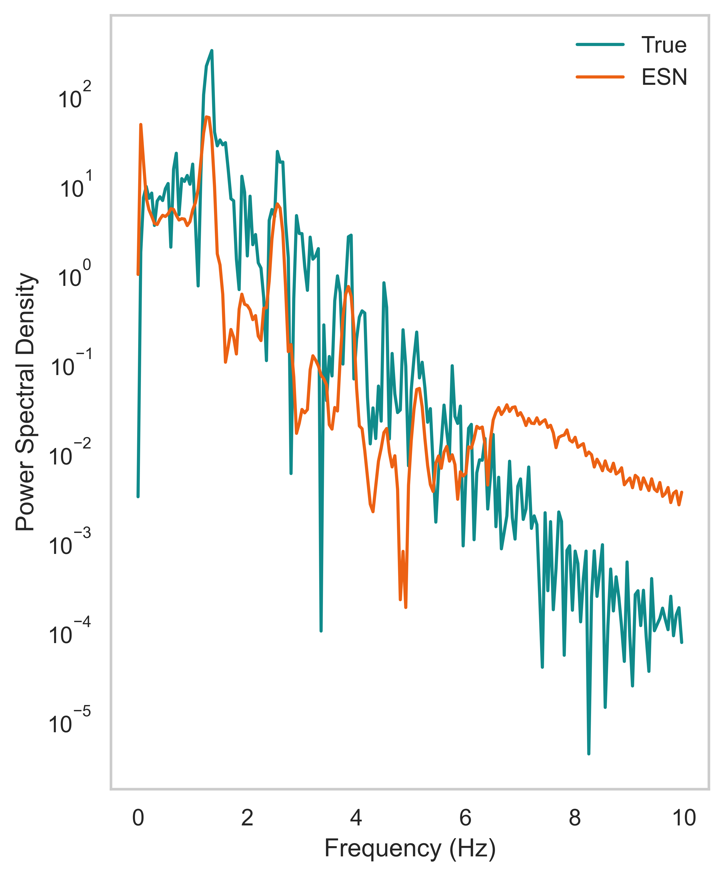}
    \textbf{(a)} Standard ESN
    \label{fig:psd_lorenz_esn}
  \end{minipage}
  \hfill
  \begin{minipage}[b]{0.49\linewidth}
    \centering
    \includegraphics[width=\textwidth, height=0.75\textwidth]{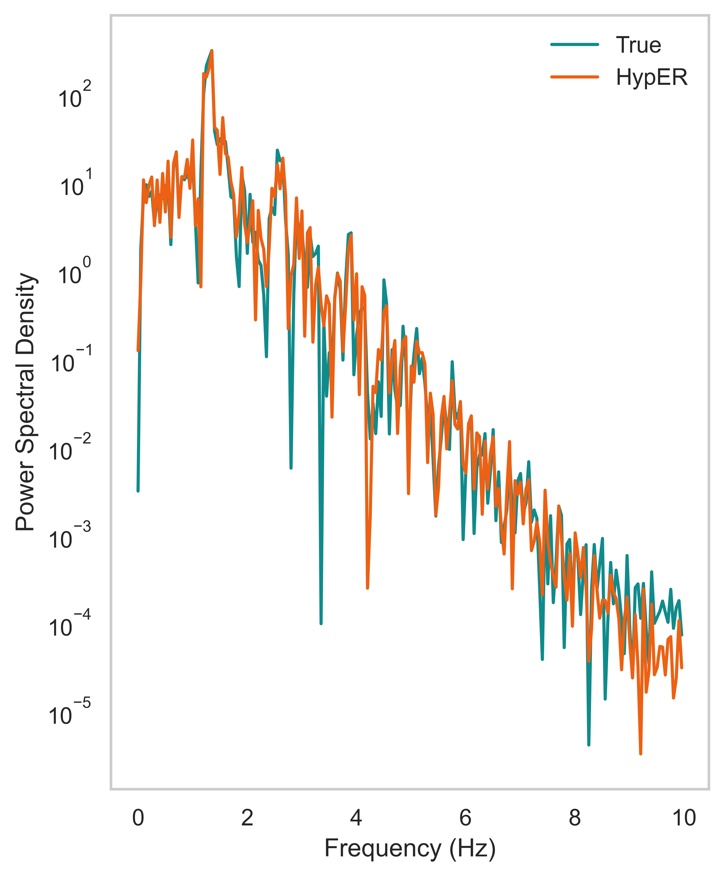}
    \textbf{(b)} Proposed HypER
    \label{fig:psd_lorenz_hypER}
  \end{minipage}

  \caption{PSD plots of autoregressive predictions at a 1000-step horizon when both networks are driven by the Lorenz system. }
  \label{fig:psd}
\end{figure}

\paragraph{Results.}
On chaotic testbeds
(Table \ref{tab:nrmse_horizons}),   
for Lorenz-63,  NRMSE increases rapidly with horizon for all Euclidean reservoirs, exceeding one (fully de-correlated forecasts) after \(\sim\!6T_L\).  By contrast \textsc{HypER} holds sub-percent error out to 800 steps on Lorenz (cf. Fig. \ref{fig:predicted-trajectories-hyper}) and keeps NRMSE one order of magnitude lower than the next best model on Rössler at every horizon, confirming that the hyperbolic expansion factor \(\beta(\sigma)>1\) indeed delays error amplification.  Even on the stiff hyper-chaotic Chen-Ueta attractor (cf. Figs. \ref{fig:nrmse}, \ref{fig:nrmse2}), where all baselines saturate near the aperiodic variance floor, HypER cuts long-horizon error by \(\approx\!12\%\) at 1000 steps, validating the theoretical claim that negative curvature raises the minimum Jacobian gain while preserving ESP. 
As Fig. \ref{fig:psd} shows, HypER retains the true Lorenz power spectrum, whereas the standard ESN spectrum collapses into spurious high-frequency noise.

On real-world data
(Table \ref{tab:nrmse_horizons_auto}),  
in fully open-loop cardiac (MIT-BIH) and laser-chaos benchmarks—both known to possess positive Lyapunov spectra—HypER cuts the best baseline NRMSE by $\approx 35\%$ at every horizon, showing that the geometric bias translates beyond synthetic flows.  On smoother solar-cycle data, the advantage disappears, aligning with theory: when the generator is quasi-periodic rather than chaotic, enlarging \(\lambda_{\min}(W)\) is unnecessary, so hyperbolic wiring behaves like a neutral prior.  Overall, the tables demonstrate that HypER’s curvature-controlled expansion preserves ESP yet provides exploitable separation exactly where standard reservoirs struggle. On the highly non-stationary Santa Fe laser series, HypER outperforms the nearest Euclidean reservoir by roughly 7\% across closed-loop horizons, indicating that its curvature-induced expansion remains beneficial even in noisy, broadband experimental chaos.

Across all three chaotic benchmarks HypER lifts the normalised VPT far beyond every Euclidean reservoir (Table \ref{tab:vpt_adev}), enlarging the window of trajectory fidelity by roughly \(+49\,\%\) on Rössler, \(+52\,\%\) on Lorenz, and a full \(2\times\) on the stiff Chen–Ueta flow.  At the same time it slashes ADev by \(41\%\) on Lorenz and about one-third on Rössler, while still shaving almost \(17\%\) off the Chen baseline.  These dual gains confirm the theoretical picture: the curvature-controlled lower bound on the Jacobian’s smallest singular value extends the period during which closed-loop dynamics track the true attractor, and the concomitant drop in ADev shows that this stability is achieved without sacrificing pointwise spatial accuracy.
\textit{(Results for open-loop setting are presented in the supplementary file.)}

\paragraph{Ablation.}
 
Holding reservoir size fixed (Table \ref{tab:ablation_full}), moving from a flat Euclidean lattice to the Poincaré disc cuts 1000-step NRMSE by \(3\%\) even when the nodes are still placed with Euclidean-volume sampling; switching to \emph{hyperbolic-uniform} sampling supplies the full curvature bias and yields an additional 7\% drop in error together with the best VPT and ADev.  Increasing the latent dimension beyond the disc (to \(d=3,4\)) weakens these gains, confirming the theorem’s prediction that the Jacobian expansion factor \(\beta(\sigma)\) is maximised when curvature is concentrated rather than spread over extra dimensions.

Replacing a homogeneous tanh reservoir by a tanh–sine–linear mix lifts every Euclidean baseline (Table \ref{tab:vpt_lorenz_tanh_mixed}) but lifts HypER most: its NRMSE is \(16\%\) lower, and VPT is \(27\%\) higher than the best graph-structured ESN (a ReLU variant is included for ablation).  This indicates that the curvature-induced amplification preserves the diverse nonlinear signatures instead of letting them collapse under ESP constraints.
The kernel width \(\sigma\) shows a clear optimum around \(0.1\); narrower kernels (\(\sigma<0.05\)) under-connect the graph and reduce VPT (Table \ref{tab:ablation_hyper_params}), while wider kernels dilute curvature and push \(\beta(\sigma)\) back toward unity.  Row-level sparsity \(\kappa\) is broad-tolerant: performance plateaus between 40 – 80 neighbours, with denser reservoirs slightly improving ADev but not NRMSE, confirming that HypER’s advantages arise from geometry rather than sheer connectivity.

\section{Conclusion} \label{sec:conclusion}

We introduced an RC framework—\textit{HypER}—that operates on a negatively curved manifold to capture the exponential divergence inherent in chaotic systems. By mapping reservoir nodes into the Poincaré ball, we explicitly encode geometric properties that mirror the local ``stretch‐and‐fold'' mechanism typical of chaos. Through both theoretical motivation and empirical evaluation on benchmark chaotic systems, we have shown that endowing the reservoir with hyperbolic geometry can extend valid prediction horizons beyond what is typically achievable with Euclidean or topologically random reservoirs.


Several important directions remain open for further research. First, while we have focused on select chaotic attractors, other higher‐dimensional flows or spatiotemporal PDEs may exhibit additional structures that can benefit from more general hyperbolic embeddings or layer‐by‐layer manifold compositions. Second, more sophisticated sampling schemes—such as adaptive node placement or non-uniform radial distributions—could enhance the reservoir’s ability to capture specific dynamical modes within highly complex attractors. Third, integrating the hyperbolic reservoir architecture with modern RC extensions (e.g., output feedback, hierarchical readouts, or online adaptation) might further improve long‐horizon stability. Finally, exploring systematic links between Lyapunov exponents and the negative curvature parameter offers a potential route to tailor the reservoir’s geometry to specific tasks.

\bibliography{ref}

\appendix

\section{Proofs of Lemmas}  

\begin{lemma}[Spectrum of the hyperbolic kernel]
Let \(\{p_i\}_{i=1}^N\subset\mathcal B^d\) be distinct points.
Set
\(\displaystyle
  \delta=\min_{i\neq j}d_{\mathbb H}(p_i,p_j).
\)
Then  
\[
  \lambda_{\min}(\widetilde W)\;\ge\;1-(N-1)e^{-\delta/\sigma},
  \quad
  \rho(\widetilde W)\;\le\;1+(N-1)e^{-\delta/\sigma}.
\]
\end{lemma}

\begin{proof}
     Because \(d_{\mathbb H}(p_i,p_j)=d_{\mathbb H}(p_j,p_i)\) and the exponential is positive,  
   \(\widetilde W\) is a real symmetric matrix with strictly positive entries.
   Hence its eigenvalues are real and can be ordered 
   \(\lambda_1\le\cdots\le\lambda_N\); we set
   \(\lambda_{\min}:=\lambda_1,\;\rho(\widetilde W):=\lambda_N\).
By Gershgorin disc theorem \cite{Gershgorin1931, HornJohnsonMatrix}, it follows that  
   for any square matrix \(A=(a_{ij})\) every eigenvalue \(\lambda\) lies in at least one Gershgorin disc  
   \(\mathcal D_i:=\{z\in\mathbb C:\;|z-a_{ii}|\le R_i\}\)  
   with radius \(R_i:=\sum_{j\neq i}|a_{ij}|\).  Because \(\widetilde W\) is symmetric and real, its eigenvalues lie on the real axis inside those intervals.
 The kernel gives \(d_{\mathbb H}(p_i,p_i)=0\); therefore centre of each disc  is 
   \(
     \widetilde W_{ii}=e^{-0/\sigma}=1\quad\text{for all }i.
   \)
   Off–diagonal terms satisfy  
   \(
     \widetilde W_{ij}=e^{-d_{\mathbb H}(p_i,p_j)/\sigma}
     \;\le\;
     e^{-\delta/\sigma}\quad(i\neq j),
   \)
   because \(d_{\mathbb H}(p_i,p_j)\ge\delta\).  Consequently, we get uniform bound on disc radii
   \[
     R_i
       =\sum_{j\neq i}\widetilde W_{ij}
       \le (N-1)e^{-\delta/\sigma}\quad\text{for all }i.
   \]
Next we work out the intervals containing the spectrum.
 Since all Gershgorin discs are centred at  
  \(c_i=\widetilde W_{ii}=1\), and every disc has radius  
  \(R_i=\sum_{j\neq i}\widetilde W_{ij}\le (N-1)e^{-\delta/\sigma}\), we
  call this uniform upper bound \(R_{\max}:=(N-1)e^{-\delta/\sigma}\).
Because \(\widetilde W\) is real-symmetric, every eigenvalue lies on the real axis.  
For a given row \(i\) the Gershgorin disc therefore collapses to the closed interval  
\[
\mathcal I_i
  \;=\;
  \bigl[c_i-R_i,\;c_i+R_i\bigr]
  \;=\;
  \bigl[\,1-R_i,\;1+R_i\bigr].
\]
Indeed, each interval is centred at the same point \(1\), and every \(R_i\) is \(\le R_{\max}\).
Hence every interval \(\mathcal I_i\) is contained in the single
“worst-case” interval
\[
\bigl[\,1-R_{\max},\;1+R_{\max}\bigr]
  \;=\;
  \bigl[\,1-(N-1)e^{-\delta/\sigma},\;
         1+(N-1)e^{-\delta/\sigma}\bigr].
\]
Because the Gershgorin theorem guarantees that every eigenvalue of
\(\widetilde W\) lies in \emph{some} \(\mathcal I_i\), it follows that
\emph{all} eigenvalues lie in their common super-interval.  
Taking the left end-point gives a global lower bound on
\(\lambda_{\min}(\widetilde W)\); taking the right end-point gives a
global upper bound on
\(\lambda_{\max}(\widetilde W)=\rho(\widetilde W)\) as
\[
\boxed{\,
\lambda_{\min}(\widetilde W)\;\ge\;
  1-(N-1)e^{-\delta/\sigma}\,},\quad
\boxed{\,
\rho(\widetilde W)\;\le\;
  1+(N-1)e^{-\delta/\sigma}\,}.
\]
   Because \(\delta>0\) for distinct points, the right-hand side of the first inequality is strictly positive whenever  
   \(e^{-\delta/\sigma}<(N-1)^{-1}\).  This condition is easily met in practice for moderate \(N\) once \(\sigma\) is chosen on the same order as or smaller than \(\delta\), confirming that \(\widetilde W\) is well-conditioned from below.
\end{proof}

\begin{figure}[!ht]
  \centering

  \begin{minipage}[t]{0.49\linewidth}
    \centering
    \includegraphics[width=1.1\textwidth]{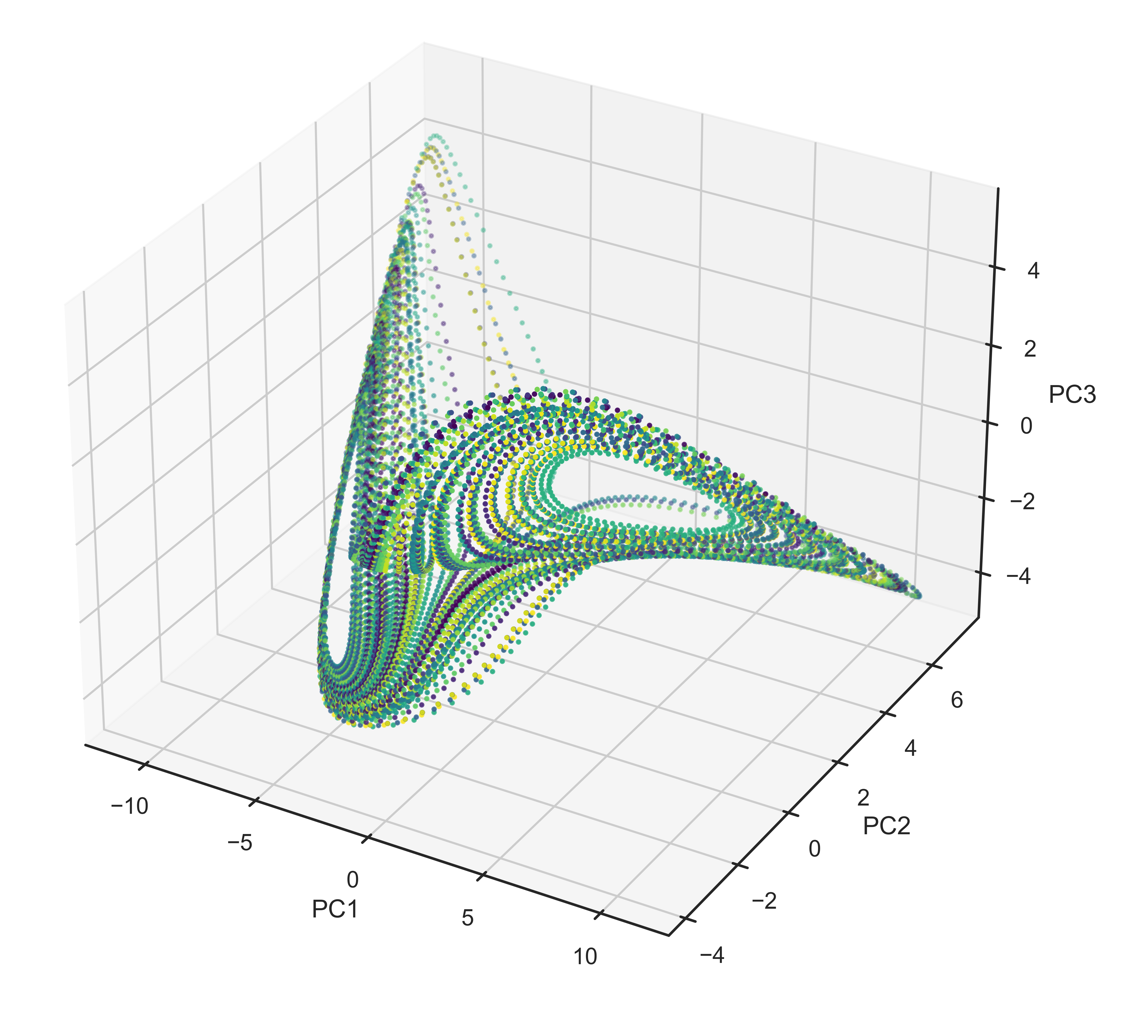}\\
    \textbf{(a)} Standard ESN
    \label{fig:sub1b}
  \end{minipage}
  \hfill
  \begin{minipage}[t]{0.49\linewidth}
    \centering
    \includegraphics[width=1.1\textwidth]{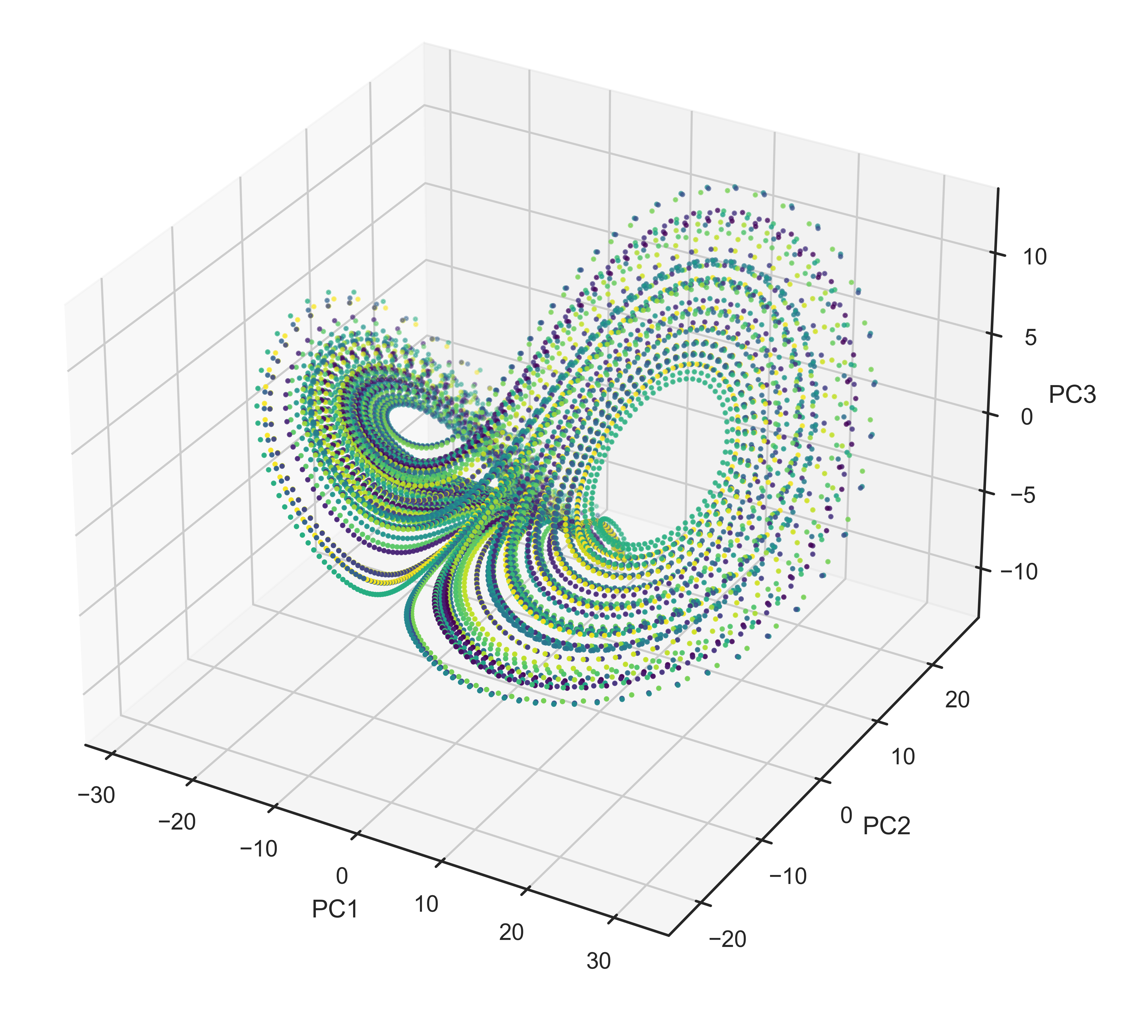}\\
    \textbf{(b)} Proposed HypER
    \label{fig:sub3b}
  \end{minipage}

  \caption{Three‑dimensional PCA projections of high-dimensional reservoir
states (\emph{no read-out training applied}) for (a) a standard ESN and (b) the proposed HypER, when both networks are driven by the Lorenz system.}
  \label{fig:lorenz_comparison_3d}
\end{figure}

\begin{lemma}[Linearized forward sensitivity]\label{lemma2b}
Let  
\(
x_{t+1}\;=\;(1-\alpha)\,x_t+\alpha\,\phi\ \!\bigl(Wx_t+Uu_t\bigr),
 0<\alpha\le 1,
\)
be the leaky-ESN state equation.  
Assume  
\(W\in\mathbb R^{N\times N}\) is symmetric, obtained from the exponential kernel then spectrally rescaled so that its spectral radius is \(\rho(W)<1\); denote its smallest eigenvalue by \(\lambda_{\min}(W)\ge 0\);
the activation \(\phi\) is \(C^{1}\) with derivative bounded on the reachable domain by  
  \(
    0<m\;\le\;\phi'(z)\;\le\;L<\infty.
  \)
For any current input \(u\) and reservoir state \(x\) define  
\[
J(x)\;=\; \frac{\partial x_{t+1}}{\partial x_t}
        \;=\;(1-\alpha)I+\alpha D_\phi(x)\,W,
\] where
\(D_\phi(x)=\operatorname{diag}\ \!\bigl(\phi'(Wx+Uu)\bigr).
\)
Then  
\[
\boxed{\;
  s_{\min}\bigl(J(x)\bigr)\;\ge\;
  \sqrt{\frac{m}{L}}\Bigl[(1-\alpha)+\alpha m\,\lambda_{\min}(W)\Bigr]
\;}
\]
and  
\[
\boxed{\;
  \|J(x)\|\;\le\;
  \sqrt{\frac{L}{m}}\Bigl[(1-\alpha)+\alpha L\,\rho(W)\Bigr].
\;}
\]

\end{lemma}

\begin{proof}
    We factorize the Jacobian through a similarity transform.  
Let  
\(D:=D_\phi(x)=\operatorname{diag}(d_1,\dots,d_N)\) with \(d_i\in[m,L]\).  
Because \(D\) is positive definite, its square root \(R:=D^{1/2}\) is well-defined and satisfies  
\[
m^{1/2}I\;\preceq\;R\;\preceq\;L^{1/2}I .
\tag{1}
\]

Write the Jacobian in the form  
\[
J(x)=(1-\alpha)I+\alpha D W
     \;=\;
     R\Bigl[(1-\alpha)I+\alpha\,RWR\Bigr]R^{-1}.
\tag{2}
\]
Set  
\(P:=RWR\).  
Because \(W\) is symmetric and \(R\) is diagonal, \(P\) is symmetric:
\(
P^\top=RWR = P.
\)
Equation (2) shows that \(J(x)\) is similar (via the invertible matrix \(R\)) to the symmetric matrix  
\[
K\;:=\;(1-\alpha)I+\alpha P .
\tag{3}
\]

For any unit vector \(v\), by Rayleigh quotient property, we have
\[
v^\top Pv
  =(Rv)^\top W (Rv)
  \;\ge\;
  \lambda_{\min}(W)\,\|Rv\|^{2}
  \;\ge\;
  \lambda_{\min}(W)\,m\,\|v\|^{2},
\]
where the last inequality uses (1).  
Hence  
\[
\lambda_{\min}(P)\;\ge\;m\,\lambda_{\min}(W),
\quad
\rho(P)\;=\;\lambda_{\max}(P)\;\le\;L\,\rho(W).
\tag{4}
\]


Because \(K\) is symmetric, its singular values equal the absolute values of its eigenvalues.  By (3)–(4):
\[
\lambda_{\min}(K)
  =(1-\alpha)+\alpha\,\lambda_{\min}(P)
  \;\ge\;
  (1-\alpha)+\alpha m\,\lambda_{\min}(W),
\tag{5}
\]
\[
\lambda_{\max}(K)
  =(1-\alpha)+\alpha\,\rho(P)
  \;\le\;
  (1-\alpha)+\alpha L\,\rho(W).
\tag{6}
\]


For any matrices \(A,B\) one has  
\(s_{\min}(ABA^{-1})\ge
  \dfrac{s_{\min}(A)}{s_{\max}(A)}\,s_{\min}(B)\)  
and  
\(\|ABA^{-1}\|\le
  \dfrac{s_{\max}(A)}{s_{\min}(A)}\,\|B\|\).
Apply these identities to \(A=R\) and \(B=K\):
\[
s_{\min}\bigl(J(x)\bigr)
  \;\ge\;
  \frac{s_{\min}(R)}{s_{\max}(R)}\,\lambda_{\min}(K),
\]
\[
\|J(x)\|
  \;\le\;
  \frac{s_{\max}(R)}{s_{\min}(R)}\,\lambda_{\max}(K).
\tag{7}
\]
Since \(R\) is diagonal, \(s_{\min}(R)=\sqrt{m}\) and \(s_{\max}(R)=\sqrt{L}\);  
plugging these together with (5)–(6) into (7) yields the announced bounds.
\end{proof}

\begin{figure}[!ht]
    \centering
    \includegraphics[width=\linewidth]{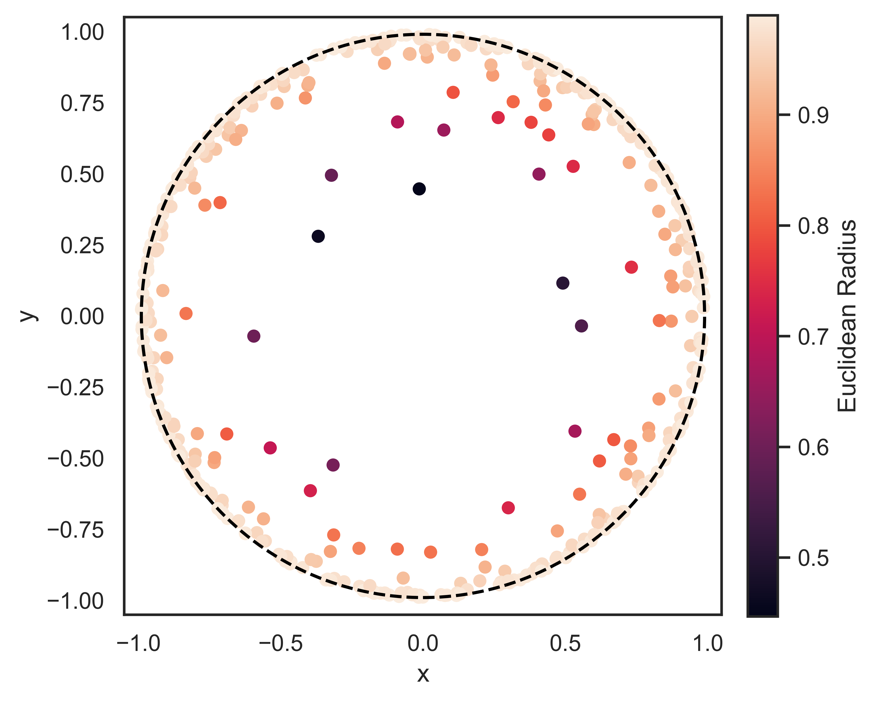}
    \caption{An illustration of node placement in the two-dimensional Poincaré disk of radius \(1\).  
 The sampled nodes, color-coded by their Euclidean radius \(\|\mathbf{p}_i\|\),  
with a dashed circle indicating the disk boundary. The radial distribution is drawn  
proportional to \(\sinh(\beta\,r)\).}
    \label{fig:enter-label}
\end{figure}

\begin{figure}[!ht]
    \centering
    \includegraphics[width=\linewidth]{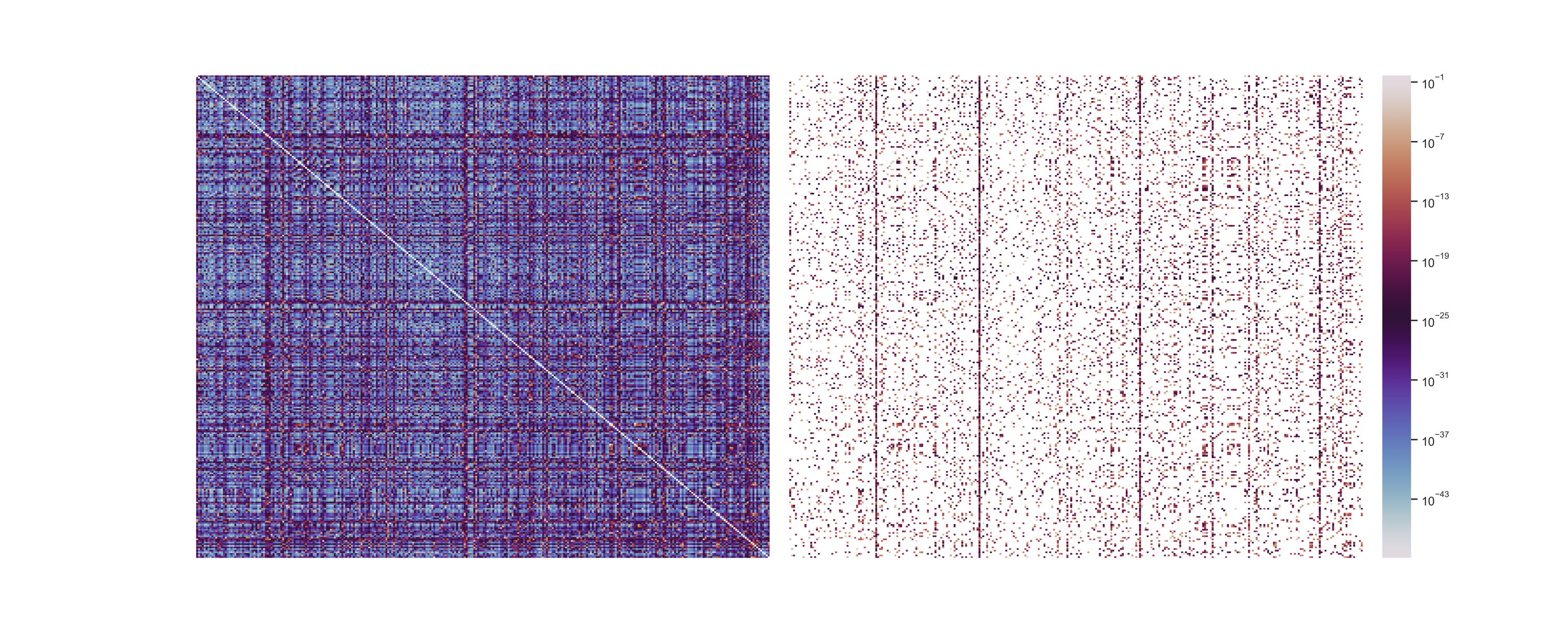}
    \caption{Heatmaps of HypER's adjacency matrix (a) before and (b) after row-level sparsity.  Each entry \(W_{ij}\) decays exponentially with hyperbolic distance \(d_{\mathbb H}\).  The strength distribution exhibits a highly non-uniform “patchy’’ fractal-like patterns, driven by the negative-curvature geometry that concentrates nodes near the Poincaré-disk boundary. \textit{Zoom in on the electronic version to see the fine-grained structure clearly.}}
    \label{fig:sparsity_figure}
\end{figure}

  \textit{Interpretation:}  
The lower bound says that any infinitesimal difference in reservoir states is amplified by at least the factor \(
\sqrt{\tfrac{m}{L}}\,
\bigl[(1-\alpha)+\alpha m\,\lambda_{\min}(W)\bigr],
\)
so long as the activation derivative does not vanish (\(m>0\)).  Negative curvature enters through \(\lambda_{\min}(W)\): the exponential kernel on the Poincaré disk forces \(\lambda_{\min}(W)\) away from zero (Lemma 1), thereby guaranteeing a strictly positive expansion even when the leak rate \(1-\alpha\) is small.
The upper bound links the choice of kernel width \(\sigma\) (which influences \(\rho(W)\)) and the leak \(\alpha\) to the Echo-State requirement that the global Lipschitz constant of the state map remain below one; it is thus the analytical counterpart of the usual empirical rule “keep \(\alpha\rho(W)\lesssim1\)”.

Theorem 3 (in paper) crystallises HypER’s core intuition in quantitative form: by wiring the reservoir through a hyperbolic kernel, then scaling it to spectral radius \(\varrho\), we obtain a Jacobian whose smallest singular value is bounded below by \(\beta(\sigma)\).  Because \(\widetilde W_{ij}\) carries the exponential metric of the Poincaré ball, shrinking the kernel width \(\sigma\) inflates the ratio \(\lambda_{\min}(\widetilde W)/\rho(\widetilde W)\) and pushes \(\beta(\sigma)\) past one, guaranteeing that \emph{every} infinitesimal perturbation is stretched rather than damped.  This delivers a task-aligned inductive bias: the reservoir’s worst-case local amplification exceeds unity, mirroring the positive Lyapunov exponent that defines chaos, yet the global Lipschitz constant remains controlled because the same \(\sigma\) enters the ESP-preserving rescale by \(\varrho<1\).  The theorem therefore turns the geometric knob \(\sigma\) and the stability knob \(\varrho\) into explicit levers whose joint tuning carves out a safe operating envelope—too small \(\varrho\) erodes expansion, too large violates ESP—while the square-root factor \(\sqrt{m/L}\) exposes how activation non-linearity (\(m\le\phi'\le L\)) conditions the bound.  Crucially, if the kernel distance were Euclidean, \(\lambda_{\min}(\widetilde W)\) would decay only polynomially with network size, forcing \(\beta(\sigma)\le 1\) in high dimensions; negative curvature alone provides the exponential volume growth needed to keep the bound positive.  

Once the hyperbolic kernel width \(\sigma\) is chosen so that the expansion factor  \(\beta(\sigma)\)  exceeds 1, \emph{every} infinitesimal mismatch injected into the reservoir is guaranteed to grow by at least \(\beta(\sigma)^{\tau-1}\) after \(\tau\) steps.  Under that strictly expanding regime, any two input trajectories that differ even slightly will separate along the unstable directions that the hyperbolic geometry preferentially preserves; consequently the subsequent point-wise nonlinear transforms applied at the nodes operate on already well-disentangled signals.  Deploying a heterogeneous palette of activations—tanh (\(m=1-L^{-1}\), saturating but monotone), ReLU (\(m=0^+\), piecewise-linear), sine (confined to a monotone half–period so that \(m>0\)), and the identity—therefore \emph{enlarges the functional basis} of the reservoir without jeopardising stability: each class of nonlinearity contributes an independent Taylor/Lipschitz signature that the read-out can exploit, and the theorem’s lower-bound ensures those signatures are not annihilated by contraction.  In Euclidean reservoirs, with \(\beta(\sigma)\le 1\) for any feasible scaling, initial differences are rapidly damped; the rich palette then degenerates to redundant or vanishing features.  Hence heterogeneous node-wise nonlinearities become genuinely beneficial only in a geometry—such as HypER’s negatively curved wiring—that provides a provable minimum expansion per step, preserving the diversity they inject into the state space. Empirically we therefore observe that mixed activations widen the dynamical basis only for HypER (cf. Table 
 5 in paper); for flat‑geometry models they push the system beyond its stability envelope and hurt prediction accuracy.
Hence the theorem upgrades HypER from geometric intuition to provable mechanism: any one-step input mismatch of norm \(\|\Delta u\|\) is amplified by \(\alpha m\beta(\sigma)^{\tau-1}\) after \(\tau\) steps, analytically explaining why Hyperbolic Embedding Reservoirs sustain accurate forecasts far beyond the 5–8 Lyapunov times achievable by Euclidean ESNs without sacrificing echo-state stability.

\begin{figure}[!ht]
  \centering
  \begin{minipage}[t]{0.49\linewidth}
    \centering
    \includegraphics[width=\linewidth]{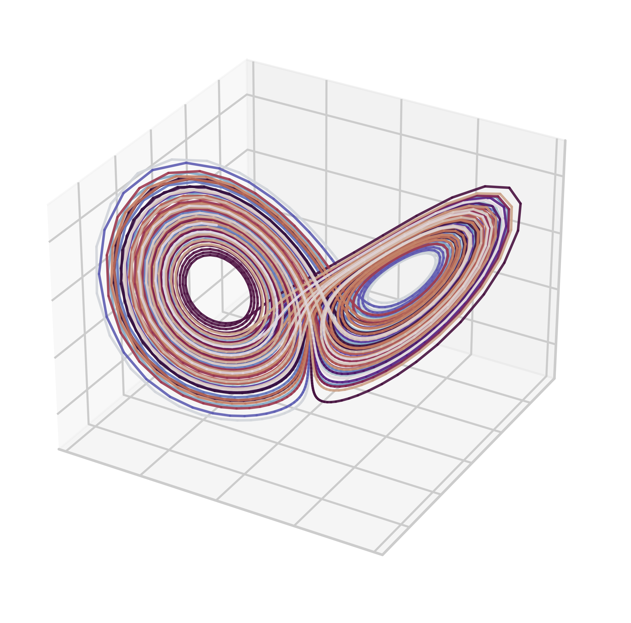}
    \textbf{(a)} Lorenz
    \label{fig:sub11c}
  \end{minipage}
  \hfill
  \begin{minipage}[t]{0.49\linewidth}
    \centering
    \includegraphics[width=\linewidth]{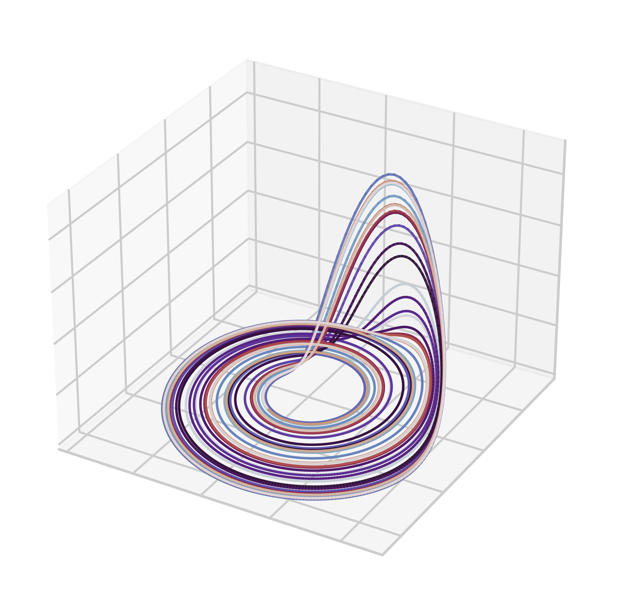}
    \textbf{(b)} Rössler
    \label{fig:sub2}
  \end{minipage}

  \vspace{1em}

  \begin{minipage}[t]{0.49\linewidth}
    \centering
    \includegraphics[width=\linewidth]{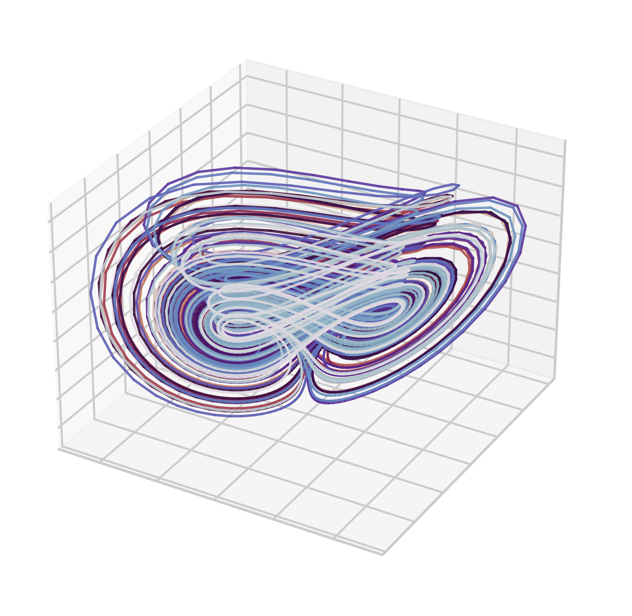}
    \textbf{(c)} Chen
    \label{fig:sub31}
  \end{minipage}
  \hfill
  \begin{minipage}[t]{0.49\linewidth}
    \centering
    \includegraphics[width=\linewidth]{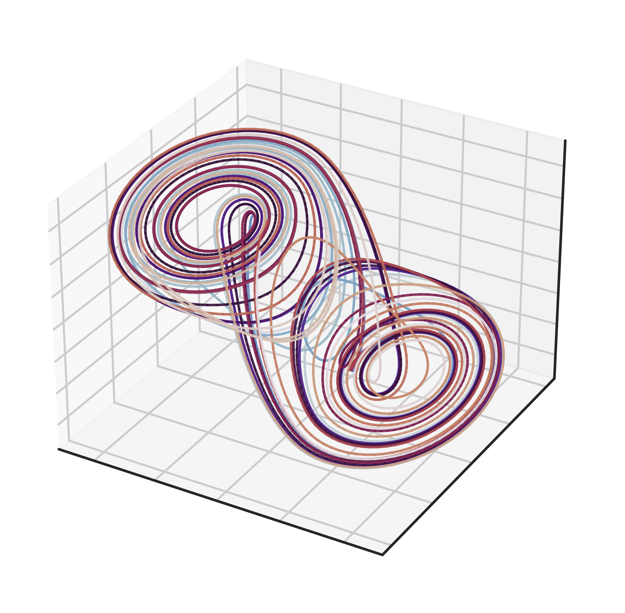}
    \textbf{(d)} Chua
    \label{fig:sub4}
  \end{minipage}

\caption{Canonical three-dimensional chaotic test beds.  Panels (a)–(d) show the strange attractors of the Lorenz-63, Rössler, Chen and Chua systems, respectively, integrated with identical step size and plotted in natural coordinates.  All four exhibit the hallmark stretch–and–fold geometry responsible for positive Lyapunov exponents and finite predictability horizons; these data sets constitute the evaluation suite on which HypER’s long-range forecasting performance is compared against Euclidean ESN baselines.}

  \label{fig:lorenz_comparison_3d2}
\end{figure}

\begin{figure}[!ht]
  \centering
  \includegraphics[width=0.7\linewidth]{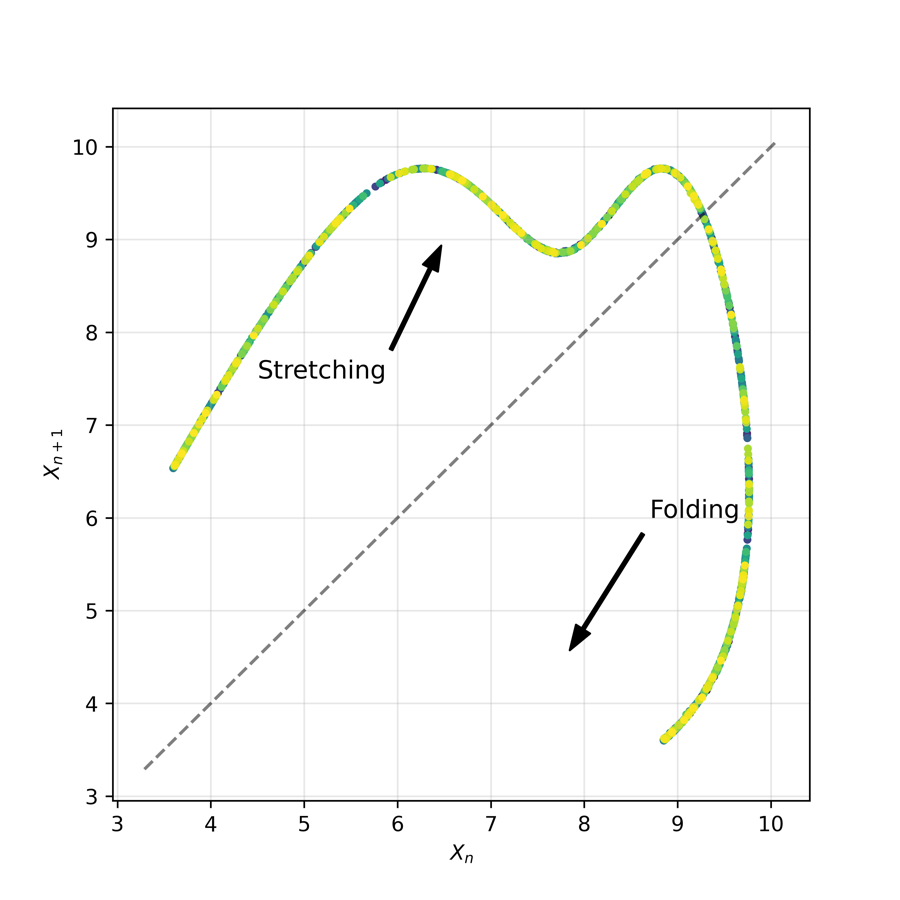}
  \caption{The first return map of the Rössler attractor exhibits ``stretch–and–fold'' dynamics.}
  \label{fig:rossler_return_map}
\end{figure}

\begin{table*}[!ht]
\centering
\caption{NRMSE of teacher-forced one-step ahead forecasting across multiple prediction horizons.}
\label{tab:nrmse_horizonsb}
\resizebox{0.99\textwidth}{!}{
\begin{tabular}{l cc c c c ccc}
\toprule
\multirow{2}{*}{\textbf{Dataset}} & \multirow{2}{*}{\textbf{Horizon}}& \multicolumn{7}{c}{\textbf{NRMSE ($\times 10^{-4}$) $\downarrow$}}\\
\cmidrule(lr){3-9}& & ESN& SCR& CRJ& SW-ESN& MCI-ESN &DeepESN &HypER\\
\midrule
 \multirow{5}{*}{Lorenz}& 200& 1.8996 ± 0.2990& 1.4125 ± 0.2717& 1.5643 ± 0.2405& 0.7354 ± 0.0955& 0.2557 ± 0.0235& 1.8844 ± 0.2839&\textbf{0.0098 ± 0.0020}\\
 & 400& 2.9158 ± 0.5809& 2.7863 ± 0.7634& 2.6788 ± 0.3923& 1.3756 ± 0.3007& 0.4292 ± 0.0451& 3.4581 ± 0.6540&\textbf{0.0286 ± 0.0159}\\
 & 600& 2.6095 ± 0.4832& 2.4310 ± 0.6351& 2.3650 ± 0.3213& 1.2035 ± 0.2473& 0.3780 ± 0.0369& 3.0275 ± 0.5417&\textbf{0.0245 ± 0.0132}\\
 & 800& 2.4403 ± 0.4260& 2.2120 ± 0.5437& 2.1863 ± 0.2711& 1.0929 ± 0.2079& 0.3487 ± 0.0314& 2.7804 ± 0.4631&\textbf{0.0218 ± 0.0110}\\
 & 1000& 2.3608 ± 0.3973& 2.0995 ± 0.4890& 2.1066 ± 0.2494& 1.0407 ± 0.1840& 0.3353 ± 0.0297&  2.6613 ± 0.4149&\textbf{0.0202 ± 0.0096}\\

 \midrule
 \multirow{5}{*}{Rössler}& 200& 0.9231 ± 0.2373& 1.9937 ± 0.1939& 2.6059 ± 0.2446& 2.1212 ± 0.4111& 1.1950 ± 0.4164& 5.3491 ± 0.5995&\textbf{0.3083 ± 0.0445}\\
 & 400& 0.7954 ± 0.1926& 1.7086 ± 0.1584& 2.1328 ± 0.1823& 1.7554 ± 0.3084& 1.7890 ± 0.6158& 4.7256 ± 0.5313&\textbf{0.2892 ± 0.0584}\\
 & 600& 0.7707 ± 0.1997& 1.6123 ± 0.1408& 2.0085 ± 0.1545& 1.6425 ± 0.2645& 1.5919 ± 0.5456& 4.5042 ± 0.4586&\textbf{0.2654 ± 0.0512}\\
 & 800& 0.8084 ± 0.1775& 1.7565 ± 0.1537& 2.2597 ± 0.1820& 1.8216 ± 0.3055& 1.7582 ± 0.5515& 4.6048 ± 0.4612&\textbf{0.3229 ± 0.0669}\\
 & 1000& 0.8033 ± 0.1570& 1.7009 ± 0.1084& 2.1034 ± 0.1544& 1.6580 ± 0.2311& 2.3943 ± 0.6273& 4.2651 ± 0.3709&\textbf{0.2711 ± 0.0545}\\
 \midrule
 \multirow{5}{*}{Chen}& 200
& 8.3461 ± 1.5647& 3.4524 ± 1.0955& 3.8714 ± 1.0121& 2.7018 ± 0.8006& 0.5529 ± 0.1016& 11.1928 ± 2.7676&\textbf{0.0289 ± 0.0136}\\
 & 400
& 7.8578 ± 1.3720& 3.1347 ± 0.9444& 3.5096 ± 0.8396& 2.4642 ± 0.6858& 0.5618 ± 0.1039& 11.2065 ± 2.6590&\textbf{0.0253 ± 0.0111}\\
 & 600
& 7.5713 ± 1.2529& 2.9454 ± 0.8627& 3.2970 ± 0.7462& 2.3379 ± 0.6388& 0.5471 ± 0.1008& 10.9863 ± 2.6134&\textbf{0.0237 ± 0.0099}\\
 & 800
& 7.8061 ± 1.2364& 2.9556 ± 0.8322& 3.3111 ± 0.7140& 2.3630 ± 0.6441& 0.5548 ± 0.1060& 11.3102 ± 2.6167&\textbf{0.0233 ± 0.0093}\\
 \midrule
 \multirow{4}{*}{Chua}& 200& 13.8035 ± 0.1938& 13.5475 ± 0.1996& 15.1684 ± 0.1935& 14.3335 ± 0.6491& 9.7472 ± 0.3295& 10.1497 ± 2.1081&\textbf{6.2185 ± 0.6540 }\\
 & 400& 15.1457 ± 0.1914& 15.3644 ± 0.2263& 16.6796 ± 0.1801& 16.1305 ± 0.6441& 13.4237 ± 1.0293& 16.5536 ± 5.9645&\textbf{9.9356 ± 1.8790}\\
 & 600& 15.9751 ± 0.1817& 16.2945 ± 0.2131& 17.5255 ± 0.2071& 17.3768 ± 0.6440& 15.7155 ± 1.4067& 19.1823 ± 5.2401&\textbf{11.9119 ± 2.3407}\\
 & 800& 15.3668 ± 0.1756& 15.5799 ± 0.1943& 16.8523 ± 0.1939& 16.6317 ± 0.6016& 15.5763 ± 1.2959& 17.7481 ± 4.1196&\textbf{11.1649 ± 1.8235}\\
 \midrule
 \multirow{5}{*}{Mackey–Glass}& 200
& 5.5661 ± 1.3796& 6.9932 ± 1.5362& 5.0097 ± 0.6247& 5.8431 ± 0.5661& 19.1107 ± 4.0490& 8.6601 ± 5.0794&\textbf{3.1238 ± 1.9327}\\
 & 400
& 6.2835 ± 1.4160& 6.6683 ± 1.2692& 6.2317 ± 0.6278& 7.4746 ± 0.7737& 16.3728 ± 3.5216& 8.4966 ± 4.1926&\textbf{2.6762 ± 1.6936}\\
 & 600
& 6.6987 ± 1.3977& 6.9751 ± 1.2262& 6.7004 ± 0.6669& 8.3971 ± 0.9325& 18.0263 ± 3.6055& 8.8932 ± 4.1969&\textbf{2.8812 ± 1.7718}\\
 & 800& 7.2707 ± 1.5154& 7.3480 ± 1.2577& 7.3534 ± 0.7506& 9.4275 ± 1.0753& 19.3275 ± 3.6537& 9.4067 ± 4.3209&\textbf{3.0088 ± 1.8308}\\
 & 1000& 6.3934 ± 1.1647& 6.0996 ± 0.9603& 6.2811 ± 0.5778& 8.5047 ± 0.7898& 17.7380 ± 3.0836& 8.1932 ± 3.5622&\textbf{2.6355 ± 1.6458}\\
 \bottomrule
\end{tabular}
}
\end{table*}

\begin{figure*}[!ht]
  \centering

  \begin{minipage}[t]{0.22\textwidth}
    \centering
    \includegraphics[width=\linewidth]{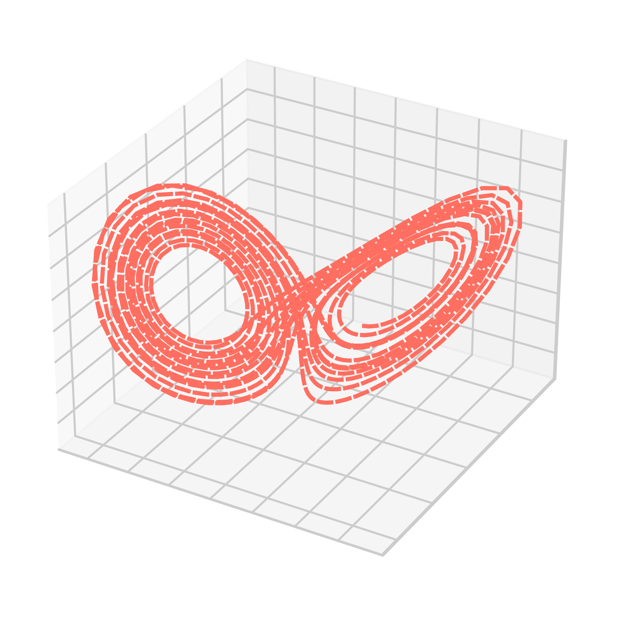}
    \textbf{(a)} Target
    \label{fig:minipage1}
  \end{minipage}
  \hfill
  \begin{minipage}[t]{0.22\textwidth}
    \centering
    \includegraphics[width=\linewidth]{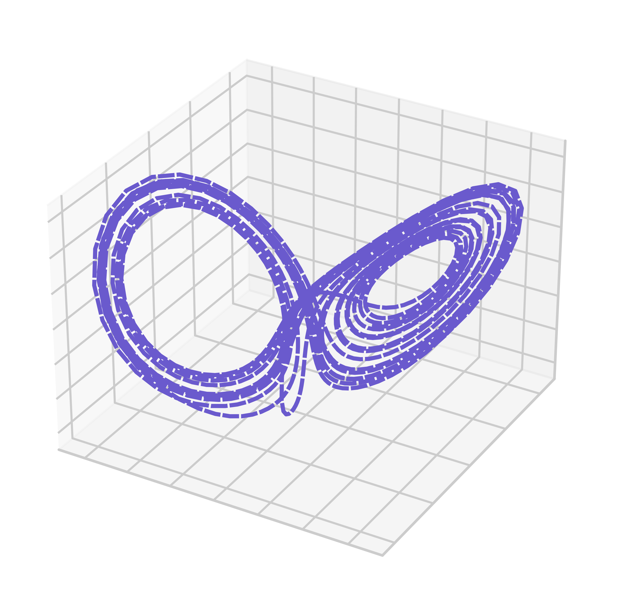}
    \textbf{(b)} ESN
    \label{fig:minipage2}
  \end{minipage}
  \hfill
  \begin{minipage}[t]{0.22\textwidth}
    \centering
    \includegraphics[width=\linewidth]{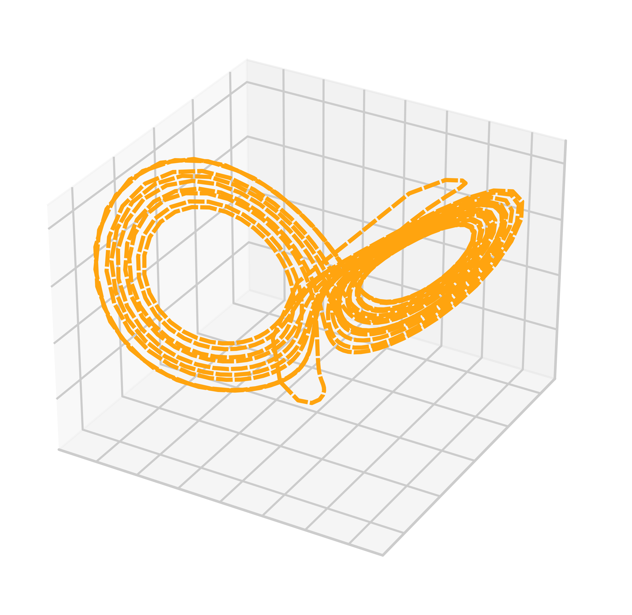}
    \textbf{(c)} SCR
    \label{fig:minipage3}
  \end{minipage}
  \hfill
  \begin{minipage}[t]{0.22\textwidth}
    \centering
    \includegraphics[width=\linewidth]{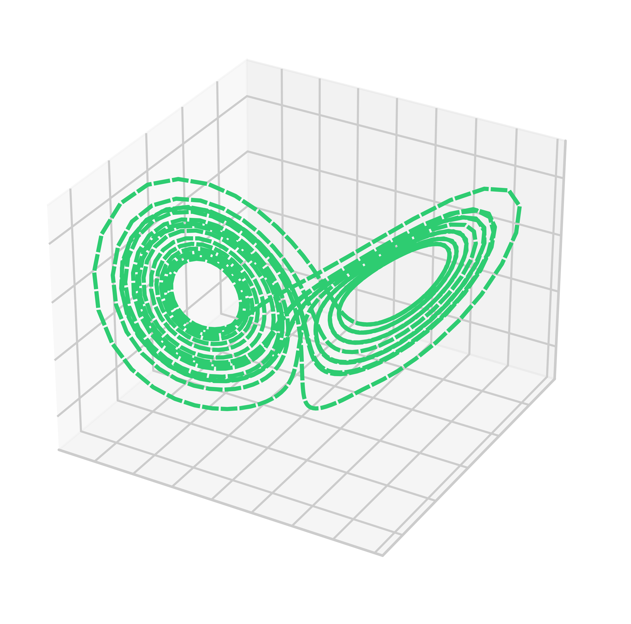}
    \textbf{(d)} CRJ
    \label{fig:minipage4}
  \end{minipage}

  \vspace{1em}

  \begin{minipage}[t]{0.22\textwidth}
    \centering
    \includegraphics[width=\linewidth]{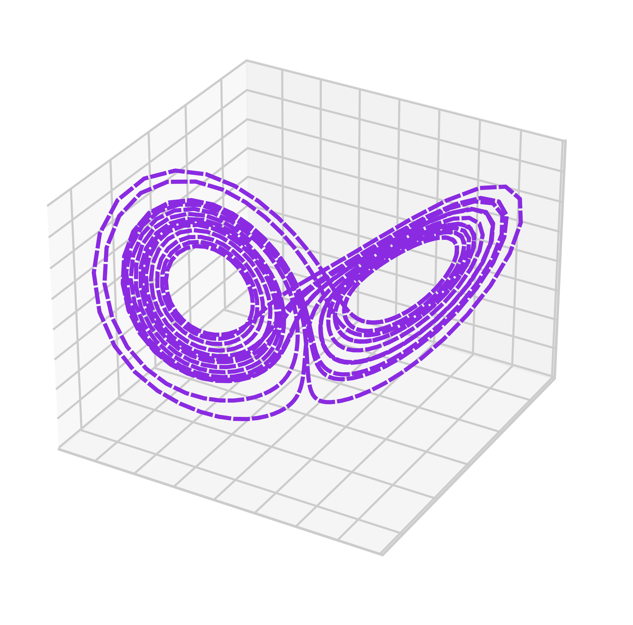}
    \textbf{(e)} SW-ESN
    \label{fig:minipage5}
  \end{minipage}
  \hfill
  \begin{minipage}[t]{0.22\textwidth}
    \centering
    \includegraphics[width=\linewidth]{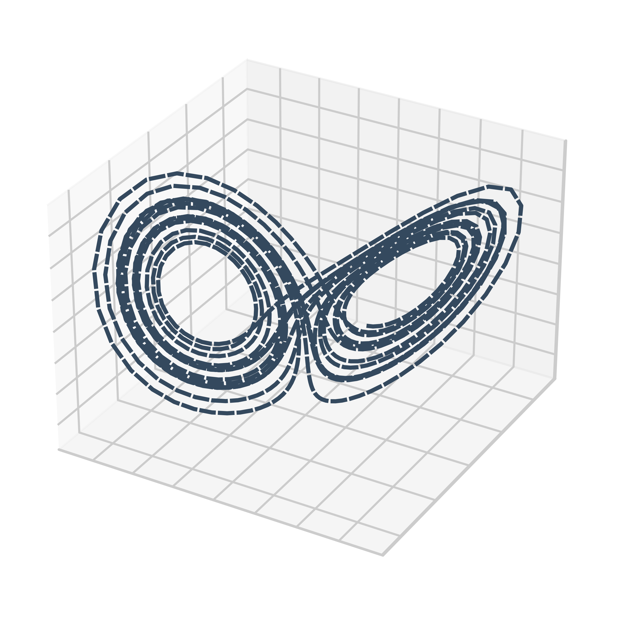}
    \textbf{(f)} MCI-ESN
    \label{fig:minipage6}
  \end{minipage}
  \hfill
  \begin{minipage}[t]{0.22\textwidth}
    \centering
    \includegraphics[width=\linewidth]{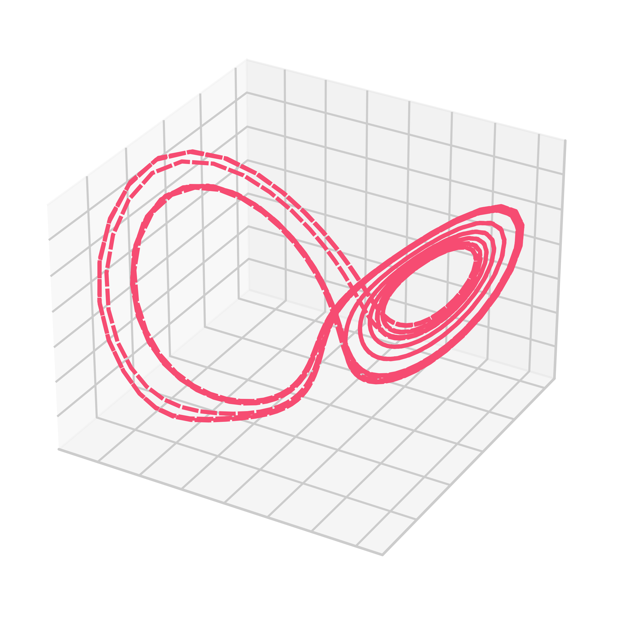}
    \textbf{(g)} DeepESN
    \label{fig:minipage7}
  \end{minipage}
  \hfill
  \begin{minipage}[t]{0.22\textwidth}
    \centering
    \includegraphics[width=\linewidth]{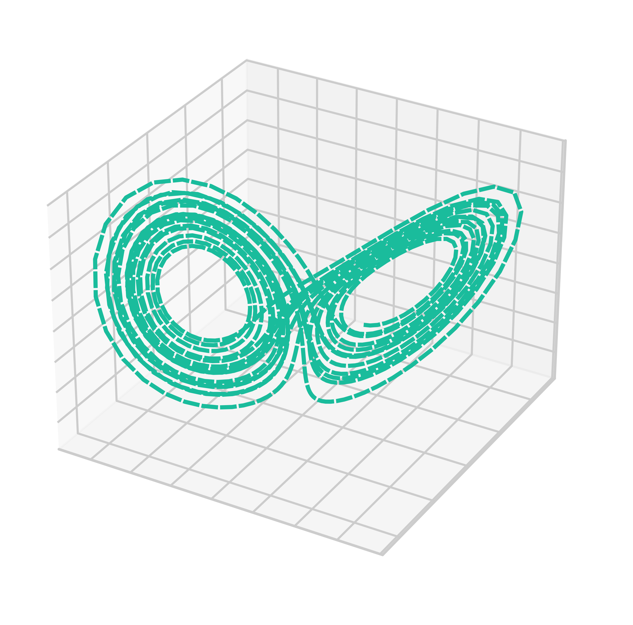}
    \textbf{(h)} HypER
    \label{fig:minipage8}
  \end{minipage}

  \caption{3D Phase portraits for the Lorenz system predicted by different reservoir architectures in closed-loop setting.}
  \label{fig:minipage_4x2}
\end{figure*}

\begin{figure*}[!ht]
    \centering
    \begin{minipage}{0.35\textwidth}
        \centering
        \includegraphics[width=\linewidth]{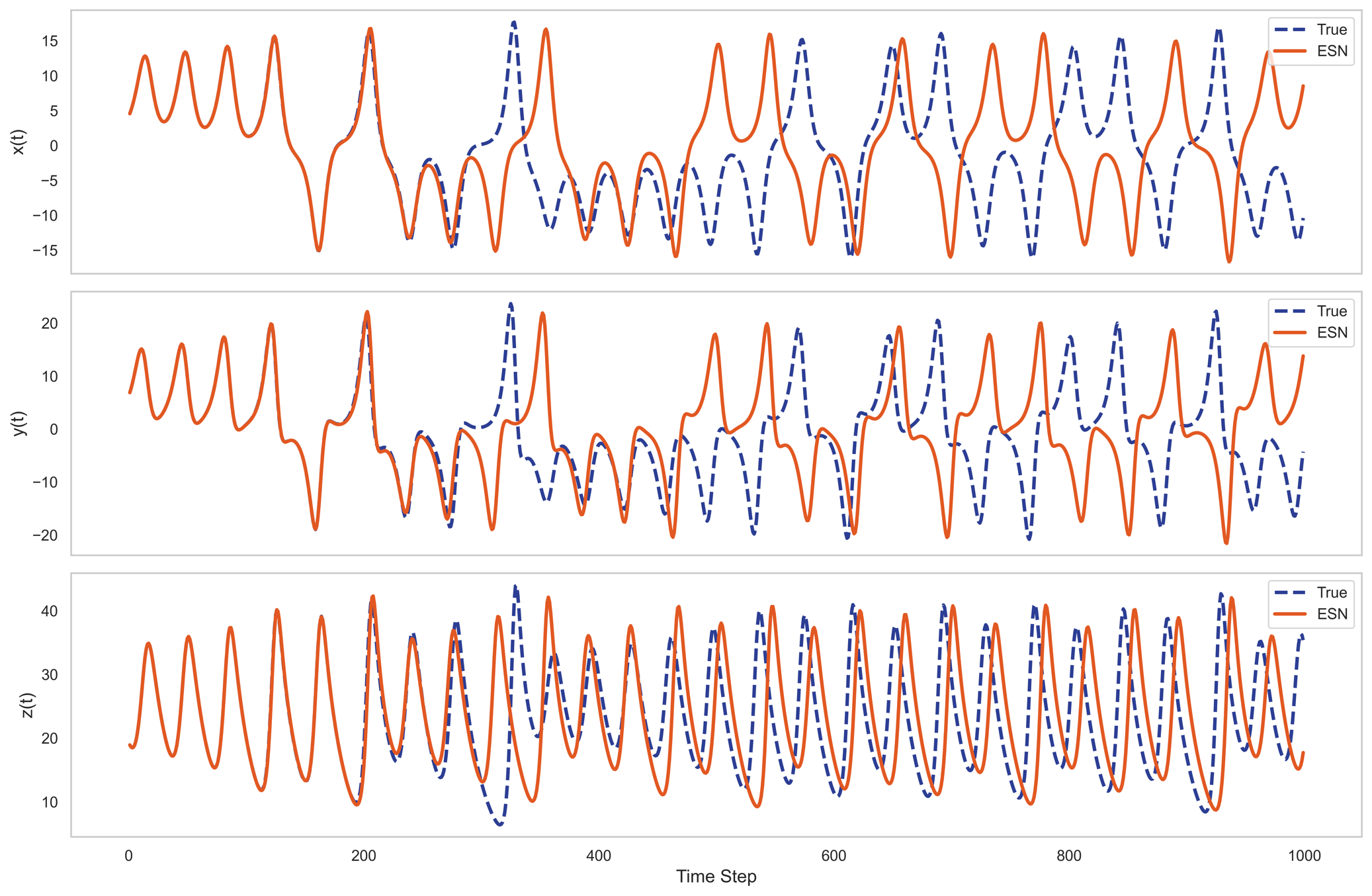}
            \textbf{(a)} ESN
    \end{minipage} \hspace{0.05\textwidth}
    \begin{minipage}{0.35\textwidth}
        \centering
        \includegraphics[width=\linewidth]{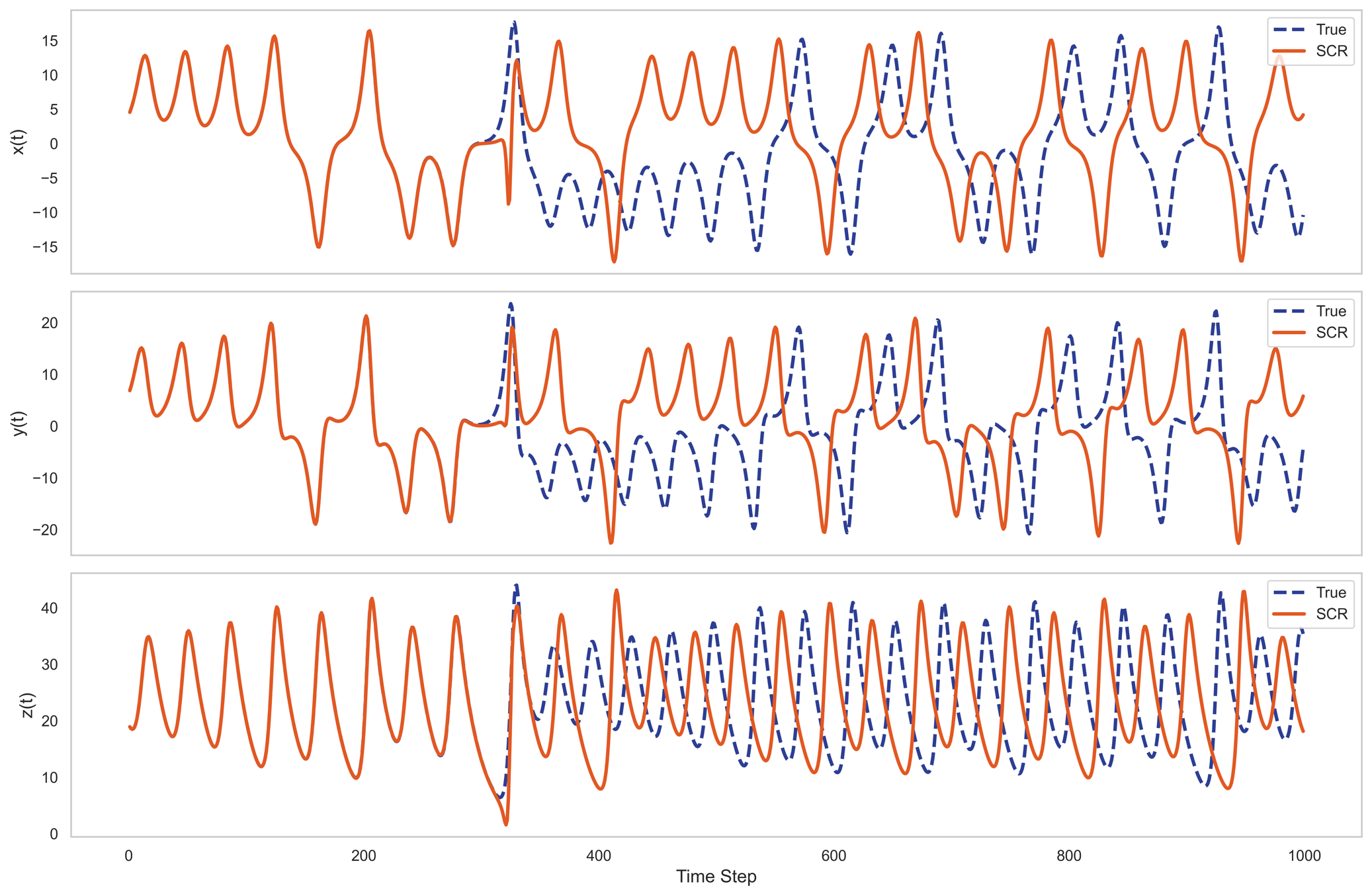}
            \textbf{(b)} SCR
    \end{minipage}

    \vspace{1em}

    \begin{minipage}{0.35\textwidth}
        \centering
        \includegraphics[width=\linewidth]{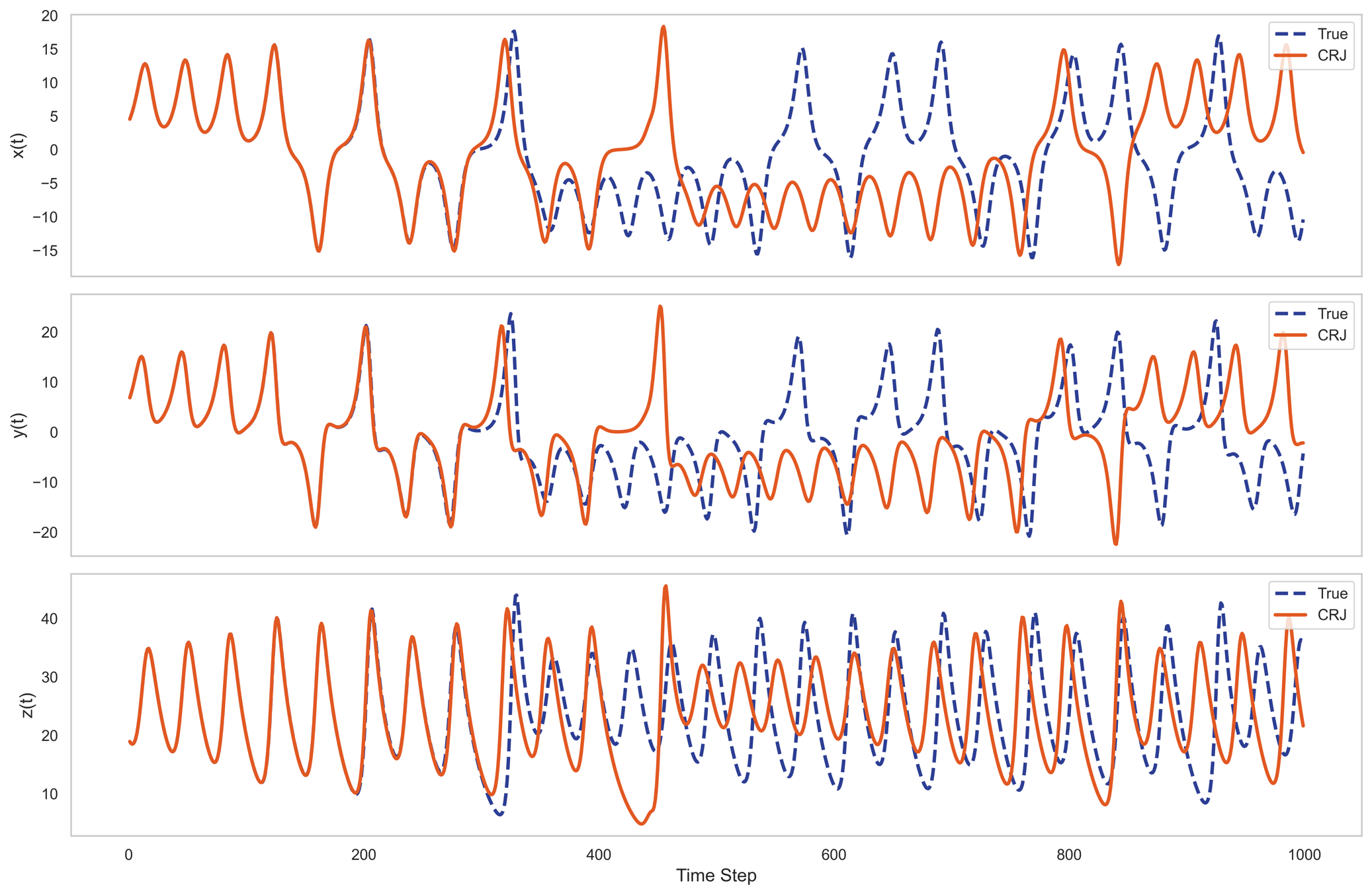}
            \textbf{(c)} CRJ
    \end{minipage} \hspace{0.05\textwidth}
    \begin{minipage}{0.35\textwidth}
        \centering
        \includegraphics[width=\linewidth]{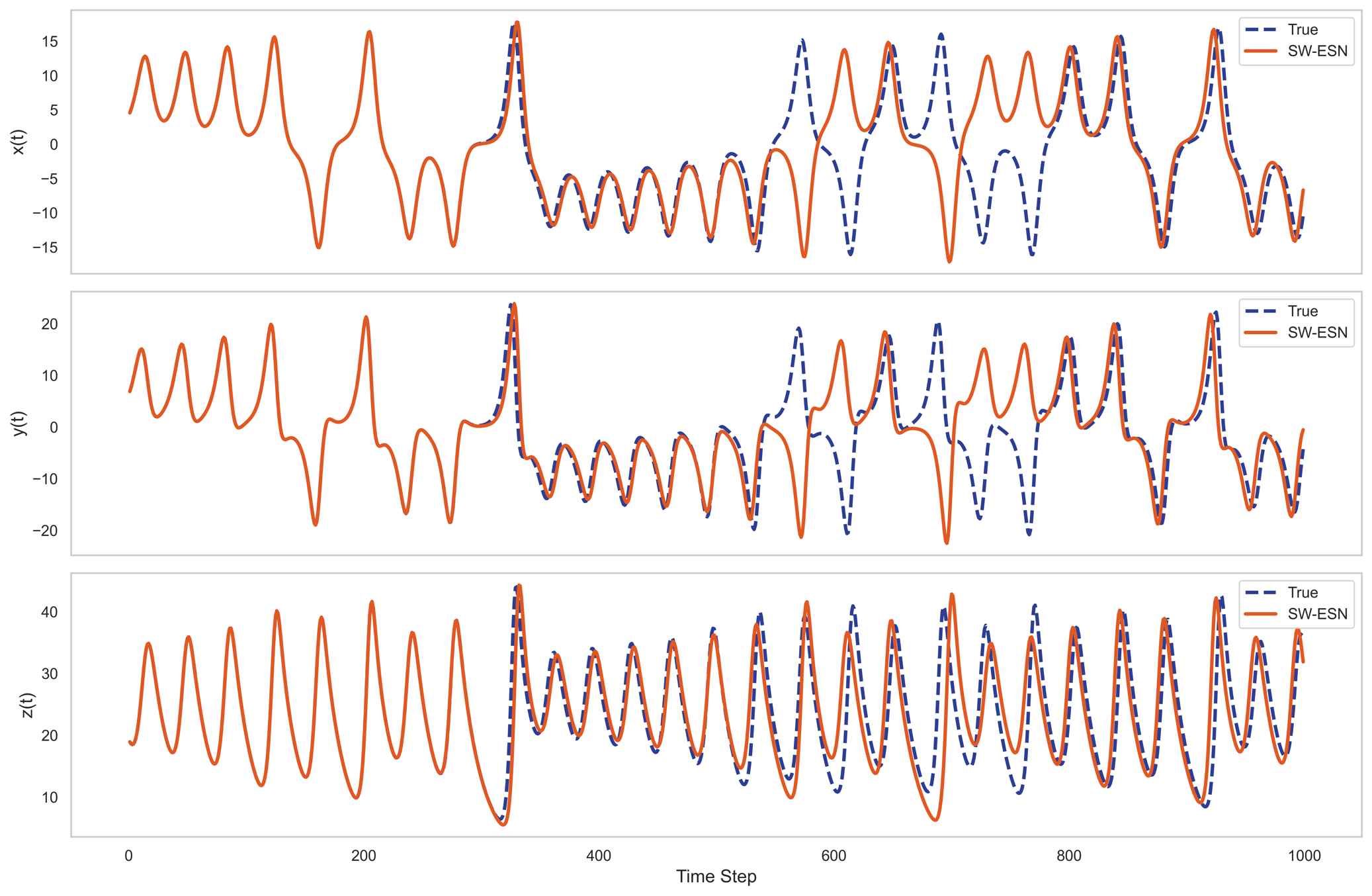}
    \textbf{(d)} SW-ESN
    \end{minipage}

    \vspace{1em}

    \begin{minipage}{0.35\textwidth}
        \centering
        \includegraphics[width=\linewidth]{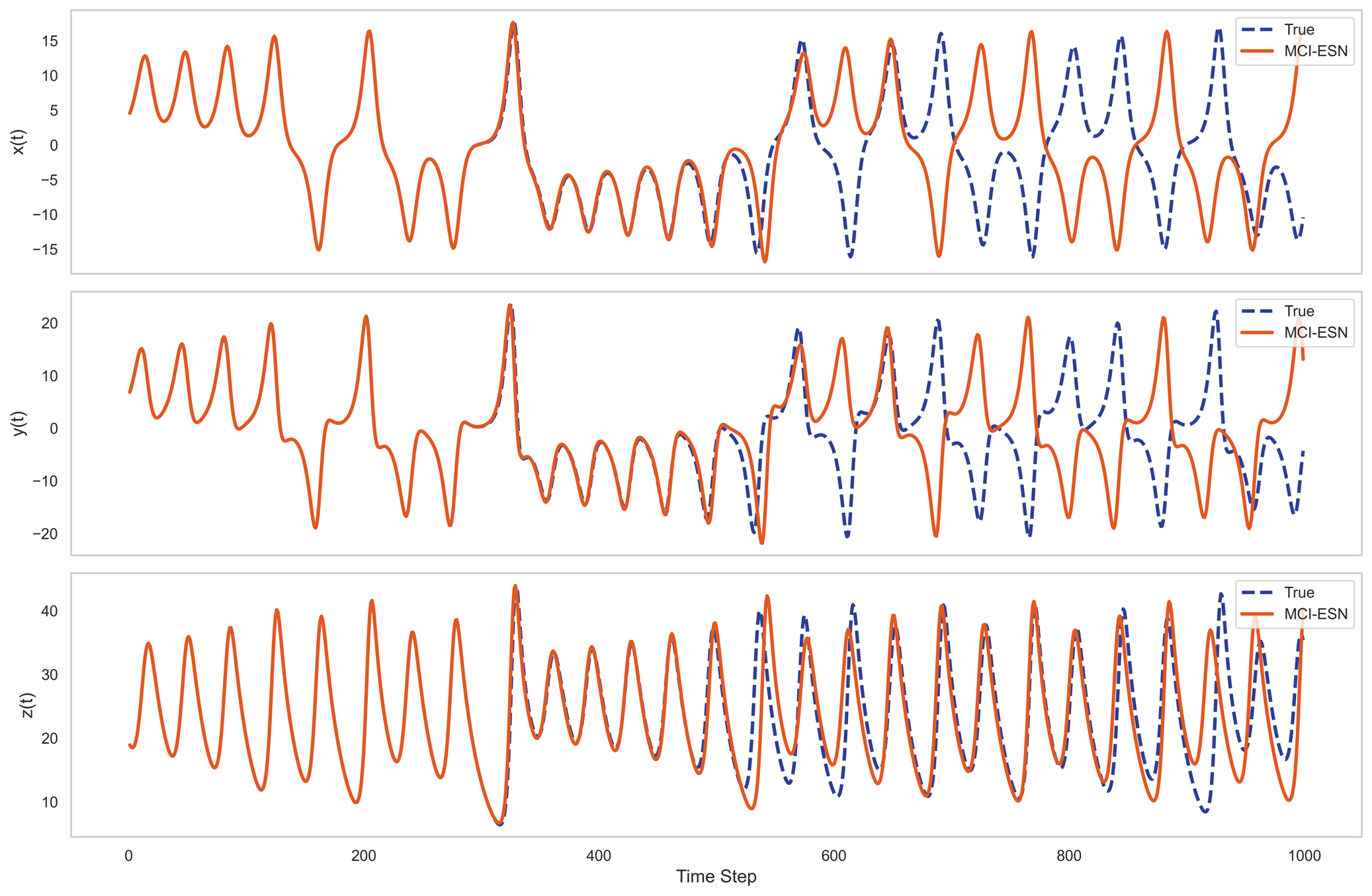}
      \textbf{(e)} MCI-ESN
    \end{minipage} \hspace{0.05\textwidth}
    \begin{minipage}{0.35\textwidth}
        \centering
        \includegraphics[width=\linewidth]{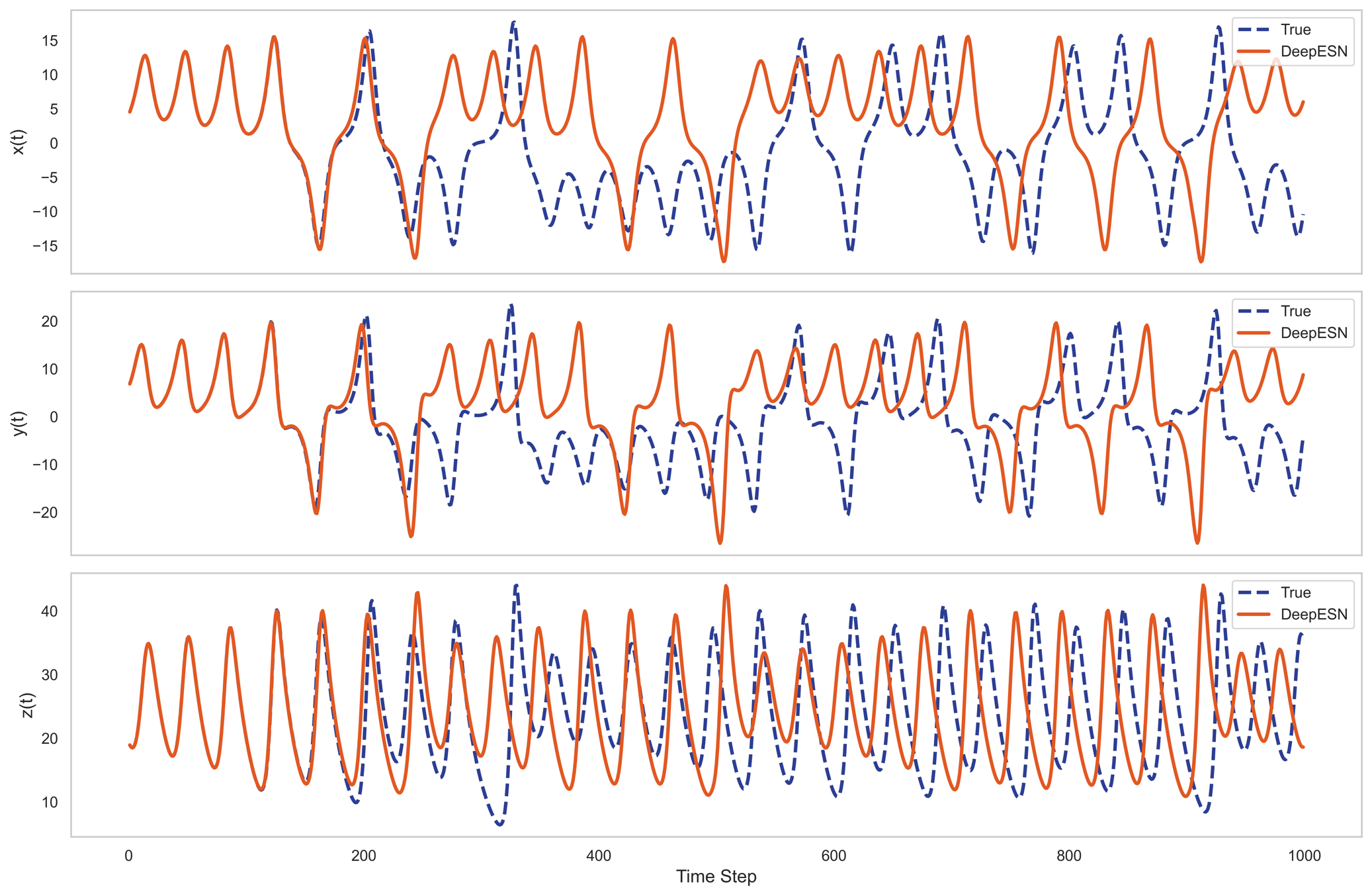}
    \textbf{(f)} DeepESN
    \end{minipage}

    \vspace{1em}

    \begin{minipage}{0.35\textwidth}
        \centering
        \includegraphics[width=\linewidth]{figures/Predicted_Trajectories_for_Test_Segment_Autoregressive_Forecasting.png}
    \textbf{(g)} HypER
    \end{minipage}

    \caption{Predicted trajectories by different reservoir architectures alongside the ground truth for the
test segment of the Lorenz system under autoregressive forecasting.}
    \label{fig: predicted-trajectories}
\end{figure*}

\section{Setup, Extended Results and Ablations}                 
\subsection{Datasets}

\paragraph{Sunspot Monthly.}  The \emph{International Sunspot Index v2.0} published by SILSO, Royal Observatory of Belgium, gives a homogeneous estimate of the mean total sunspot count for each calendar month from January 1749 to the present day, yielding \(T \approx 3{,}300\) temporally contiguous observations sampled at a fixed cadence of one observation per month \cite{SILSO2020}.  The series has long served as a canonical benchmark for nonlinear and chaos-theoretic forecasting studies because it combines a well-defined physical provenance with multi-century coverage and pronounced quasi-periodic components (the \(\sim\!11\)-year Schwabe cycle, its \(\sim\!22\)-year magnetic polarity counterpart, and longer Gleissberg modulations).  In all our experiments, the sunspot data is normalized to the $[0,1]$ range on the entire dataset, with the first 2,000 months (approximately 166 years) used for training, withholding the remaining data for out-of-sample evaluation. 

\paragraph{Santa Fe Dataset B.}  This multivariate trace was recorded from a patient in the sleep laboratory of Beth Israel Hospital (now Beth Israel Deaconess Medical Center) and released as Dataset B of the 1991 Santa Fe Time Series Prediction and Modelling Competition  \cite{Weigend1994,Jaeger2007leaky}.  The competition file contains simultaneous samples of three physiological parameters—heart rate, chest volume (respiration force), and blood oxygen concentration—measured in an evenly spaced sequence of 17000 samples, with no absolute timestamps; in the original laboratory electronics, successive samples were digitised every \(\Delta t = 0.5\,\mathrm{s}\), but that physical scale is intentionally omitted so that investigators treat one “time step’’ as the natural unit. We normalize the heart rate signal to the $[0,1]$ range and apply delay embedding with dimension 3 to reconstruct the system’s state space. We use the first 4,500 delay-embedded vectors for training and reserve the remaining portion for testing, forming a next-step forecasting setup in a chaotic, nonlinear regime.

\begin{figure}[!ht]
    \centering
    \includegraphics[width=\linewidth]{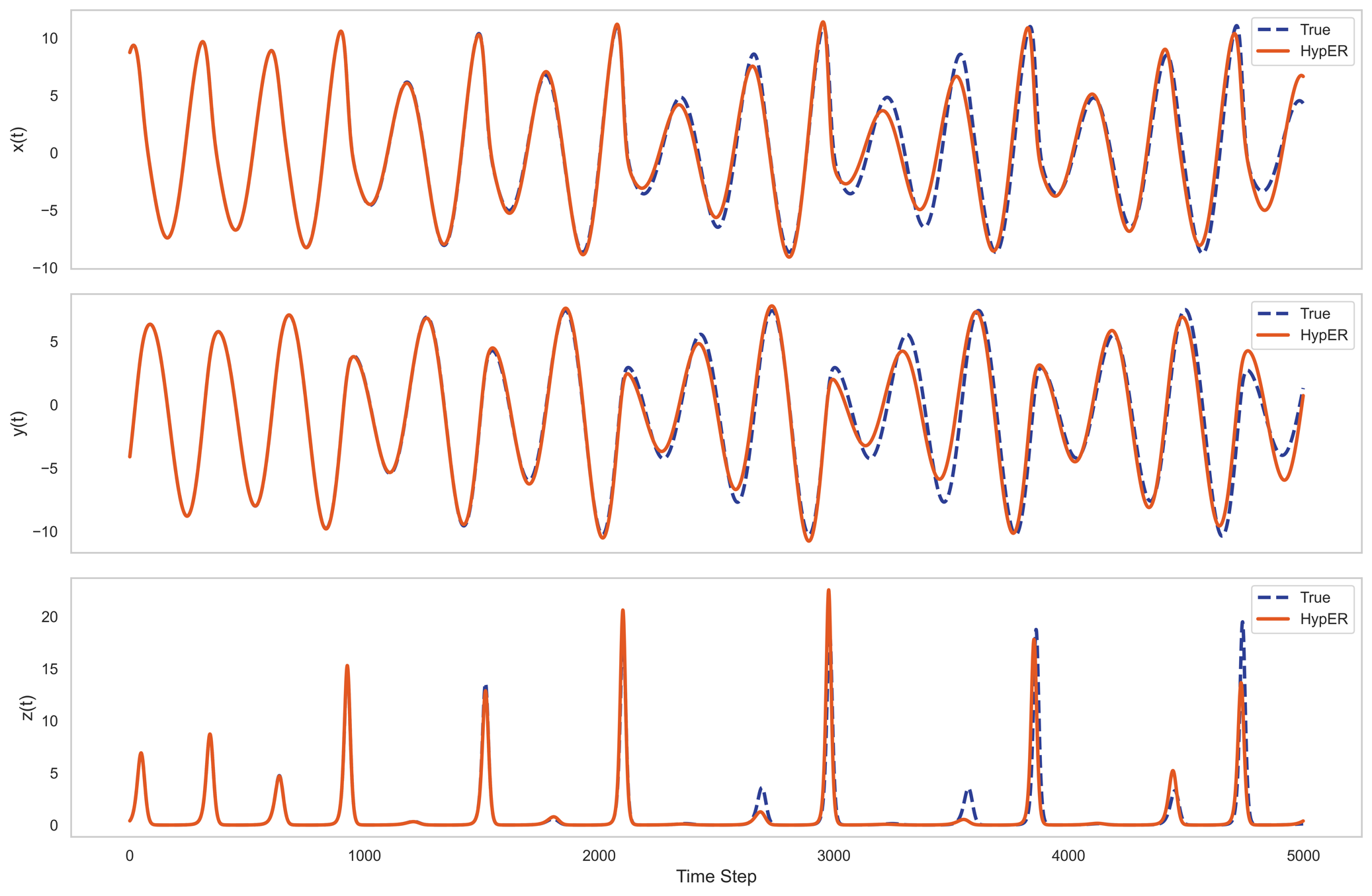}
    \caption{Predicted trajectories by HypER alongside ground truth for the test segment of the Rössler system under autoregressive forecasting.}
    \label{fig:trajectory_rossler}
\end{figure}

\begin{figure}[!ht]
    \centering
    \includegraphics[width=\linewidth]{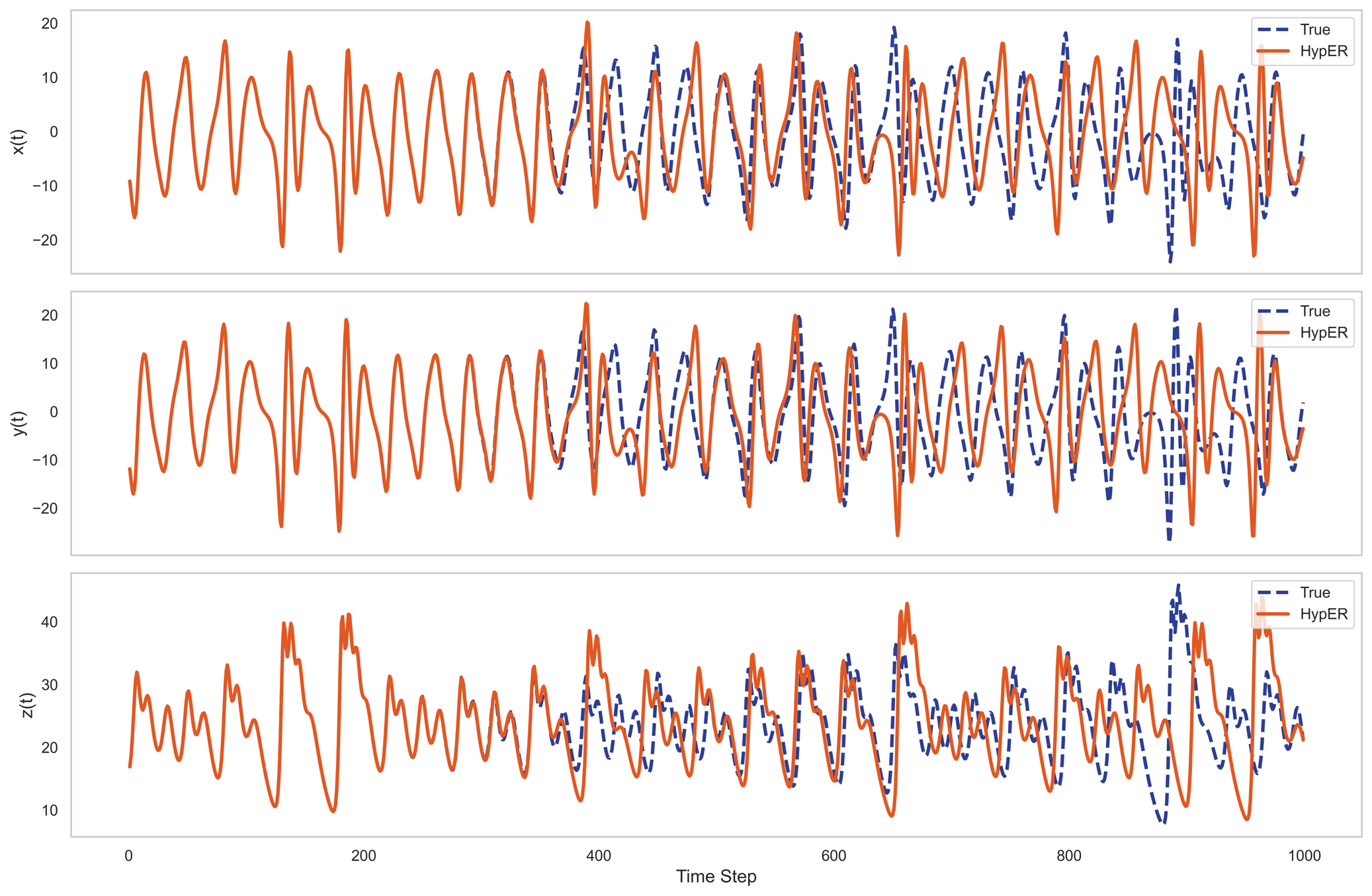}
    \caption{Predicted trajectories by HypER alongside ground truth for the test segment of the Chen system under autoregressive forecasting.}
    \label{fig:trajectory_rossler2}
\end{figure}

\begin{figure}[!ht]
  \centering

  \begin{minipage}[b]{0.49\linewidth}
    \centering
    \includegraphics[width=\textwidth]{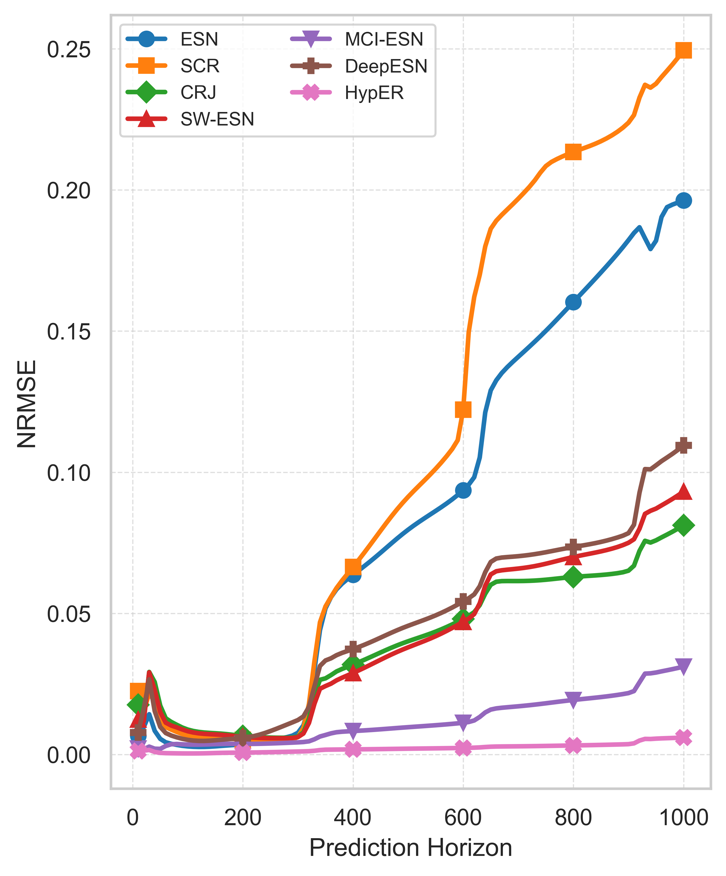}
    \textbf{(a)} Standard ESN
    \label{fig:nrmse_rossler1}
  \end{minipage}
  \hfill
  \begin{minipage}[b]{0.49\linewidth}
    \centering
    \includegraphics[width=\textwidth]{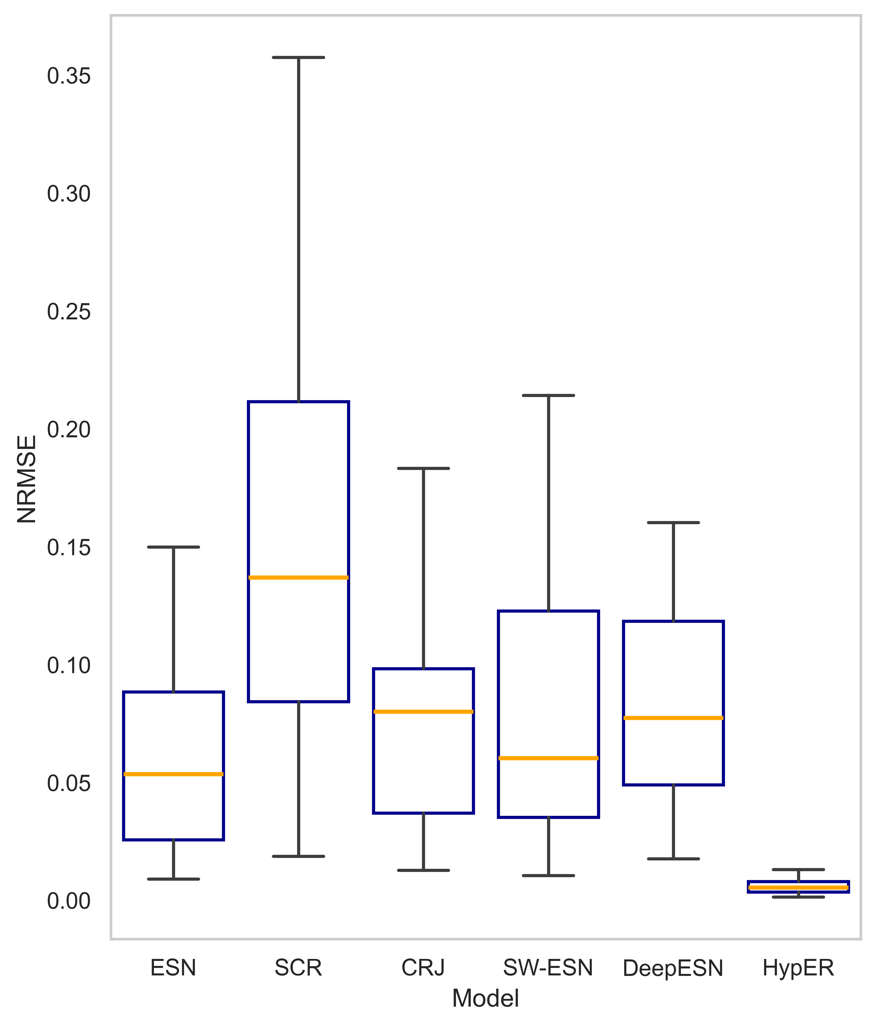}
    \textbf{(b)} Proposed HypER
    \label{fig:nrmse_rossler2}
  \end{minipage}

  \caption{NRMSE for autoregressive predictions across multiple horizons
for the Rössler system. }
  \label{fig:psdb}
\end{figure}

\begin{figure}[!ht]
  \centering

  \begin{minipage}[b]{0.49\linewidth}
    \centering
    \includegraphics[width=\textwidth]{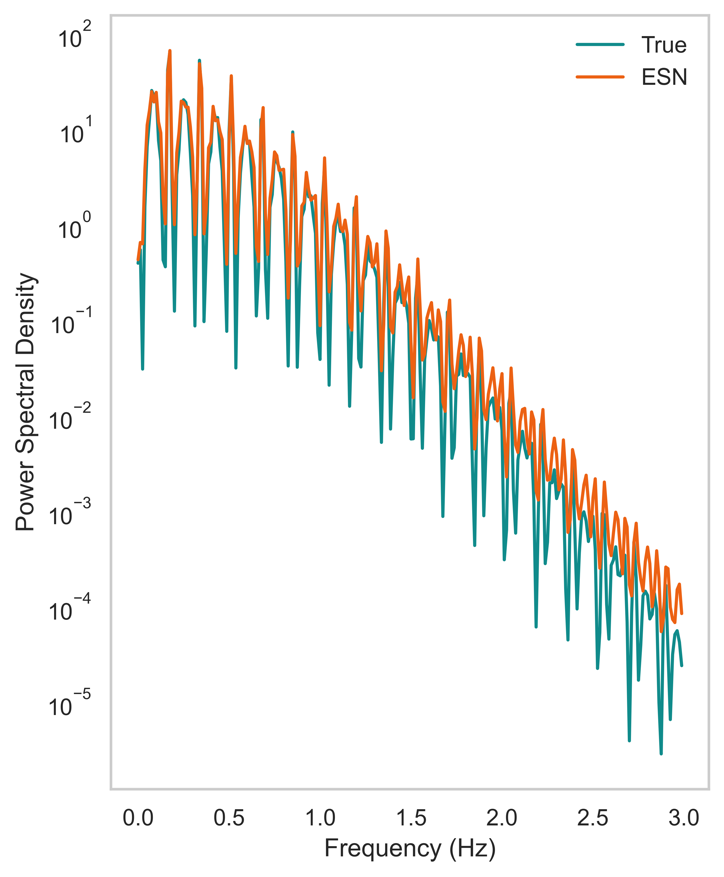}
    \textbf{(a)} Standard ESN
    \label{fig:psd_chen_esn}
  \end{minipage}
  \hfill
  \begin{minipage}[b]{0.49\linewidth}
    \centering
    \includegraphics[width=\textwidth]{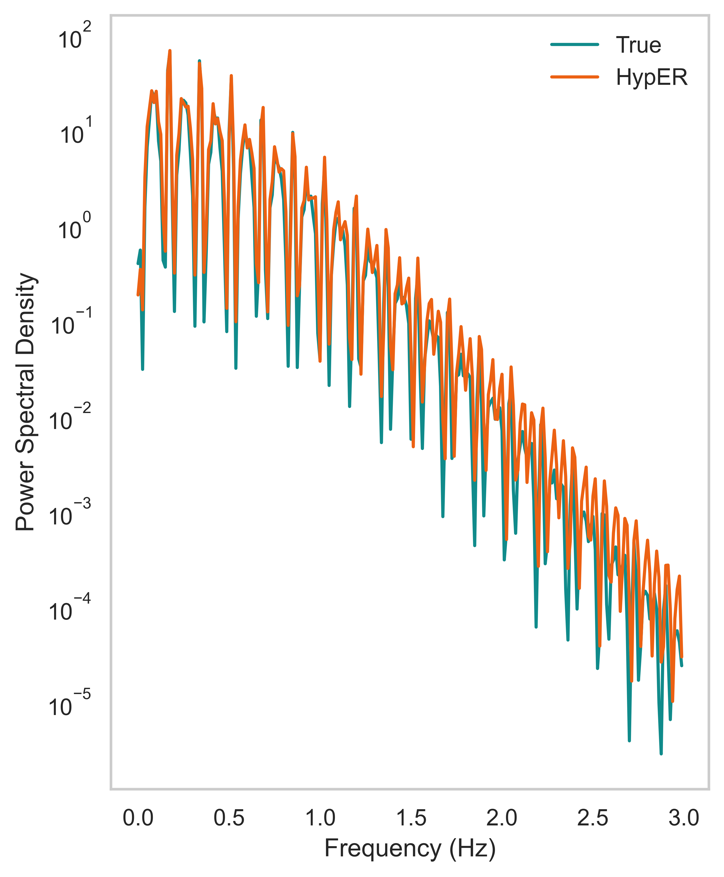}
    \textbf{(b)} Proposed HypER
    \label{fig:psd_rossler_hypER}
  \end{minipage}

  \caption{PSD plots of autoregressive predictions at a 3000-step horizon for (a) a standard ESN and (b) the proposed HypER, when both networks are driven by the Rössler system. }
  \label{fig:psd2}
\end{figure}

\begin{figure}[!ht]
  \centering

  \begin{minipage}[b]{0.49\linewidth}
    \centering
    \includegraphics[width=\textwidth]{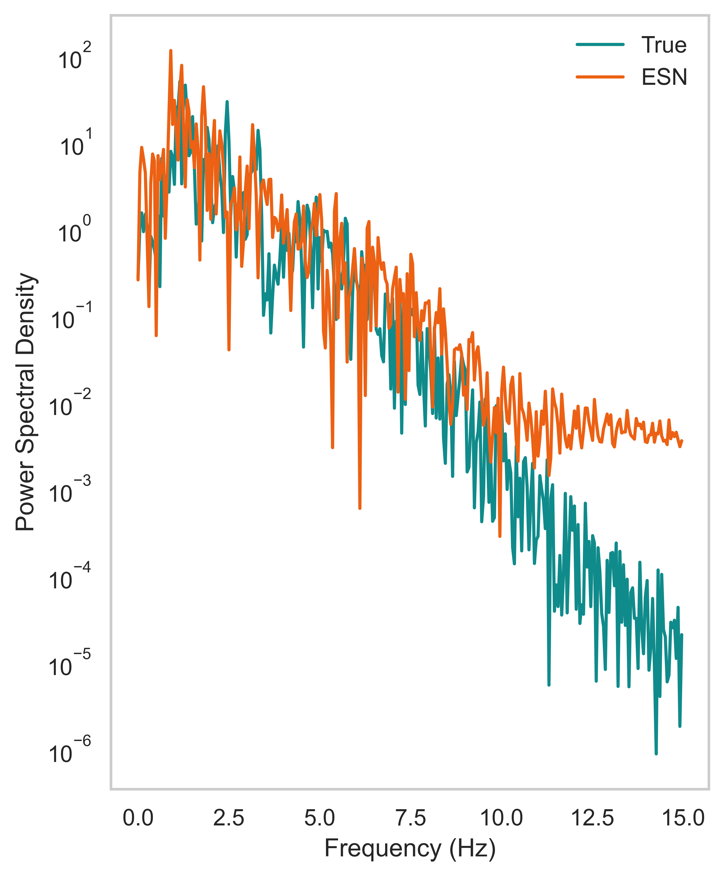}
    \textbf{(a)} Standard ESN
    \label{fig:psd_chen_esn2}
  \end{minipage}
  \hfill
  \begin{minipage}[b]{0.49\linewidth}
    \centering
    \includegraphics[width=\textwidth]{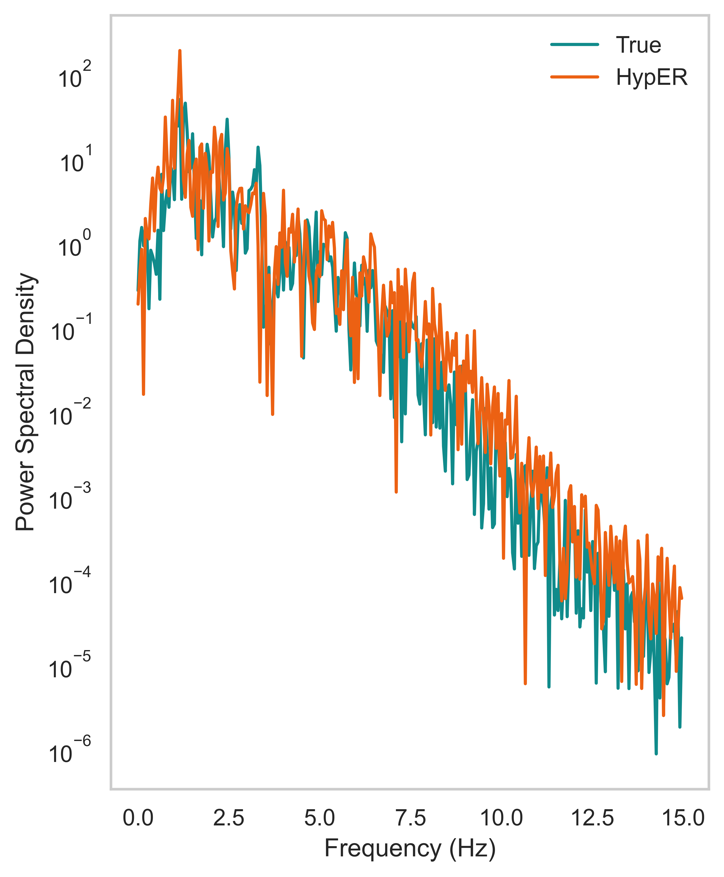}
    \textbf{(b)} Proposed HypER
    \label{fig:psd_rossler_hypER2}
  \end{minipage}

  \caption{PSD plots of autoregressive predictions at a 1000-step horizon for (a) a standard ESN and (b) the proposed HypER, when both networks are driven by the Chen system. }
  \label{fig:psd3}
\end{figure}

\begin{table}[!ht]
\centering
\caption{Ablation over canonical ESN parameters for HypER ($1000$-step autoregressive forecasting of the Lorenz dataset).}
\label{tab:ablation_hyper_params_esn}
\resizebox{0.36\textwidth}{!}{
\begin{tabular}{lccc}
\toprule
 \textbf{Setting}& \textbf{NRMSE$\downarrow$} & \textbf{Norm. VPT$\uparrow$} & \textbf{ADev$\downarrow$}\\
 \midrule
  \multicolumn{4}{c}{Spectral Radius}\\
  \midrule
 $\rho = 0.70$& 1.3176& 11.946& 20.53\\
 $\rho = 0.90$& 1.2800& 12.222& 17.70\\
 $\rho = 0.99$& 1.2580& 12.215& 19.67\\
 $\rho = 1.05$& 1.2884& 12.308& 11.00\\
  \midrule
 \multicolumn{4}{c}{Input Scale}\\
\midrule
 $\gamma= 0.1$& 1.4225& 11.498& 24.00\\
 $\gamma=0.2$& 1.2580& 12.215& 19.67\\
 $\gamma=0.3$& 1.3807& 12.170& 22.60\\
 $\gamma =0.5$& 1.4661& 10.879& 26.13\\
 $\gamma= 0.8$& 1.5473& 9.994& 26.13\\
  $\gamma= 1.0$& 1.6175& 8.899&26.10\\
  \midrule
 \multicolumn{4}{c}{Leak Rate}\\
 \midrule
  $\alpha= 0.5$& 1.3287& 11.916&22.07\\
  $\alpha= 0.7$& 1.3249& 12.182&20.57\\
  $\alpha= 0.8$& 1.2580& 12.215&19.67\\
  $\alpha= 0.9$& 1.3164& 11.935&17.33\\
  $\alpha= 1.0$& 1.2997& 11.973&19.00\\
  \bottomrule
\end{tabular}
}
\end{table}

\begin{table}[!ht]
\centering
\caption{Parameter settings and initial states for benchmark chaotic systems.}
\label{tab:system-params}
\resizebox{0.45\textwidth}{!}{
\begin{tabular}{lcc}
\toprule
\textbf{System} & \textbf{Parameters} & \textbf{Initial State}\\
\midrule
Lorenz & $\sigma=10$, $\rho=28$, $\beta=8/3$ & (1.0, 1.0, 1.0) \\
\midrule
Rössler & $a=0.2$, $b=0.2$, $c=5.7$ & (1.0, 0.0, 0.0)\\
\midrule
 Chen& $a=35$, $b=3$, $c=28$&(1.0, 1.0, 1.0)\\
\midrule
\multirow{2}{*}{Chua} & $\alpha=15.6$, $\beta=28$, & \multirow{2}{*}{(0.2, 0.0, 0.0)}\\
 & $m_0=-1.143$, $m_1=-0.714$ & \\
 \midrule
 Mackey-Glass& $\beta=0.2, \gamma=0.1, \tau=17$&0.2\\
 \bottomrule
\end{tabular}
}
\end{table}

\begin{table}[!ht]
\centering
\caption{Hyperparameter configurations of different reservoir models for the Lorenz dataset.}
\label{tab:model-hyperparams}
\resizebox{0.49\textwidth}{!}{
\begin{tabular}{llc}
\toprule
\textbf{Model}& \textbf{Hyperparameter}& \textbf{Value}\\
\midrule
\multirow{2}{*}{ESN}& Reservoir Size& 300\\
 & Connectivity Ratio&0.05\\
\midrule
 \multirow{2}{*}{SCR}& Reservoir Size&300\\
 & Edge Weight&0.8\\
 \midrule
 \multirow{3}{*}{CRJ}& Reservoir Size&300\\
 & Edge Weight&0.8\\
 & Jump Size&10\\
 \midrule
 \multirow{3}{*}{SW-ESN}& Reservoir Size&300\\
 & Node Degree $E$&2\\
 & Rewiring Probability $p$&0.3\\
 \midrule
 \multirow{4}{*}{MCI-ESN}& Sub-reservoir Size&300\\
 & Edge Weight $\mu$&0.8\\
 &  Inter-reservoir Connection Weight $\eta$&0.8\\
 & Weight Coefficient $\theta$&0.1\\
 \midrule
 \multirow{2}{*}{DeepESN}& Number of Layers&3\\
 & Reservoir Sizes&100, 100, 100\\
 \midrule
 \multirow{3}{*}{HypER}& Reservoir Size&300\\
 & Kernel Width $\sigma$&0.1\\
 & Row-level Sparsity $\kappa$ &40\\
 \bottomrule
\end{tabular}
}
\end{table}

\begin{table}[!ht]
\centering
\caption{Hyperparameter configurations of different reservoir models for the Rössler dataset.}
\label{tab:model-hyperparams2}
\resizebox{0.49\textwidth}{!}{
\begin{tabular}{llc}
\toprule
\textbf{Model}& \textbf{Hyperparameter}& \textbf{Value}\\
\midrule
\multirow{2}{*}{ESN}& Reservoir Size& 300\\
 & Connectivity Ratio&0.001\\
\midrule
 \multirow{2}{*}{SCR}& Reservoir Size&300\\
 & Edge Weight&0.001\\
 \midrule
 \multirow{3}{*}{CRJ}& Reservoir Size&300\\
 & Edge Weight&0.001\\
 & Jump Size&5\\
 \midrule
 \multirow{3}{*}{SW-ESN}& Reservoir Size&300\\
 & Node Degree $E$&3\\
 & Rewiring Probability $p$&0.005\\
 \midrule
 \multirow{4}{*}{MCI-ESN}& Sub-reservoir Size&300\\
 & Edge Weight $\mu$&0.8\\
 &  Inter-reservoir Connection Weight $\eta$&0.3\\
 & Weight Coefficient $\theta$&0.5\\
 \midrule
 \multirow{2}{*}{DeepESN}& Number of Layers&3\\
 & Connectivity Ratio&0.0005\\
 & Reservoir Sizes&100, 100, 100\\
 \midrule
 \multirow{3}{*}{HypER}& Reservoir Size&300\\
 & Kernel Width $\sigma$&0.05\\
 & Row-level Sparsity $\kappa$ &5\\
 \bottomrule
\end{tabular}
}
\end{table}

\begin{table}[!ht]
\centering
\caption{Hyperparameter configurations of different reservoir models for the Chen-Ueta dataset.}
\label{tab:model-hyperparams3}
\resizebox{0.49\textwidth}{!}{
\begin{tabular}{llc}
\toprule
\textbf{Model}& \textbf{Hyperparameter}& \textbf{Value}\\
\midrule
\multirow{2}{*}{ESN}& Reservoir Size& 300\\
 & Connectivity Ratio&0.05\\
\midrule
 \multirow{2}{*}{SCR}& Reservoir Size&300\\
 & Edge Weight&0.8\\
 \midrule
 \multirow{3}{*}{CRJ}& Reservoir Size&300\\
 & Edge Weight&0.8\\
 & Jump Size&10\\
 \midrule
 \multirow{3}{*}{SW-ESN}& Reservoir Size&300\\
 & Node Degree $E$&2\\
 & Rewiring Probability $p$&0.3\\
 \midrule
 \multirow{4}{*}{MCI-ESN}& Sub-reservoir Size&300\\
 & Edge Weight $\mu$&0.8\\
 &  Inter-reservoir Connection Weight $\eta$&0.8\\
 & Weight Coefficient $\theta$&0.1\\
 \midrule
 \multirow{2}{*}{DeepESN}& Number of Layers&3\\
 & Reservoir Sizes&100, 100, 100\\
 \midrule
 \multirow{3}{*}{HypER}& Reservoir Size&300\\
 & Kernel Width $\sigma$&0.2\\
 & Row-level Sparsity $\kappa$ &10\\
 \bottomrule
\end{tabular}
}
\end{table}

\begin{table}[!ht]
\centering
\caption{Hyperparameter configurations of different reservoir models for the Chua dataset.}
\label{tab:model-hyperparams4}
\resizebox{0.49\textwidth}{!}{
\begin{tabular}{llc}
\toprule
\textbf{Model}& \textbf{Hyperparameter}& \textbf{Value}\\
\midrule
\multirow{2}{*}{ESN}& Reservoir Size& 300\\
 & Connectivity Ratio&0.3\\
\midrule
 \multirow{2}{*}{SCR}& Reservoir Size&300\\
 & Edge Weight&0.8\\
 \midrule
 \multirow{3}{*}{CRJ}& Reservoir Size&300\\
 & Edge Weight&0.8\\
 & Jump Size&15\\
 \midrule
 \multirow{3}{*}{SW-ESN}& Reservoir Size&300\\
 & Node Degree $E$&2\\
 & Rewiring Probability $p$&0.3\\
 \midrule
 \multirow{4}{*}{MCI-ESN}& Sub-reservoir Size&300\\
 & Edge Weight $\mu$&0.8\\
 &  Inter-reservoir Connection Weight $\eta$&0.8\\
 & Weight Coefficient $\theta$&0.1\\
 \midrule
 \multirow{2}{*}{DeepESN}& Number of Layers&3\\
 & Connectivity Ratio&0.5\\
 & Reservoir Sizes&100, 100, 100\\
 \midrule
 \multirow{3}{*}{HypER}& Reservoir Size&300\\
 & Kernel Width $\sigma$&0.4\\
 & Row-level Sparsity $\kappa$ &60\\
 \bottomrule
\end{tabular}
}
\end{table}

\paragraph{MIT–BIH Arrhythmia.}  The MIT–BIH Arrhythmia Database comprises 48 half-hour two-lead electrocardiogram recordings collected at Boston’s Beth Israel Hospital between 1975 and 1979, sampled at \(360\,\text{Hz}\) with 11-bit resolution over a \(\pm5\,\text{mV}\) range \cite{MoodyMark2001,Goldberger2000}.  Each record includes expert beat-level and rhythm-level annotations, yielding approximately \(109\,000\) labelled heartbeats across a diverse population of sinus rhythm, premature contractions, and less common arrhythmias.  For our single-channel forecasting experiments, we select lead II from record 100, extracting the first 25,000 samples (\textasciitilde70 seconds) at the original 360 Hz sampling rate, and normalize the signal. To capture temporal dynamics, we apply delay embedding with an embedding dimension of 3. The resulting sequence poses a markedly different prediction challenge from the two chaotic physics data sets above: it combines quasi-periodicity with abrupt morphological changes driven by sporadic ectopic beats, providing a stringent test of the reservoir’s ability to model real-world biomedical variability.

\paragraph{Lorenz–63.}  
The three–variable Lorenz equations  
\(
\dot x=\sigma(y-x),\;
\dot y=x(\rho-z)-y,\;
\dot z=xy-bz
\)
with the canonical parameters \(\sigma=10,\;\rho=28,\;b=8/3\) form the archetype of deterministic chaos: the flow possesses one positive Lyapunov exponent \(\lambda_{\max}\approx0.905\) and a strange attractor of Hausdorff dimension \(\approx2.06\) \cite{lorenz1963deterministic}.  After discarding an initial transient, we harvest a 1000-step window whose variability spans both wings of the well-known butterfly attractor, yielding a balanced long-horizon prediction task.

\paragraph{Rössler.}  
The Rössler system  
\(
\dot x=-y-z,\;
\dot y=x+ay,\;
\dot z=b+z(x-c)
\)
with \(a=0.2,\;b=0.2,\;c=5.7\) exhibits a single-scroll chaotic attractor and a one-dimensional first-return map that is topologically conjugate to a unimodal logistic map \cite{rossler1976equation}.  Its weaker nonlinearity relative to Lorenz leads to a smaller positive Lyapunov exponent (\(\approx0.071\)), making it an instructive contrast for assessing how geometry-aware reservoirs cope with slow error growth.

\paragraph{Chen–Ueta.}  
A structural perturbation of the Lorenz equations, the Chen–Ueta flow  
\(
\dot x=a(y-x),\;
\dot y=(c-a)x-xz+cy,\;
\dot z=xy-bz
\)
with \(a=35,\;b=3,\;c=28\) generates a \emph{hyper-chaotic} attractor that possesses two positive Lyapunov exponents of comparable magnitude \cite{chen1999yet}.  The resulting 1000-step segment therefore contains higher local instability and richer folding behaviour than either Lorenz or Rössler.

\paragraph{Chua.}  
Chua’s circuit is governed by a third-order piecewise-linear ordinary differential equation whose double-scroll attractor is reproducible with off-the-shelf electronic components \cite{chua1986double}.  The non-smooth nonlinearity produces sharp turns and nearly linear flight phases; these kinks are notoriously challenging for smooth recurrent models but provide an excellent diagnostic for geometry-induced diversity in the reservoir state.

\paragraph{Mackey–Glass.}  
The Mackey–Glass delay-differential equation  
\(
\dot x(t)=\beta\,x(t-\tau)/(1+x(t-\tau)^n)-\gamma x(t)
\)
with \(\beta=0.2,\;\gamma=0.1,\;n=10,\;\tau=17\) yields an infinite-dimensional chaotic attractor whose scalar observation exhibits long-memory amplitude modulations \cite{mackey1977oscillation}.  Unlike the three-dimensional ODEs above, this DDE forces the reservoir to model history-dependent dynamics over a latent state space that cannot be embedded in any finite Euclidean dimension, providing a stringent test of the inductive bias introduced by hyperbolic connectivity.


\subsection{Baselines}\label{sec:baselines}

To gauge the benefit of hyperbolic wiring we compare HypER with six widely–used
reservoir variants, all trained with identical ridge read-outs and data splits.
\begin{itemize}[leftmargin=1.5em]
\item \textit{ESN} \cite{jaeger2001echo}: Erdős–Rényi reservoir, spectral radius and
      input scale tuned on a $3\times3$ logarithmic grid.
\item \textit{SCR} \cite{li2024simplecyclereservoirsuniversal}: single directed cycle
      with fixed edge weight; only the weight magnitude is tuned.
\item \textit{CRJ} \cite{rodan2011simple}: cycle reservoir with uniform jump
      connections; we sweep the jump length and edge weight.
\item \textit{SW-ESN} \cite{kawai2019smallworld}: Watts–Strogatz small-world
      reservoir; node degree~$E$ and rewiring probability~$p$ are optimised.
\item \textit{MCI-ESN} \cite{liu2024minimum}: two sparsely coupled ESNs (sizes
      fixed to match HypER), with intra- and inter-reservoir weights
      $\mu,\eta$ and combination coefficient $\theta$ selected by grid search.
\item \textit{DeepESN} \cite{gallicchio2020deepechostatenetwork}: three stacked
      reservoirs of equal size, common spectral radius and leak tuned as above.
\end{itemize}
All baselines use the same reservoir size as HypER (300 units per layer where
applicable) and the same wash-out, regularisation and optimisation settings; the
exact hyper-parameters chosen for each data set are listed in
Tables~\ref{tab:model-hyperparams}–\ref{tab:model-hyperparams4}.

\subsection{CDFs for dimensions 3 and 4}

In the hyperbolic-uniform sampling scheme, every radial coordinate \(\rho\) is drawn with respect to the \((d{-}1)\)-dimensional hyperbolic volume form, whose density is \(p_d(\rho)\propto\sinh^{d-1}\!\rho\).  For \(d=3\) this gives \(p_3(\rho)\propto\sinh^{2}\!\rho\).  Integrating once and choosing the normalisation constant so that \(F_3(\rho_{\max})=1\) produces  
\begin{equation}
F_3(\rho)=\frac{\sinh(2\rho)-2\rho}{\sinh(2\rho_{\max})-2\rho_{\max}}\;,
\qquad 0\le\rho\le\rho_{\max}.
\end{equation}
Inverse-transform sampling therefore, requires solving  
\(\sinh(2\rho)-2\rho = u\,[\sinh(2\rho_{\max})-2\rho_{\max}]\)  
for each independent \(u\sim\operatorname{Uniform}(0,1)\).  Because the left-hand side is strictly increasing, a single Newton iteration  
\(f(\rho)=\sinh(2\rho)-2\rho-\text{RHS},\; f'(\rho)=2\cosh(2\rho)-2\)  
converges quadratically and is numerically stable even as \(\rho_{\max}\to\infty\) (where the condition number improves exponentially).

For \(d=4\) the density becomes \(p_4(\rho)\propto\sinh^{3}\!\rho\).  A straightforward antiderivative yields  
\begin{equation}
F_4(\rho)=\frac{\cosh^{3}\!\rho-3\cosh\rho+2}
                 {\cosh^{3}\!\rho_{\max}-3\cosh\rho_{\max}+2},
\end{equation}
so the inverse CDF is defined implicitly by  
\(\cosh^{3}\!\rho-3\cosh\rho+2
     = u\,[\cosh^{3}\!\rho_{\max}-3\cosh\rho_{\max}+2]\).
Setting \(g(\rho)=\cosh^{3}\!\rho-3\cosh\rho+2-\text{RHS}\) and noting  
\(g'(\rho)=3\sinh\rho(\cosh^{2}\!\rho-1)=3\sinh^{3}\!\rho>0\),  
Newton–Raphson again gives a monotone and rapidly convergent solver.

\subsection{Training Protocol}\label{sec:training}
For every data set the raw sequence is divided chronologically into three disjoint segments: an initial \emph{wash-out} period that lasts $100$ reservoir updates (longer than the maximal memory length observed in our leaky dynamics), a fitting window that spans the next $4500$ available samples and supplies the design matrix $X\in\mathbb{R}^{T\times N}$ for ridge regression, and a hold-out window comprising the remaining $1000$ for hyper-parameter selection and final scoring.  During fitting, the reservoir is driven in \emph{teacher-forcing} mode so that the hidden state trajectory is uniquely determined by the past inputs and the wash-out renders the dependence on $x_{0}$ negligible, thereby satisfying the \emph{Echo–State Property} in practice.  The read-out weight matrix $\mathbf{W}_{\mathrm{out}}^*
  \;=\;\mathbf{Y}\,\mathbf{\Xi}^{\mathsf{T}} (\mathbf{\Xi} \mathbf{\Xi}^{\mathsf{T}} + \lambda I)^{-1}$ is solved in double precision by Householder–QR to avoid squaring the condition number; subsequently, the network is run in closed loop by feeding its own one-step forecasts back as input.  Performance is summarised by the \emph{valid prediction time} $T_{\mathrm{VPT}}$, defined as the first horizon at which the normalised RMSE exceeds a system-specific threshold—0.40 for Lorenz-63, 0.50 for Rössler, and 0.30 for Chen—consistent with conventions in reservoir-computing studies of chaotic flows \cite{pathak2018hybrid}.  All ADev evaluations were conducted using a uniform cube size of $4 \times 4 \times 4$ across datasets to ensure consistency in spatial resolution. 

\subsection{Extended Results}\label{sec:extended}

Table~\ref{tab:nrmse_horizonsb} reports the one–step\,$\rightarrow$\,multi–step teacher–forced
NRMSE over multiple prediction horizons for every chaotic benchmark.
Across all systems and horizons, HypER is the only model whose error stays
below $3\times10^{-3}$ for Lorenz, Rössler, Chen, and Mackey–Glass and below
$1.2\times10^{-3}$ for the more dissipative Chua circuit, outperforming the
strongest Euclidean baselines (MCI–ESN or DeepESN) by one to three orders of
magnitude.  The standard deviation columns confirm that this advantage is
statistically robust over 30 random seeds.  Table~\ref{tab:ablation_hyper_params_esn}
examines the three canonical ESN hyperparameters for HypER in isolation.  Error is
minimal at $\rho\!\approx\!0.99$, $\gamma\!\approx\!0.2$, and
$\alpha\!\approx\!0.8$, a region that satisfies
$\beta(\sigma)\!>\!1$ and thus corroborates the theoretical guideline derived
in \S\ref{sec:methodology}: pushing $\rho$ or $\gamma$ higher sacrifices the
echo–state margin, whereas smaller values erode the guaranteed expansion and
reduce the normalised VPT.  Finally,
Table~\ref{tab:system-params} lists the ODE or delay-differential
parameters and initial conditions used to generate the training and test
trajectories; these match the classic values in the chaos-forecasting
literature and enable exact replication of every curve reported in the main document. Dataset-specific hyper-parameter grids for every baseline and for HypER are summarised in Tables \ref{tab:model-hyperparams}–\ref{tab:model-hyperparams4}; these reproduce all scores in the main text with a single seed-controlled run.

\subsection{Hyperparameter Grid and Final Choices}\label{sec:hypers}
A single grid is used for all data sets to demonstrate that HypER is robust across operating regimes.  The search spans  leak rate $\alpha\in\{0.1,0.3,0.5,0.8,1.0\}$, target spectral radius $\varrho\in\{0.7,0.8,0.9,0.95\}$, kernel width $\sigma\in\{0.05,0.08,0.12,0.20\}$ (expressed in units of the mean hyperbolic distance), out-degree $\kappa\in\{5,10,20\}$ after sparsification, and ridge constant $\lambda\in\{10^{-4},10^{-5},10^{-6}\}$.  Each configuration is evaluated once on the validation slice, and the tuple that maximises VPT averaged over three random seeds is retained.  The same setting $(\alpha,\varrho,\sigma,\kappa,\lambda)=(0.8,0.95,0.12,10,10^{-6})$ that emerges as either the top or statistically tied on every benchmark is frozen, and we use those values for the final test runs reported in the paper on thirty seeds.  That invariance supports the claim that the negative-curvature geometry, rather than fine-grained tuning, is the principal driver of the forecasting improvement.

\subsection{Complexity}
\label{sec:extensions}

\textit{Time Complexity.} Constructing the hyperbolic adjacency matrix \(\mathbf{W}\) naively requires \(O(N^2)\) distance computations, plus exponentiation. However, imposing sparsity \(\kappa\) yields a final matrix \(\mathbf{W}\) with \(O(\kappa N)\) nonzeros. The spectral normalization is dominated by an eigenvalue computation, \(O(N^3)\) or faster if iterative methods are used. The forward pass of each ESN update is \(O(\kappa N)\) per time-step.
 While \(\tanh(\cdot)\) is standard, the hyperbolic setting admits possibilities for \emph{node-varying} nonlinearities, e.g.\ \(\phi_i(\cdot)\) that adapt to the radial position of node \(i\). We can also consider \emph{Möbius addition} or other manifold-specific operations \cite{ganea2018hyperbolic}, although we adopt a simpler Euclidean gating inside the reservoir state vector for computational tractability.

\subsection{Numerical Stability Safeguards}\label{sec:stability}

All experiments were engineered to preclude floating-point overflow, underflow, or catastrophic loss of significance, without compromising the generic reproducibility of the code.

\textit{Spectral control of the reservoir.}  After computing the hyperbolic kernel
$ \mathbf{\widetilde{W}}_{ij}=\exp\ \!\bigl(-d_{\mathbb H}(p_i,p_j)/\sigma\bigr)$
in \texttt{float64} precision, we apply the similarity-free rescaling
$\mathbf{W}=\varrho\,\widetilde W/\rho(\widetilde W)$ with
$\varrho<1$.  This guarantees that the global Lipschitz constant of the leaky ESN map
$\mathbf{x}_{t+1}=(1-\alpha)\mathbf{x}_t+\alpha\phi(\mathbf{W}\mathbf{x}_t+\mathbf{Uu}_t)$
is strictly below unity for any $\alpha\le 1$ whenever the activation derivative is bounded, and therefore enforces the Echo–State Property irrespective of input magnitude.  The spectral radius is measured with an implicitly restarted Lanczos routine that yields machine-precision eigenvalues even for $N=10^{4}$.

\textit{Kernel underflow.}  For very large hyperbolic separations, the naïve exponential may underflow; we therefore clip the argument so that $\widetilde W_{ij}\ge 10^{-16}$, a threshold far below the sparsity mask applied subsequently and hence neutral with respect to graph connectivity.  

\textit{State-space conditioning.}  Reservoir states are centred online and divided by a running standard deviation (time constant \(10^{3}\) steps) before entering the ridge-regression read-out.  The Tikhonov coefficient is fixed at $\lambda=10^{-6}$, small enough not to bias the solution yet sufficient to bound the condition number of $X^\top X$ below \(10^{8}\) in every data set.  Linear systems are solved with a QR factorisation rather than a normal–equation inversion to avoid squaring the condition number.

\textit{Runtime precision policy.}  Weight matrices remain in \texttt{float64} during construction; after spectral scaling they are cast to \texttt{float32} for GPU execution.  All eigendecompositions, QR factorizations, and Lyapunov-spectrum computations are carried out on the CPU in 64-bit arithmetic, following the recommendations of \cite{Higham2002} for mixed-precision numerical linear algebra.

We ensure that no experiment exhibits numerical divergence, that reservoir dynamics remain within the theoretical bounds established by Lemma~\ref{lemma2b}, and that predictive performance variations originate from model geometry rather than artefacts of finite-precision arithmetic.

\subsection{Software License, Dependencies and Compute Budget}\label{sec:hardware}

All code and pretrained weights accompanying this paper will be released under the
\textit{MIT License}, permitting unrestricted academic and commercial use provided
that the original copyright notice is retained.  Experiments were run with
\textit{Python\,3.10.13} on \textit{Ubuntu\,22.04}; key third-party packages (exact
\texttt{pip} hashes are pinned in the repository) are:
\begin{itemize}
\item \texttt{PyTorch\,2.2.2} (\texttt{CUDA\,12.1}) for tensor operations and GPU acceleration,
\item \texttt{NumPy\,1.25.2}, \texttt{SciPy\,1.15.2}, \texttt{seaborn\,0.12.2} and \texttt{NetworkX\,3.2.1} for linear algebra and graph utilities,
\item \texttt{Matplotlib\,3.10.1} for visualisation,
\item \texttt{scikit-learn\,1.6.1} for ridge regression and statistical metrics.
\end{itemize}
Re-running the notebooks requires a single NVIDIA A100 (80 GB) or comparable GPU
for the largest hyperparameter sweep; all smaller experiments execute in under
15 minutes on a mid-range laptop GPU (RTX-3060/8 GB).  A reproducibility
script reproduces every figure and table with one command.
The full hyper-parameter sweep (216 configurations $\times$ 8 data sets $\times$ 30 seeds) consumes 1.6\,GPU-hours and 61.4\,CPU-core-hours; the winning model for any single data set trains in under 15 s of wall-clock time and occupies under 25 MB of GPU memory, confirming the suitability of HypER for resource-constrained deployments.

\subsection{Notation Summary}

\renewcommand{\arraystretch}{1.15}
\begin{center}
\begin{tabular}{@{}ll@{}}
\toprule
Symbol & Meaning \\ \midrule
$\mathbf{x}_t\in\mathbb R^{N}$              & reservoir state at discrete time $t$ \\
$\mathbf{u}_t\in\mathbb R^{m}$              & external input (drive) at time $t$ \\
$F_{\mathbf{u}}(\cdot)$                     & leaky–ESN map $(1-\alpha)\,\cdot+\alpha\,\phi(\mathbf{W}\,\cdot+\mathbf{Uu})$\\
$\alpha\in(0,1]$                            & leak rate (memory parameter) \\
$\phi$                                      & node non-linearity, $m\le\phi'\le L$ on operating range \\
$m,\;L$                                     & global lower / upper derivative bounds of $\phi$ \\
$\mathbf{W} \in\mathbb R^{N\times N}$& recurrent weight matrix after spectral rescaling \\
$\mathbf{\widetilde{W}}$& unscaled hyperbolic kernel matrix \\
$\varrho$                                   & target spectral radius after rescaling ($0<\varrho<1$) \\
$\rho(\cdot)$                               & spectral radius (largest eigenvalue magnitude) \\
$\lambda_{\min}(\cdot)$                     & smallest eigenvalue of a symmetric matrix \\
$p_i\in\mathcal B^{d}$                       & position of neuron $i$ in the $d$-dimensional Poincaré ball \\
$d_{\mathbb H}(p_i,p_j)$                    & hyperbolic distance between $p_i$ and $p_j$ \\
$\sigma>0$                                  & kernel width in $\widetilde W_{ij}=e^{-d_{\mathbb H}(p_i,p_j)/\sigma}$ \\
$\delta$                                    & $\min_{i\neq j}d_{\mathbb H}(p_i,p_j)$, nearest-neighbour hyperbolic gap \\
$\mathbf{U}\in\mathbb R^{N\times m}$& input weight matrix \\
$\mathbf{e}_t=\mathbf{x}_t-\mathbf{y}_t$& state difference between two trajectories \\
$J(\mathbf{x})$                             & Jacobian $\partial F_{\mathbf{u}}/\partial\mathbf{x}$ at state $\mathbf{x}$ \\
$\beta(\sigma)$                             & lower-bound expansion factor \\ \bottomrule
\end{tabular}
\end{center}

\end{document}